\documentclass[twoside,11pt]{article}
\usepackage[preprint]{jmlr2e}
\usepackage[utf8]{inputenc} 
\usepackage[T1]{fontenc}  
\usepackage{hyperref}    
\usepackage{url}      
\usepackage{booktabs}    
\usepackage{amsfonts}    
\usepackage{nicefrac}    
\usepackage{microtype}   
\usepackage{shortcuts}
\usepackage{collcell}
\usepackage{enumitem}
\usepackage{tabularx, multirow}
\usepackage{amsmath}
\usepackage{amssymb}
\usepackage{pifont}
\usepackage{tikz, pgfplots}
\pgfplotsset{compat=1.10}
\usepackage{tikzsymbols}
\usetikzlibrary{calc}
\usetikzlibrary{patterns,decorations.pathreplacing}
\usepgfplotslibrary{fillbetween}
\newtheorem{assumption}{Assumption}
\newtheorem{condition}{Condition}
\usepackage{url}
\usepackage{color, colortbl}
\definecolor{Gray}{gray}{0.9}
\usepackage{mathtools}
\usepackage{wrapfig}
\usepackage{subcaption}
\usepackage[T1]{fontenc}
\usepackage{lmodern}
\usepackage{multirow}
\usepackage{natbib}
\usepackage[font={small}]{caption}
\usepackage{array, makecell}

\newcommand{\unc}{\textup{Unc}} 
\newcommand{\conf}{\textup{Conf}} 
\newcommand{\cor}{\textup{Cor}} 
\newcommand{\avg}{\textup{Avg}} 
\newcommand{\weight}{\textup{Weight}}

\newcommand{\risk}[2]{\epsilon({#1},{#2})}

\newcommand{\riskC}[2]{\epsilon_{\conf}({#1},{#2})}
\newcommand{\emriskUC}[2]{\hat{\epsilon}_{\unc}({#1},{#2})}
\newcommand{\emerrorUC}[3]{\hat{\epsilon}_{\unc}({#1},{#2},{#3})}

\newcommand{\emriskC}[2]{\hat{\epsilon}_{\conf}({#1},{#2})} 
 
\newcommand{\HC}{\mathcal{H}} 
\newcommand{\RC}{\Phi} 
\newcommand{\GC}[2]{\mathcal{G}_{#1}({#2})} 
\newcommand{\WGC}[2]{\bar{\mathcal{G}}_{#1}({#2})} 
 
\usepackage{cleveref}
\usepackage[titlenumbered,ruled]{algorithm2e}
\crefname{algocf}{algorithm}{algorithms}
\Crefname{algocf}{Algorithm}{Algorithms}
\newcommand{\cmark}{\ding{51}}%
\newcommand{\xmark}{\textcolor{lightgray}{\ding{55}}}%

\makeatletter
\renewcommand{\SetKwInOut}[2]{%
  \sbox\algocf@inoutbox{\KwSty{#2}\algocf@typo:}%
  \expandafter\ifx\csname InOutSizeDefined\endcsname\relax
    \newcommand\InOutSizeDefined{}\setlength{\inoutsize}{\wd\algocf@inoutbox}%
    \sbox\algocf@inoutbox{\parbox[t]{\inoutsize}{\KwSty{#2}\algocf@typo:\hfill}~}\setlength{\inoutindent}{\wd\algocf@inoutbox}%
  \else
    \ifdim\wd\algocf@inoutbox>\inoutsize%
    \setlength{\inoutsize}{\wd\algocf@inoutbox}%
    \sbox\algocf@inoutbox{\parbox[t]{\inoutsize}{\KwSty{#2}\algocf@typo:\hfill}~}\setlength{\inoutindent}{\wd\algocf@inoutbox}%
    \fi%
  \fi
  \algocf@newcommand{#1}[1]{%
    \ifthenelse{\boolean{algocf@inoutnumbered}}{\relax}{\everypar={\relax}}%
    {\let\\\algocf@newinout\hangindent=\inoutindent\hangafter=1\parbox[t]{\inoutsize}{\KwSty{#2}\algocf@typo:\hfill}~##1\par}%
    \algocf@linesnumbered
  }}%
\makeatother

\hypersetup{
    colorlinks = true,
    citecolor  = blue,
    linkcolor  = blue
}

\newcommand{\sigmoid}[1]{#1*(1/(1 + exp(-x))) -(#1/3)}

\begin{document}

%

%





\jmlrheading{1}{2022}{xx-xx}{xx/xx}{xx/xx}{}{}


\ShortHeadings{Combining Observational and Randomized Data}{}
\firstpageno{1}


\title{Combining Observational and Randomized Data\\for Estimating Heterogeneous Treatment Effects}

\author{\name Tobias Hatt  \email thatt@ethz.ch  \\
       \addr ETH Zurich
       \AND
      \name Jeroen Berrevoets \email jeroen.berrevoets@damtp.cam.ac.uk \\
      \name Alicia Curth \email amc253@damtp.cam.ac.uk \\
       \addr University of Cambridge
        \AND
       \name Stefan Feuerriegel \email feuerriegel@lmu.de\\
       \addr ETH Zurich \& LMU Munich
        \AND
       \name Mihaela van der Schaar \email mv472@damtp.cam.ac.uk \\
       \addr University of Cambridge \& The Alan Turing Institute \& UCLA
}

\editor{Draft}

\maketitle

\begin{abstract}
Estimating heterogeneous treatment effects is an important problem across many domains. In order to accurately estimate such treatment effects, one typically relies on data from observational studies \emph{or} randomized experiments. Currently, most existing works rely \emph{exclusively} on observational data, which is often confounded and, hence, yields biased estimates. While observational data is confounded, randomized data is unconfounded, but its sample size is usually too small to learn heterogeneous treatment effects. In this paper, we propose to estimate heterogeneous treatment effects by combining large amounts of observational data \emph{and} small amounts of randomized data via representation learning. In particular, we introduce a two-step framework: first, we use observational data to learn a shared structure (in form of a representation); and then, we use randomized data to learn the data-specific structures. We analyze the finite sample properties of our framework and compare them to several natural baselines. As such, we derive conditions for when combining observational and randomized data is beneficial, and for when it is not. Based on this, we introduce a sample-efficient algorithm, called CorNet. We use extensive simulation studies to verify the theoretical properties of CorNet and multiple real-world datasets to demonstrate our method's superiority compared to existing methods.
\end{abstract}



\section{Introduction}
Estimating heterogeneous treatment effects is of great importance in many domains such as marketing \citep{Brodersen2015a}, economics \citep{Heckman1997}, and epidemiology \citep{Robins2000a}. For instance, in the medical domain, clinicians must know the effect a drug (treatment) has on an individual patient in order to personalize treatment decisions.

When estimating heterogeneous treatment effects, most works rely on data from observational studies (OSs), since it is available in large quantities. This allows to estimate individual-level heterogeneity of treatment effects \citep[\eg,][]{Johansson2016, Shalit2017a, Wager2018a}. However, data from OSs is often subject to unobserved confounding. That is, some confounders (\ie, variables that affect both treatment assignment and outcome) are not measured and, hence, not available in the data. Therefore, we are unable to control for these unobserved confounders, which results in biased estimates of the treatment effect.

In order to circumvent the problem of unobserved confounding, existing methods assume that data from OSs is unconfounded, \ie, all confounders are observed \citep[\eg,][]{Johansson2016, Shalit2017a}, or rely on other assumptions \citep[\eg,][]{wang2019blessings, hatt2021seqdeconf}. However, these assumptions are strong, since they are not testable and often do not hold true in practice \citep[\eg,][]{rosenbaum2010design,kallus2018removing,kallus2018confounding, kallus2018interval,wang2019blessings, zhao2019sensitivity,rosenman2020combining,hatt2021seqdeconf}. As a consequence, without making strong assumptions, the use of observational data for estimating heterogeneous treatment effect is prohibited in practice.

An alternative data source stems from randomized experiments. For instance, randomized controlled trials (RCTs) are widely recognized as the gold standard for estimating treatment effects \citep{Robins2000a}. The reason for this is that, in RCTs, the treatment assignment is controlled by an investigator, and, hence, all variables that affect the treatment assignment are known (\ie, observed). As a result, data from RCTs is {\it unconfounded}, and, therefore, using such data for estimating treatment effects yields unbiased estimates. However, randomized data has been predominantly used to estimate average treatment effects rather than heterogeneous treatment effects. There are two major reasons for this: (i)~The size of randomized data is often to small to estimate treatment effect heterogeneity across patients. This is due to the large cost of RCTs as well as the difficulties of recruiting enough eligible subjects. (ii)~Subjects in RCTs are, by design, highly selected and, therefore, regularly not representative of the population of interest \citep{downs1998feasibility,norris2001effectiveness,willan2004regression,rothwell2005external,cole2010generalizing, stuart2011use,buchanan2018generalizing, hatt2021generalizing,flores2021assessment}. For instance, a review of HIV/AIDS clinical trials found that women are largely underrepresented in these trials \citep{gandhi2005eligibility, greenblatt2011priority}. Hence, findings from such trials can often not be generalized to the population of interest (\ie, the population of HIV-positive patients). In sum, relying \emph{exclusively} on either observational \emph{or} randomized data may not yield reliable estimates of heterogeneous treatment effects.

In this paper, we investigate how to combine: (i)~large but possibly confounded observational data; \emph{and} (ii)~unconfounded but small randomized data, to estimate heterogeneous treatment effects. This is in contrast to most previous works, which rely on one data type exclusively. Combining \textit{both} data types allows us to leverage randomized data to overcome the untestable unconfoundedness assumption \textit{and} observational data to discover effect heterogeneity. This are particularly relevant for medical and pharmaceutical sciences, where both types of data coexist. That is, RCTs are required for a drug's approval and, once approved and used in practice, observational data is routinely collected during post-drug monitoring. Furthermore, the 21st Century Cures Act \citep{21curesact} places additional focus on the use of observational and randomized data to support healthcare decision-making. Most importantly, it enhances the U.S. Food and Drug Administration's ability to modernize RCTs by authorizing the use of OSs.

We propose a two-step framework for combining observational and randomized data which relies on a shared structure between both data types. In the first step, observational data is used to learn a shared structure (in the form of a representation). In the second step, randomized data is used to learn the data-specific structures. We theoretically and empirically compare our proposal to natural baselines making use of either one or both data types. This allows us to derive conditions indicating when it is useful to combine observational and randomized data, and when it is not. In particular, we prove finite sample (\ie, non-asymptotic) learning bounds. A finite sample view is particularly important in our setting, since it offers insights into which factors drive the estimation error when the size of randomized data is small. Our theoretical analysis reveals three driving factors of relative performance: (i)~the size of the observational data; (ii)~the discrepancy between the RCT population and the population of interest; and (iii)~the complexity of the bias function due to unobserved confounding. Guided by the insights obtained through this finite sample view, we design a sample-efficient algorithm of our two-step framework for combining observational and randomized data. The specific instantiation is based on neural networks; hence, we call this algorithm \textbf{CorNet}. We then verify the finite sample properties of our algorithm empirically using simulation studies. Moreover, we use multiple real-world randomized experiments to demonstrate that, compared to existing baselines, our CorNet yields superior performance. A summary of our findings is outlined in \Cref{tbl:outline}.

We summarize our main \textbf{contributions}\footnote{Code available at \url{https://github.com/tobhatt/CorNet}.} as follows:
\begin{enumerate}
	\item We study how to combine observational and randomized data for estimating heterogeneous treatment effects. For this, we propose a two-step framework which relies on a shared structure between both data types.
	\item We derive finite sample learning bounds for our two-step framework and, moreover, for several natural baselines. This yields theoretical insights into the drivers of the estimation error and allows us to derive conditions for when observational and randomized data should be combined and when not.
	\item Guided by these theoretical insights, we propose \textbf{CorNet}, a sample-efficient algorithm of our two-step framework. We verify its finite sample properties empirically using extensive simulation studies and compare to baselines which use one data type exclusively. Moreover, we use multiple real-world randomized experiments to demonstrate that our CorNet outperforms the state-of-the-art by a substantial margin.
\end{enumerate}

\begin{table*}[ht!]
	\caption{\footnotesize Summary of the questions raised and answered in this paper and where to find the details. We use ``obs.'' as abbreviation for ``observational'' and ``rand.'' as abbreviation for ``randomized''.}\label{tbl:outline}
	\vspace{-1em}
	\begin{center}
				\begin{tabularx}{\textwidth}{lXl}
					\toprule
					\makecell[l]{\bf Question} & {\bf Answer} & {\bf Details in}\\
					\midrule
					\addlinespace[0.75ex]
					\rowcolor{black!5}\parbox{.25\textwidth}{\it How to combine obs. and rand. data?} & \parbox{.4\textwidth}{Our two-step framework}&\Cref{sec:combining_data}\\[0.2cm]
					\addlinespace[0.75ex]
					\parbox{.25\textwidth}{\it What factors drive the finite sample error?} &  \parbox{.4\textwidth}{\begin{itemize}[leftmargin=*]\setlength\itemsep{-0.5em}
    					    \item Size of obs. data
    					    \item Distributional discrepancy between obs. and rand. data
    					    \item Complexity of the bias function
    				    \end{itemize}
    				}&
    				\Cref{sec:error_bounds}, \Cref{thm_bound}\\[0.2cm]
					\addlinespace[0.75ex]
					\rowcolor{black!5}\parbox{.25\textwidth}{\it When to combine obs. and rand. data?} & 
					\parbox{.4\textwidth}{
    			    \begin{itemize}[leftmargin=*]\setlength\itemsep{-0.5em}
    					    \item Large obs., but small rand. data
    					    \item Large discrepancy between obs. and rand. data
    				    \end{itemize}
    				}& 
    				\Cref{sec::comparison_baselines}, Prop.~\ref{prop:condition}\\[0.2cm]
    				\addlinespace[0.75ex]
                    \parbox{.25\textwidth}{\it How to efficiently combine obs. and rand. data?}& 
                    \parbox{.4\textwidth}{Our \textbf{CorNet}, which can:       \begin{itemize}[leftmargin=*]\setlength\itemsep{-0.5em}
				        \item balance the covariate distributions
				        \item regularize the bias function
			        \end{itemize}
                    }& 
                    \Cref{sec:algorithms}, \Cref{alg:CORNet}\\[0.2cm]
					\bottomrule
			\end{tabularx}
	\end{center}
\end{table*}

\section{Related Work}
In this section, we give an overview of (i)~methods for estimating heterogeneous treatment effects that rely on observational data whilst making strong assumptions and (ii)~methods that combine observational and randomized data, but which are mostly for estimating average treatment effects. In addition, (iii)~we review related approaches from multi-task and transfer learning.

\textbf{(i)~Estimating heterogeneous treatment effects using observational data.} Many machine learning methods have been adapted for estimating heterogeneous treatment effects. These methods range across random forest-based methods \citep[\eg,][]{Wager2018a}, Bayesian algorithms due to their ability to quantify uncertainty \citep[\eg,][]{Alaa2017d, alaa2018limits, zhang2020}, and deep learning algorithms due to their strong predictive performance \citep[\eg,][]{Johansson2016, Yoon2018a, hatt2021estimating}. Especially deep learning has been used to learned a shared representation and treatment-specific hypotheses for each outcome \citep[\eg,][]{Shalit2017a, Yao2018a, curth2021nonparametric, curth2021inductive}. However, all these methods are based on the assumption that the observational data is unconfounded, which is not testable and usually fails to hold true in practice \citep[\eg,][]{rosenbaum2010design,kallus2018confounding, kallus2018interval,wang2019blessings, zhao2019sensitivity,hatt2021seqdeconf}. Hence, this prohibits the above methods from being applied in practice.

Another stream of research aims towards using latent variable models to recover unobserved confounders. For instance, some approaches try to recover the true confounders from noisy proxies of the true confounders \citep[\eg,][]{Louizos2017, kuzmanovic2021deconfounding}. However, we do not know whether the covariates we observe are indeed proxies of the true confounders. Other attempts try recover unobserved confounders using the treatment assignment of either multiple treatments \citep{wang2019blessings, bica2019time} or sequential treatments \citep{hatt2021seqdeconf}. However, these approaches again assume variants of no unmeasured confounders, so-called single strong ignorability and time-invariant confounders. Neither of the assumptions can be tested in practice, thereby limiting the usefulness of these methods in practice. 

\textit{Difference to our work:} We accept the presence of unobserved confounders in observational data, but recognize the potential use of randomized data. Therefore, we proposed a two-step framework for combine observational \emph{and} randomized data.

\textbf{(ii)~Combining observational and randomized data.} There exist only a few recent methods for combining observational and randomized data. For instance, \citet{kallus2018removing} use observational data to estimate a flexible, but biased function for the heterogeneous treatment effect and then aim to remove the bias using the randomized data. However, they makes strong parametric assumptions on the form of the bias due to the small sample size of randomized data. This restricts their method in its applicability to biases which are linear in the covariates. A recently proposed approach is to estimate two separate estimators: one estimator on observational data, which is biased, and one estimator on randomized data, which is unbiased. Then, a weighted average of these two estimators is used. For estimating average treatment effects, the weight can be chosen by a shrinkage factor determined by Stein’s unbiased risk estimate \citep{rosenman2020combining} or by minimizing the asymptotic variance \citep{yang2020combining}. Concurrent to our work, this averaging approach has been extended to heterogeneous treatment effects. \citet{cheng2021adaptive} propose to estimate one estimator on observational data and one on randomized data. Then, a weighted average of the two estimators is used. However, in order to tune the averaging weight, a large validation set of randomized data is required. We observe in our experiments that this stands in conflict with the small sample size of randomized data.

Several further approaches have been proposed recently, but for different settings than ours. \citet{ilse2021combining} propose a causal reduction method for combining interventional and observational data. However, they consider a discrete treatment \emph{and} discrete outcome and rely on linear-Gaussian models. This is different from our setting, since we consider a discrete treatment and a continuous outcome.
Other approaches rely on outcomes from different time-steps such as before and after the treatment assignment \citep{strobl2021generalizing} or short- and long-term outcomes \citep{athey2020combining, imbens2022long}. This renders both approaches inapplicable to the standard causal inference setting, which is the focus of our work.

\textit{Difference to our work:} We propose to combine observational and randomized data via representation learning. For this, we leverage observational data to learn a shared representation and, then, use the randomized data to learn the data-specific structures. By leveraging a shared representation, we can learn complex and non-linear biases even though the size of the randomized data is small. Moreover, we can leverage representation learning to account for the distributional discrepancy between observational and randomized data.

\textbf{(iii)~Multi-task and transfer learning.} Combining observational and randomized data can be seen as an instance of multi-task learning (MTL) or transfer learning (TL). MTL leverages data from multiple different, but related prediction tasks to estimate similar predictive models for these tasks. To this end, MTL assumes some similarity across the predictive tasks \citep[\eg,][]{caruana1997multitask, argyriou2006multi}. For instance, such similarity can be enforcing the same covariate support (for linear regression) or the same kernel \citep{caruana1997multitask, meier2008group, jalali2010dirty} for all tasks. A similarity can also be achieved by intermediate (neural network) representations (\ie, the same weights are used for intermediate layers for all tasks) \citep{collobert2008unified, maurer2016benefit}. While MTL estimates a model that works well on multiple related tasks, TL focuses on learning a single new task by transferring information from an already learned related task \citep{pan2009survey, tripuraneni2020theory, du2020few}. 

\textit{Difference to our work:} Similar to MTL and TL, we leverage a shared representation between observational and randomized data to transfer information from observational data to (different, but related) randomized data for which we have far fewer samples. However, our setting is distinctly different for two reasons. First, MTL and TL usually do not focus on possible covariate distribution mismatches between tasks. However, in practice, the subjects in an RCT are often not representative for the population of interest, which yields such a distribution mismatch \citep[\eg,][]{rothwell2005external,cole2010generalizing, stuart2011use,buchanan2018generalizing, hatt2021generalizing}. Second, in MTL and TL, task-specific functions are used, since the data types considered are usually not closely related. For instance, \citet{bayati2018statistical} simultaneously estimate logistic regressions with the aim to predict \emph{different} diseases (heart failure, diabetes, dementia, cancer, pulmonary disorder, etc.). In contrast, observational and randomized data are closely related, since both concern the same treatment, but biased due to unobserved confounding. This goes as far as practitioners ignoring randomized data and relying only on observational data \cite[\eg,][]{Shalit2017a, berrevoets2020organite,berrevoets2021learning,  hatt2021estimating}. This allows us to learn the bias more efficiently by exploiting structural similarities.

\section{Preliminaries}
\subsection{Problem Setup}\label{sec::problem_setup}
We consider the standard causal inference setting in which our objective is to estimate the effect of a binary treatment. For this, let $T\in\{0,1\}$ be a treatment, $X\in\R d$ patient covariates, and $Y\in\Rl$ an outcome. Then, our objective is to estimate the conditional average treatment effect (CATE) defined as 
\begin{equation}\label{eq:cate}
	\tau(x) = \E[Y(1)-Y(0)\mid X=x],
\end{equation}
where $Y(1), Y(0)\in\Rl$ are the potential outcomes for treatment and control \citep{Robins2000a}. Estimating the CATE in \labelcref{eq:cate} is non-trivial, since we never observe both potential outcomes, but only the outcome under treatment \emph{or} control. 

When estimating the CATE, two types of data can be considered: (i)~observational data, which originates from an OS, and (ii)~randomized data, which originates from an RCT. Formally, let $R\in\{0,1\}$ be an indicator of participation in the RCT ($R=1$) or OS ($R=0$). Then, randomized data is drawn from $(X^{\unc}, T^{\unc}, Y^{\unc}) \sim (X, T, Y\mid R=1)$ and observational data is drawn from $(X^{\conf}, T^{\conf}, Y^{\conf}) \sim (X, T, Y\mid R=0)$.\footnote{We use the superscript ``Unc'' for the unconfounded randomized data and ``Conf'' for the confounded observational data.} When estimating the CATE from solely (i)~observational data \emph{or} (ii)~randomized data, there may occur the following problems:

\textbf{(i)~Problems with observational data: unobserved confounders.} Observational data, $(X^{\conf}, T^{\conf}, Y^{\conf})$, is likely to be confounded \citep[\eg,][]{rosenbaum2010design,kallus2018removing,kallus2018confounding, wang2019blessings, zhao2019sensitivity}. That is, some confounders, \ie, variables that affect treatment assignment and outcome, are unobserved and, therefore, not contained in the observational data. Hence, when conditioning on the observed covariates, the potential outcomes and the treatment assignment are \emph{not} independent, \ie, $Y(1), Y(0)\not\!\perp\!\!\!\perp T^{\conf}\mid X^{\conf}$. Because of this, in general,
\begin{align}
    \E[Y^{\conf}\mid X^{\conf}=x, T^{\conf}=t]\neq \E[Y(t)\mid X^{\conf}=x],
\end{align}
and, therefore, using $\E[Y^{\conf}\mid X^{\conf}=x, T^{\conf}=t]$ as an estimate for $\E[Y(t)\mid X^{\conf}=x]$ may yield biased estimates of the CATE. 




On the other hand, randomized data, which originated from an RCT, is unconfounded, since the treatment assignment was randomized. Hence, $Y(1), Y(1) \!\perp\!\!\!\perp T^{\unc}\mid X^{\unc}$ holds true for randomized data. As such, different to observational data, we can obtain reliable estimates of $\E[Y(t)\mid X^{\unc}=x]$, since
\begin{align}
    \E[Y^{\unc}\mid T^{\unc}=t, X^{\unc}=x]= \E[Y(t)\mid X^{\unc}=x].
\end{align}

\textbf{(ii)~Problems with randomized data: sample size and generalizability.} RCTs are costly to conduct and the patient recruitment is a tedious undertaking. As a consequence, randomized data admits two different problems than observational data. First, due to the costs of RCTs, the sample size of randomized data is usually very small. This is limiting its usefulness for estimating treatment heterogeneity. Second, due to the difficulty of patient recruitment, many RCTs lack generalizability (also known as external validity) \citep[\eg,][]{rothwell2005external,cole2010generalizing,buchanan2018generalizing, hatt2021generalizing}. This means that the covariate distribution in the randomized data, may differ from the actual covariate distribution of the population of interest, \ie, 
\begin{equation}
    p(X^{\unc})\neq p(X).
\end{equation}

Fortunately, observational data is usually available in large quantities and, hence, can be used to estimate treatment effect heterogeneity. Moreover, observational data shares the covariate distribution of the population of interest, \ie, $p(X^{\conf})=p(X)$. This is due to the nature of OSs, in which information about treatment outcomes is collected in the actual population of interest.

In this paper, we argue that these two data types are complementary and study how to combine observational data, $(X^{\conf}, T^{\conf}, Y^{\conf})$, and randomized data, $(X^{\unc}, T^{\unc}, Y^{\unc})$, to estimate the CATE, $\tau(x)$. Moreover, we study the influence of the different properties of observational and randomized data when estimating the CATE.





\subsection{Assumptions for Identification}\label{sec::identification_assum}
Identification of the CATE, $\tau(x)$, requires assumptions, which relate the potential outcomes to the observed data. For this, we first state the standard assumptions for causal inference \citep[\eg,][]{rosenbaum2010design} that hold globally and then consider additional assumptions that hold for either randomized or observational data.

\begin{assumption}\label{assum:standard}(Standard Assumptions)
    \begin{enumerate}
        \item[(i)] Consistency, \ie, $Y=T\,Y(1)+(1-T)\,Y(0)$,
        \item[(ii)] Treatment overlap, \ie, $0<\mathbb{P}(T=t\mid X=x, R=r)<1$, for all $t$, $x$, and $r$,
        \item[(iii)] Population overlap, \ie, $0<\mathbb{P}(R=r\mid X=x)<1$, for all $x$ and $r$.
    \end{enumerate}
\end{assumption}
Consistency states that if a patient receives treatment $t$, then the observed outcome is the corresponding potential outcome $Y(t)$. In particular, this holds true for both observational and randomized data, since $Y^{\unc} \sim (Y\mid R=1)$ and, thus, $Y^{\unc} = T^{\unc}\,Y(1)+(1-T^{\unc})\,Y(0)$ (analogously for $Y^{\conf}$). Treatment overlap requires that the treatment assignment is nondeterministic in both the RCT and OS. These are standard assumptions in causal inference \citep[\eg,][]{rosenbaum2010design}. Population overlap requires that the covariates in the randomized and observational data (\ie, $X^\unc$ and $X^\conf$) to have common support. However, it allows the covariates to follow different distributions, since the subjects in the RCT may not be representative for the population of interest.

Further, in order to make the CATE identifiable from the randomized data and generalizable to the population of interest, we make the following assumptions on the RCT and OS separately.
\begin{assumption}\label{assum:unc_rct}(Unconfoudedness, no outcome modification, and generalizability)
    \begin{enumerate}
        \item[(i)] Unconfounded RCT, \ie, $Y(1), Y(0) \!\perp\!\!\!\perp T \mid X, R=1$,
        \item[(ii)] No outcome modification, \ie, $Y(1), Y(0) \!\perp\!\!\!\perp R \mid X$,
        \item[(iii)] Generalizable OS, \ie, $\mathbb{P}(R=0\mid X=x) = \mathbb{P}(R=0)$.
    \end{enumerate}
\end{assumption}
Assumption~(i) ensures that the data we obtained from the RCT is unconfounded, and, hence, the expected potential outcomes can be estimated from this data. Assumption~(ii) means that being in the RCT or OS does not affect the potential outcomes.\footnote{In practice, it can happen that the potential outcomes are different in an OS due to non-compliance. However, this problem is orthogonal to our work. For works on non-compliance, see \citet{robins1994correcting, cuzick1997adjusting,ye2014estimating}.} This is a very weak assumption and merely ensures that the functional relationship between covariates and potential outcomes is the same across data sources. Assumption~(iii) implies that the covariate distribution in the OS is the same as in the population of interest.\footnote{Assumption~(iii) implies the same covariate distributions, since $p(X^{\conf}) = p(X\mid R=0) = \mathbb{P}(R=0\mid X)p(X)/\mathbb{P}(R=0) = p(X)$.} While (i)~and (ii)~are made explicitly in previous work for combining RCT and OS data \citep{kallus2018removing, rosenman2020combining, cheng2021adaptive}, (iii)~has been implicitly used, but not explicitly stated.

These assumptions ensure that the treatment effect is identifiable only within the small RCT. However, the estimation of CATE is hindered by the small sample size of RCTs and their lack of generalizability. We overcome these obstacles by combining observational data and randomized.



\subsection{Notation}
Throughout this paper, we use the following notations. We use upper-case letters (\ie, $X$) for random 
variables and lower-case letters (\ie, $x$) for realizations of these random variables. Further, we use 
bold upper-case letters (\ie, $\mathbf{W}$) 
for matrices and bold lower-case letters (\ie, $\mathbf{v}$) for vectors. Generically, we 
use ``hatted'' functions, vectors, and matrices 
(\eg, $\hat{\textup{f}}$, $\hat{\mathbf{v}}$, 
and $\hat{\mathbf{W}}$) to refer to their (random) estimators. We use 
$\RC$ to refer to a function class of representations 
$\phi: \R{d} \rightarrow \R{d_\phi}$ and
$\HC$ to refer to a function class of hypotheses
$\textup{h}: \R{d_\phi} \rightarrow \R{}$. 
Further, we use 
$\mathcal{B}$ to refer to a function class of bias functions 
$\delta: \R{d_\phi} \rightarrow \R{}$. For the hypothesis class 
$\HC$, we use 
$\HC^{\otimes 2}$ to refer to the 2-fold Cartesian product, 
\ie, $\HC^{\otimes 2} = \{\mathbf{h}=(\textup{h}_1,\textup{h}_0):\, \textup{h}_1\in\HC,\, \textup{h}_0\in\HC\}$. 
Similarly, $\mathcal{B}^{\otimes 2}$ refers to the 2-fold Cartesian product of the bias function class 
$\mathcal{B}$. We 
use `$\circ$' to denote composition of functions. That is, for two functions 
$\phi\in\RC$ and $\textup{h}\in\HC$, their composition is given by 
$\textup{h}\circ\phi(x) = \textup{h}(\phi(x))$. In particular, for 
$\textup{h}\in\HC$ and $\delta\in\mathcal{B}$, we also 
use `$\circ$' to refer to addition of 
$\textup{h}$ and $\delta$ using the same representation, 
\ie, $(\textup{h}+\delta)\circ \phi(x) = \textup{h}(\phi(x))+\delta(\phi(x))$. We use 
$\tilde{\mathcal{O}}$ to denote an expression that hides polylogarithmic factors.

\section{Combining Observational and Randomized Data by Learning a Shared Representation}\label{sec:combining_data}
In this section, we discuss how to estimate the CATE using both observational and randomized data. For this, we build upon a long-standing history in machine learning of assuming a shared structure between learning tasks. This is particularly prominent in multi-task learning, transfer learning, and, most relevant to our work, in causal inference \citep[\eg,][]{caruana1997multitask,argyriou2006multi,maurer2016benefit,Shalit2017a,du2020few,zhang2017survey} and has been proven useful in machine learning across general learning tasks \citep[\eg,][]{ruder2017overview,radford2019language,liu2019multi, li2020federated,brown2020language}. In particular, we first formalize the difference in observational and randomized outcomes using a shared representation and connect it to the bias due to unobserved confounding (\Cref{sec::dgp}). Motivated by this, we introduce our two-step procedure for combining observational and randomized data (\Cref{sec::two_step_procedure}).


\subsection{Formalizing the Bias due to Unobserved Confounding}\label{sec::dgp}
Here, we assume the shared structure between observational and randomized data to be the \emph{representation}, which is a function $\phi^\ast:\Rl^d\rightarrow\Rl^{d_\phi}$ that maps the covariates $X\in\Rl^d$ to some representation of the covariates $Z\in\Rl^{d_\phi}$. This representation captures the shared structure between observational and randomized data and, hence, is the same function for both types of data. On top of the representation, we consider hypotheses for both treatment and control. These hypotheses are ``data-specific'' and, thus, can vary across observational and randomized data. That is, we consider, for $t\in\{0, 1\}$, separate hypotheses $\textup{h}_t^u, \textup{h}_t^c:\Rl^{d_\phi}\rightarrow\Rl$ for randomized and observational data.\footnote{The superscripts ``u'' and ``c'' denote ``\textbf{u}nconfounded'' and ``\textbf{c}onfounded'' hypotheses.} 

Then, we recognize that the difference between observational and randomized outcomes can be connected to the bias due to unobserved confounding by the following:
\begin{align}
    &\E[Y(t)\mid X^{\conf}=x] - \E[Y^{\conf}\mid T^{\conf}=t, X^{\conf}=x]=\label{eq:bias_u_c}\\
    & \E[Y^{\unc}\mid T^{\unc}=t, X^{\unc}=x] - \E[Y^{\conf}\mid T^{\conf}=t, X^{\conf}=x]\label{eq:diff_outcomes}=\\
    & (\textup{h}^u_t - \textup{h}^c_t)\circ \phi^\ast(x),\label{eq:diff_hypo}
\end{align}
where \labelcref{eq:bias_u_c} is the bias due to unobserved confounding and \labelcref{eq:diff_outcomes} follows directly from (i) and (ii)~in Assumption~\labelcref{assum:unc_rct}. As such, the bias due to unobserved confounding is connected to the difference in the ``data-specific'' hypotheses.

Formally, the observational and randomized outcomes, $Y^{\conf}$, $Y^{\unc}$, are generated by 
\begin{align}
		Y^{\conf}\,\sim \,\textup{h}^c_{T^{\conf}}\circ\phi^\ast(X^{\conf}) + \varepsilon^{\textup{c}},\\
	    Y^{\unc}\,\sim \,\textup{h}^u_{T^{\unc}}\circ\phi^\ast(X^{\unc}) + \varepsilon^{\textup{u}},
\end{align}
where $\varepsilon^{\textup{c}}, \varepsilon^{\textup{u}}\sim\mathcal{N}(0, \sigma_{{c,u}}^2)$. The representation, $\phi^\ast$, is shared and, hence, encodes the common structure between the observational and randomized data. The confounded and unconfounded hypotheses, $\textup{h}_t^c,\,\textup{h}^u_t$ may differ and, hence, capture the bias due to unobserved confounding as in \labelcref{eq:diff_hypo}. 

Moreover, the covariates, \ie, $X^{\conf},X^{\unc}\sim p^{\conf}_{x}, p^{\unc}_{x}$, and the treatment assignment, \ie, $T^{\conf},T^{\unc}\sim p^{\conf}_{t}, p^{\unc}_{t}$, can follow any distribution, which satisfies the assumptions from \Cref{sec::identification_assum}. In particular, the distributions $p^{\conf}_{x}$ and $p^{\unc}_{x}$ may differ across observational and randomized data.\footnote{This is due to the fact that the subjects in RCTs are often not representative of the population of interest. See discussion in \Cref{sec::problem_setup} for details.} Moreover, the treatment assignment may depend on unobserved confounders. However, as we have seen earlier in this section, the bias due to unobserved confounding can be formulated as the difference between of the observational and randomized outcomes.

Motivated by these insights, we propose an intuitive two-step procedure, which exploits the shared structure of observational and randomized data for estimating the CATE.

\subsection{Two-Step Procedure for Combining Observational and Ranomized Data}\label{sec::two_step_procedure}
We introduce our two-step procedure for estimating the CATE that combines observational and randomized data. In the first step, we use observational data to learn a representation $\hat{\phi}$ and confounded hypotheses $\hat{\textup{h}}_t^c$ for each $t\in\{0,1\}$. Since the representation is shared across observational and randomized data, we can use it in the second step, together with the randomized data, to debias the confounded hypotheses. Note that we do not propose to learn the unconfounded hypothesis, $\hat{\textup{h}}_t^u$. Rather, we use the fact that the bias due to unobserved confounding can be expressed by the difference in the ``data-specific'' hypotheses, as in \labelcref{eq:diff_hypo}, and learn the bias function, \ie, $\delta_t = \textup{h}_t^u - \textup{h}_t^c$. We do so, since learning $\hat{\textup{h}}_t^u$ may be statistically inefficient due to the small sample size of the randomized data. As such, structural restrictions on the form of $\hat{\textup{h}}_t^u$ would have to be made upfront. However, to improve upon this, we seek to learning $\delta_t$ directly. Specifically, we recognize that, while $\hat{\textup{h}}_t^c$ and $\hat{\textup{h}}_t^u$ are different, they are likely to be related (\ie, structurally similar), because they both concern the same drug and partially overlapping populations. Hence, this allows us to make structural restrictions on $\delta_t$, which plays an important role later, when we derive a finite sample learning bound. An illustration of our two-step procedure can be found in \Cref{fig:2_step}.

\begin{figure}
    \centering
    \begin{subfigure}[b]{.5\textwidth}
        \begin{tikzpicture}[
            square/.style={rectangle, inner sep=3, draw=black, rounded corners, fill=blue!10},
            point/.style={circle, inner sep=2, draw=none},
        ]
            
            \node[square] (f) at (0,0) {\texttt{model}};
            \node[square] (phi) at (2, 0) {$\hat{\phi}$};
            

            \node[anchor=west] (biased) at (5, 0) {\textcolor{blue}{$\hat{\text{h}}^c_t$}};
            \node[anchor=west] (true) at (5, -2) {\textcolor{red}{$\hat{\text{h}}^u_t$}};
            
            \draw [
                blue,
                decoration={
                    brace,
                    raise=.35cm
                },
                rotate=90,
                decorate,
                thick
            ] ($(f.west) + (.3, 0)$) -- ($(biased.east) + (.3, 0)$) node [pos=0.5,anchor=south, yshift=.5cm] {\textcolor{blue}{{\bf step 1:} $\mathcal{D}^\conf$}};
            
            \draw [
                thick,
                red,
                decoration={
                    brace,
                    raise=.35cm
                },
                decorate
            ] ($(biased.north) + (.5, .1)$) -- ($(true.south) + (.5, -.1)$) node [pos=0.5,anchor=west,xshift=.5cm] {\rotatebox{270}{\textcolor{red}{{\bf step 2:} $\mathcal{D}^\unc$}}};

            \draw[-latex, thick] (f) -- (phi);
            \draw[-latex, thick, blue] (phi) -- (biased);
            \draw[-latex, thick,  red] (phi) to[out=0] (true);
            \draw[latex-latex, dashed] (biased) -- node[midway, right] {\rotatebox{270}{bias}} (true);

        \end{tikzpicture}
    \end{subfigure}%
    ~
    \begin{subfigure}[b]{.5\textwidth}
        \begin{tikzpicture}
            \begin{axis}[
                every axis plot post/.append style={
                mark=none,domain=-5:5,samples=50,smooth, thick},
                x label style={anchor=north},
                xlabel=$\hat{\phi}$,
                axis x line=bottom,
                axis y line=left,
                height=4cm,
                width=\textwidth,
                ticks=none,
                axis line style = thick,
                axis on top,
                clip=false
            ]
            \addplot+[color=blue, name path=biased, very thick] {\sigmoid{1}}
            node[pos=1, right, fill opacity=1] {\textcolor{blue}{$\hat{\text{h}}^c_u$}};
            
            \addplot[color=red, name path=true,thick] {\sigmoid{2}} 
            node[pos=1, right, fill opacity=1] {\textcolor{red}{$\hat{\text{h}}^u_t$}};
            
            \addplot[color=black!10] fill between[of=true and biased];
            
            \node at (axis cs: 3.5, 1) {bias};
 
            \end{axis}
        \end{tikzpicture}
    \end{subfigure}

    \caption{{\bf Conceptual overview of our procedure.} Given a confounded dataset ($\mathcal{D}^\conf$), and an unconfounded dataset ($\mathcal{D}^\unc$), our procedure learns an unconfounded estimate using both datasets. We propose to do this in two steps: in step 1, we learn a shared representation $\hat{\phi}$ and {\it confounded} hypotheses, and in step 2 we {\it correct} for the bias in these confounded hypotheses such that we have {\it unconfounded} hypotheses ($\hat{h}^u_t$). Left, we depict a conceptual overview of our procedure, and right we show the potential impact on the actual hypotheses.}
    \label{fig:2_step}
\end{figure}
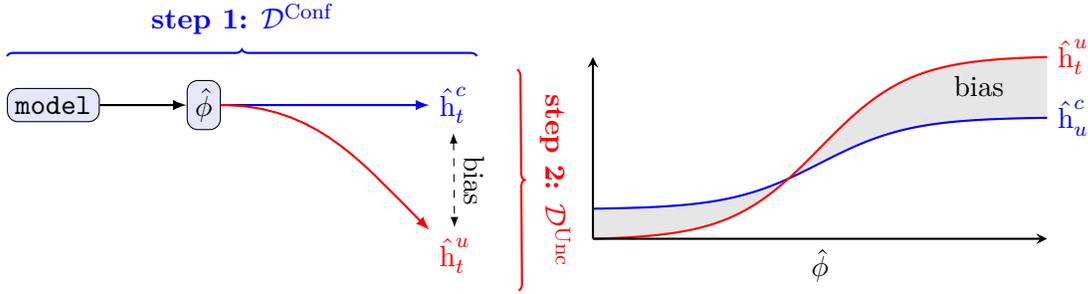

Formally, we consider a two-step empirical risk minimization procedure that leverages observational data $\{(x_i^{\conf}, t_i^{\conf}, y_i^{\conf})\}_{i=1}^{n^{\conf}}$ and randomized data $\{(x_i^{\unc}, t_i^{\unc}, y_i^{\unc})\}_{i=1}^{n^{\unc}}$. For this, we consider the function classes $\RC$, $\HC$, and $\mathcal{B}$ for the representation $\phi$, the confounded hypotheses $\textup{h}_t^c$, and the bias function $\delta_t$.\\

\underline{\textbf{Step~1.}} In the first step, we learn the shared representation and the confounded hypotheses by minimizing the empirical loss of the observational data
\begin{equation}\label{eq:conf_loss_minimizer}
\hat{\phi},\, \hat{\mathbf{h}}^c = \argmin_{\phi\in\Phi, \mathbf{h}^c\in\HC^{\otimes 2}} \hat{\epsilon}_{\conf}(\mathbf{h}^c, \phi), 
\end{equation}
where the empirical loss of the observational data is defined as
\begin{equation}\label{eq:conf_risk}
	\hat{\epsilon}_{\conf}(\mathbf{h}^c, \phi) = \frac{1}{n^{\conf}} \sm i {n^{\conf}}\left(\textup{h}^c_{t_i^{\conf}}\circ\phi(x_i^{\conf}) - y_i^{\conf}\right)^2,
\end{equation}
and $\mathbf{h}^c = (\textup{h}^c_1, \textup{h}^c_0)\in\HC^{\otimes 2}$.\\

\underline{\textbf{Step~2.}} In the second step, we use randomized data together with the estimates from the first step, $\hat{\phi}$ and $\hat{\mathbf{h}}^c$, to learn the bias in the confounded hypothesis by minimizing the empirical loss of the randomized data
\begin{align}\label{eq:unc_loss_minimizer}
    \hat{\boldsymbol{\delta}} = \argmin_{\boldsymbol{\delta}\in\mathcal{B}^{\otimes 2}}\hat{\epsilon}_{\unc}(\boldsymbol{\delta}, \hat{\mathbf{h}}^c, \hat{\phi}),
\end{align}
where the empirical loss of the randomized data is defined as
\begin{align}\label{eq:unc_risk}
\hat{\epsilon}_{\unc}(\boldsymbol{\delta}, \hat{\mathbf{h}}^c, \hat{\phi})= \frac{1}{n^{\unc}} \sm i {n^{\unc}} \left((\hat{\textup{h}}_{t_i^{\unc}}^c + \delta_{t_i^{\unc}}) \circ \hat{\phi}(x_i^{\unc}) - y_i^{\unc}\right)^2,
\end{align}
and $\boldsymbol{\delta} = (\delta_1, \delta_0)\in\mathcal{B}^{\otimes 2}$.\\

Once we obtained the representation $\hat{\phi}$, the confounded hypothesis $\hat{\mathbf{h}}^c$, and the bias function $\hat{\boldsymbol{\delta}}$, the unconfounded hypothesis is then obtained by
\begin{equation}
    \hat{\textup{h}}_t^u = \hat{\textup{h}}_t^c + \hat{\delta}_t.
\end{equation}
Finally, the unconfounded hypotheses yield our estimator for combining observational and randomized data
\begin{align}\label{eq:tau_hat}
    \hat{\tau}_{\cor}(x) = (\hat{\textup{h}}^u_1 -  \hat{\textup{h}}^u_0) \circ \hat{\phi}(x).
\end{align}

\begin{remark}
Note that the first step only requires
observational data, whereas the second step only requires
randomized data. As a result, observational and randomized data do not need to be simultaneously available during training. This is very useful if the two data sources cannot be easily merged, for instance, due to regulatory issues, which is common in medical practice. Our two-step procedure only requires to share the estimated functions, $\hat{\phi}$ and $\hat{\mathbf{h}}^c$.
\end{remark}

The efficacy of an estimator $\hat{\tau}$ is gauged by the precision of estimating heterogeneous effects (PEHE), which is defined as 
\begin{equation}
	\epsilon_{\textup{PEHE}}(\hat{\tau})=\int_{\Rl^d}(\hat{\tau}(x)-\tau(x))^2p_x^{\conf}(x)\mathrm{d}x.
\end{equation}
The density $p_x^{\conf}(x)$ occurs, since the covariate distribution in the observational data is representative of (\ie, the same as) the covariate distribution in the population of interest. In order to study the efficacy of our two-step procedure, we need to make a connection between our two-step procedure and the quantity $\epsilon_{\textup{PEHE}}(\hat{\tau}_{\cor})$. In particular, we derive a finite sample bound on $\epsilon_{\textup{PEHE}}(\hat{\tau}_{\cor})$. This finite sample view provides useful insights into which components influence the convergence of $\epsilon_{\textup{PEHE}}(\hat{\tau}_{\cor})$ and, based on this, how algorithms for combining observational and randomized data should be designed.

\section{Finite Sample Learning Bounds}\label{sec:error_bounds}
In this section, we derive finite sample learning bounds for our two-step procedure. Based on this, we further derive conditions for when it is useful to combine confounded and unconfounded data and for when it is not. In particular, we derive finite sample bounds on $\epsilon_{\textup{PEHE}}(\hat{\tau}_{\cor})$, where $\hat{\tau}_{\cor}$ is the CATE estimator obtained by our two-step procedure (see \Cref{thm_bound}). The finite sample view offers insights into which factors impact the estimation error when the sample size is \emph{not} infinite; in particular, when sample the size is small as in our setting. We find three driving factors of the error: (i)~the size of observational data, (ii)~the distributional discrepancy between the covariate distributions in observational data and randomized data, and (iii)~the complexity of the bias in the confounded hypotheses. Moreover, we derive conditions for when our two-step procedure should be used over standard algorithms, but also when our procedure should not be used (see \Cref{sec::comparison_baselines}). This gives a complete picture of the disadvantages and advantages of our two-step procedure. 

\subsection{Preliminaries}\label{sec::bound_preliminaries}
Throughout this section, we assume that the true shared representation, hypotheses, and bias functions are contained in the function classes $\RC$, $\HC$, and $\mathcal{B}$ over which the two-step procedure optimizes. For this, we make the following standard realizability assumption.
\begin{assumption}\label{assum:realizability}
 (Realizability). The true shared representation $\phi^\ast$ is contained in the functions class $\RC$. Further, the true hypotheses and bias functions, $\mathbf{h}^c = (\textup{h}^c_1, \textup{h}^c_0)$ and $\boldsymbol{\delta} = (\delta_1, \delta_0)$, are contained in $\HC^{\otimes 2}$ and $\mathcal{B}^{\otimes 2}$, respectively.
\end{assumption}
We first introduce an appropriate distance measure in order to measure the difference between the learned representation $\hat{\phi}$ and the true representation $\phi^\ast$ on the confounded data. For this, we introduce the \textit{representation difference}, which measures the extent to which two representation functions $\phi$ and $\phi^\prime$ differ in aggregation across treatment and control group. For a function class $\HC$ and hypotheses $\mathbf{h} = (\textup{h}_1, \textup{h}_0) \in \HC^{\otimes 2}$, the \textit{representation difference} between two representations $\phi, \phi^\prime\in\RC$ is defined as
\begin{align}
d_{\HC, \mathbf{h}}(\phi^\prime;\phi)
= \inf_{\mathbf{h}^\prime\in\mathcal{H}^{\otimes 2}}
\sum_{t=0}^{1}\E\left[\left(\textup{h}^\prime_t \circ\phi^\prime(X^{\conf}) - \textup{h}_t\circ\phi(X^{\conf})\right)^2\right].
\end{align}
The representation difference $d_{\mathcal{B}, \boldsymbol{\delta}}(\phi^\prime; \phi)$ for the function class $\mathcal{B}$ and bias functions $\boldsymbol{\delta} = (\delta_1, \delta_0)\in \mathcal{B}^{\otimes 2}$ is defined analogously. 

The representation difference measures the estimation error due to not exactly learning the true representation. The following conditions ensures that the estimation error due to not exactly learning the true shared representation $\phi^\ast$ is, up to a constant, larger in the hypotheses than in the bias functions.
\begin{condition}\label{cond:rep_diff} $d_{\mathcal{B}, \boldsymbol{\delta}}(\hat{\phi}; \phi^\ast) \leq\gamma\, d_{\HC, \mathbf{h}^c}(\hat{\phi};\phi^\ast)$, for some $\gamma > 0$.
\end{condition}
Note that this only requires that the two representation differences can be set into relation to each other; $d_{\HC, \mathbf{h}^c}(\hat{\phi};\phi^\ast)$ itself does not need to be larger than $d_{\mathcal{B}, \boldsymbol{\delta}}(\hat{\phi}; \phi^\ast)$, only up to a constant $\gamma>0$. In particular, we show later that this conditions holds true for the specific instantiation of the two-step procedure that we propose.

While our procedure works with arbitrary function classes, we consider feedforward neural networks for our representation class $\RC$, hypothesis class $\HC$, and bias function class $\mathcal{B}$. We provide a mathematical definition of feedforward neural networks and an architecture illustration of our two-step procedure using neural networks in \Cref{apx:ffnn}. We also provide the theoretical results proven in this section for general function classes in \Cref{apx:thm_gen_bound}. 

\subsection{Main Result}\label{sec::main_result}
We can now present the main result of this section: the finite sample learning bound for our two-step procedure.
\begin{theorem}(Finite sample learning bound.)\label{thm_bound}
		Let $(\hat{\phi}, \hat{\mathbf{h}}^c)$ be the empirical loss minimizer of $\hat{\epsilon}_{\conf}(\cdot, \cdot)$ from \labelcref{eq:conf_loss_minimizer} over the function classes $\RC$ and $\HC$, and let  $\hat{\boldsymbol{\delta}}$ be the empirical loss minimizer of $\hat{\epsilon}_{\unc}(\cdot, \hat{\mathbf{h}}^c, \hat{\phi})$ from \labelcref{eq:unc_loss_minimizer} over the function class $\mathcal{B}$. Further, let $\hat{\tau}_{\cor}$ be the resulting CATE estimator from \labelcref{eq:tau_hat}. Then, if Assumption~\labelcref{assum:realizability} and Condition~\labelcref{cond:rep_diff} hold true, we have that, with probability at least $1-p$,
\begin{align}
		\epsilon_{\textup{PEHE}}(\hat{\tau}_{\cor})
        \leq \tilde{\mathcal{O}}\bigg(\frac{\mathcal{C}_{\RC} + \mathcal{C}_{\HC}}{\sqrt{n^{\conf}}} + \frac{\mathit{d}_\infty(p^{\conf}_\phi\mid p^{\unc}_\phi)\,\mathcal{C}_{\mathcal{B}}}{\sqrt{n^{\unc}}}\bigg),
\end{align}
	where $\mathcal{C}_{\RC}$, $\mathcal{C}_{\HC}$, and $\mathcal{C}_{\mathcal{B}}$ are constants depending on the complexity of the neural networks and $\mathit{d}_\infty(p^{\conf}_\phi\mid p^{\unc}_\phi) = \sup_{z\in\mathcal{Z}}\frac{p^{\conf}_\phi(z)}{p^{\unc}_\phi(z)}$ is a measure for the distributional discrepancy between $p^{\conf}_\phi$ and $p^{\unc}_\phi$.\footnote{The terms $p^{\conf}_\phi$ and $p^{\unc}_\phi$ are the push-forwards of $p^{\conf}_x$ and $p^{\unc}_x$ through the representation $\phi$. Moreover, $\mathit{d}_\infty(p^{\conf}_\phi\mid p^{\unc}_\phi)$ is the exponential in base 2 of the Rényi divergence $D_\infty(p^{\conf}_\phi\mid p^{\unc}_\phi)$. Note that $\mathit{d}_\infty(p^{\conf}_\phi\mid p^{\unc}_\phi)$ is well-defined, since we assume population overlap in Assumption~\labelcref{assum:standard}.}${}^{,}$\footnote{In \Cref{thm_bound} and throughout this section, we use $\tilde{\mathcal{O}}$ to hide polylogarithmic factors. In particular, in the error bounds, we hide factor such as $\log(n^{\conf})$ and $\log(n^{\unc})$. We do so, since $\mathcal{O}(\log(n^{\conf}))=\mathcal{O}(\sqrt{n^{\conf}})$, but not vice versa. Similar holds true for $\mathcal{O}(\log(n^{\unc}))$. Moreover, using L'Hôspital's rule, $\mathcal{O}(\log(n^{\conf})^2)=\mathcal{O}(\sqrt{n^{\conf}})$.}
\end{theorem}
\begin{proof}
	See \Cref{apx:thm_gen_bound}.
\end{proof}
From \Cref{thm_bound}, we can derive three major insights about the factors that impact the error of our two-step procedure. In particular, we gain insights into how (i)~the size of observational data, (ii)~the distributional discrepancy, and (iii)~the complexity of the bias function impact the error. This is only possible, since we take a finite sample (\ie, non-asymptotic) point of view. In contrast to an infinite sample view, which considers infinitely many samples, a finite sample view allows to study what happens when $n^{\conf}$ and $n^{\unc}$ are varied (\ie, small or large). In particular, an finite sample view allows to study what happens when $n^{\unc}$ is small, which is our setting of interest. We elaborate more on each of the above points in the following. An illustration of the three factors can be found in \Cref{fig:driving_factors}.

\textbf{(i)~The size of the observational data.} Leveraging observational data is beneficial, since the error bound in \Cref{thm_bound} decreases with the size of observational data, \ie, $n^{\conf}$. This is not obvious, since the observational data is biased, and, hence, it is a priori not clear that using biased data improves the estimation error. Moreover, the more observational data we have, the better, since this decreases the error bound in \Cref{thm_bound}. This is particularly favorable in our setting, in which we usually have large amounts of observational data.

\textbf{(ii)~The distributional discrepancy between observational and randomized data.} The distributional discrepancy, $\mathit{\bar{d}}_\infty(p^{\conf}_\phi\mid p^{\unc}_\phi)$, arises since the randomized data does not have the same covariate distribution as the observational data (which coincides with the target distribution). Therefore, since the bias function is learned from randomized data, distributional discrepancy between observational and randomized data impacts the error. In particular, the discrepancy between the observational and randomized data, $\mathit{d}_\infty(p^{\conf}_\phi\mid p^{\unc}_\phi)$, increases the error bound (since, by definition, $\mathit{d}_\infty(p^{\conf}_\phi\mid p^{\unc}_\phi)\geq 1$). This, itself, is an interesting result, since most works on covariate shift derive an additive impact of the distributional discrepancy. This leads to the intuition that we \emph{have to} minimize the discrepancy in order for the error to converge to zero. However, this is not true as we can see in \Cref{thm_bound}. As long as the discrepancy is finite, the error converges to zero in any case, but the discrepancy slows down the convergence. Thus, in a data-rich environment (\ie, where $n^{\unc}$ is large), the impact of the discrepancy may not be substantial. However, and more importantly for our setting, the discrepancy may have a large impact in data-scarce environments (\ie, when $n^{\unc}$ is small as in our setting). Hence, while the discrepancy may be negligible when $n^{\unc}$ is large, it is certainly important to consider it when $n^{\unc}$ is small. As a direct consequence, when designing algorithms to integrate observational data into randomized data, we should bear this in mind and account for the discrepancy. 

\textbf{(iii) The complexity of the bias function.} Based on \Cref{thm_bound}, we gain insight into impact of the complexity of the bias function on the error. In particular, \Cref{thm_bound} suggests that the more complex the bias function (\ie, the larger $\mathcal{C}_{\mathcal{B}}$), the larger the error. This again is crucial, since in a data-scarce regime, such as in RCTs, this term can substantially impact the error. Hence, while the bias may be negligible when $n^{\unc}$ is large, it is certainly important to consider it when $n^{\unc}$ is small. Again, this insight should guide the design of algorithms for combining observational and randomized data.
\begin{remark}\label{rmk:selection_bias}
    Note that, in \Cref{thm_bound}, we do not consider selection bias in the observational data. That is, we consider the case in which treatment and control group in the observational data possess the same covariate distribution. There is a large body of research addressing selection bias in observational data \citep[\eg,][]{Johansson2016,Shalit2017a, Yao2018a,zhang2020}, which is orthogonal to our work. However, since selection bias can arise in observational data, we provide a detailed discussion in \Cref{apx:selection_bias}. In particular, we prove a version of \Cref{thm_bound} which considers selection bias. We find that selection bias has a similar impact as the distributional discrepancy between observational and randomized data. However, if the size of observational data is large, \ie, $n^{\conf}$ is large, selection bias may play a less important role. For more details, see the discussion in \Cref{apx:selection_bias}.
\end{remark}
Guided by the above theoretical insights, we design sample-efficient algorithms of our two-step procedure in \Cref{sec:algorithms}.
\begin{figure}\label{fig:driving_factors}
	\centering 
	\scalebox{0.5}{\hspace{-0.8cm}\includegraphics{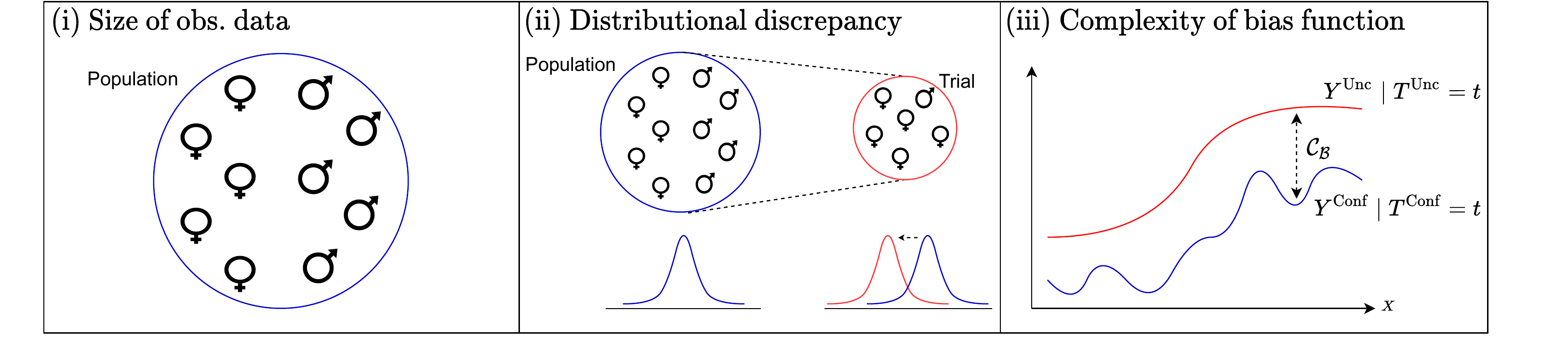}}
    \caption{\footnotesize Illustration of the three factors diving the error bound of $\epsilon_{\textup{PEHE}}$. We increase the number of observational samples. All other parameters are fixed. The solid blue line represent the error for the unregularized two-step procedure. The dashed blue line represent the error for $\hat{\tau}_{\cor}$. We observe that, for both algorithms, $\epsilon_{\textup{PEHE}}$ decreases in the number of observational samples.}
\end{figure}

\subsection{Theoretical Comparison with Baseline Alternatives}\label{sec::comparison_baselines}
In this section, we compare multiple natural baseline alternatives for combining observational and randomized data and theoretically compare them to our two-step procedure. In particular, we consider four baselines approaches: The estimators trained on (i)~only randomized data, denoted as $\tau_{\unc}$, and (ii)~only observational data, denoted as $\tau_{\conf}$, as well as two alternatives, (iii)~the averaging estimators, denoted as $\tau_{\avg}$, and (iv)~the weighted risk estimator, denoted as $\tau_{\weight}$. For each of these approaches, we derive corresponding finite sample learning bounds similar to \Cref{thm_bound}. Based on these results, we can then compare the baselines to our two-step procedure and derive conditions for when combining observational and randomized data is beneficial and for when it is not.

\subsubsection{Estimator on randomized data \texorpdfstring{($\tau_{\unc}$)}{}}\label{sec::tau_unc}
We introduce an approach that only uses randomized data. In order to allow appropriate comparison, we rely again on neural networks for the function classes for the representation, $\RC$, and hypotheses, $\HC$. However, the representation and hypotheses are learned using only randomized data. This is equivalent to learning a neural network (such as TarNet from \citet{Shalit2017a}) on the randomized data. Formally, the representation and the hypotheses are learned by minimizing the empirical loss of the randomized data, \ie, 
\begin{equation}\label{eq:naive_unc_minimizer}
\hat{\phi},\, \hat{\mathbf{h}} = \argmin_{\phi\in\Phi, \mathbf{h}\in\HC^{\otimes 2}} \hat{\epsilon}_{\unc}(\mathbf{h}, \phi), 
\end{equation}
where the empirical loss of the randomized data is given by
\begin{equation}
	\hat{\epsilon}_{\unc}(\mathbf{h}, \phi) = \frac{1}{n^{\unc}} \sm i {n^{\unc}} \left(\textup{h}_{t_i^{\unc}}\circ\phi(x_i^{\unc}) - y_i^{\unc}\right)^2.
\end{equation}
This then yields the following CATE estimator
\begin{equation}\label{eq:tau_hat_naive_unc}
    \hat{\tau}_{\unc} = (\hat{\textup{h}}_1 - \hat{\textup{h}}_0)\circ\hat{\phi}(x).    
\end{equation}
Then, we can derive a finite sample learning bound for $\hat{\tau}_{\unc}$ similar to our two-step procedure.

\begin{theorem}(Finite sample learning bound for $\tau_{\unc}$.)\label{lemma:unconfounded_gen_bnd}
Let $(\hat{\phi}, \hat{\mathbf{h}})$ be the empirical loss minimizer from \labelcref{eq:naive_unc_minimizer} and $\hat{\tau}_{\unc}$ as in \labelcref{eq:tau_hat_naive_unc}. Then, if Assumption~\labelcref{assum:realizability} holds true, we have that, with probability at least $1-p$,
    \begin{align}
        \epsilon_{\textup{PEHE}}(\hat{\tau}_{\unc})\leq\tilde{\mathcal{O}}\bigg(\mathit{\bar{d}}_\infty(p^{\conf}_\phi\mid p^{\unc}_\phi)\frac{\mathcal{C}_{\RC} + \mathcal{C}_{\HC}}{\sqrt{n^{\unc}}}\bigg),
    \end{align}
    where $\mathcal{C}_{\RC}$ and $\mathcal{C}_{\HC}$ are the same constants depending on the complexity of the neural networks as in \Cref{thm_bound} and $\mathit{\bar{d}}_\infty(p^{\conf}_\phi\mid p^{\unc}_\phi) = \sup_{\phi\in\RC}\mathit{d}_\infty(p^{\conf}_\phi\mid p^{\unc}_\phi)$.
\end{theorem}
\begin{proof}
    See \Cref{apx:unconfounded_gen_bnd}.
\end{proof}
The distributional discrepancy, $\mathit{\bar{d}}_\infty(p^{\conf}_\phi\mid p^{\unc}_\phi)$, arises since the representation and hypotheses are estimated on randomized data, which does not have the same covariate distribution as the observational data (which coincides with the target population). For details, we refer to the proof in the \Cref{apx:unconfounded_gen_bnd}.

\subsubsection{Estimator on Observational Data \texorpdfstring{($\tau_{\conf}$)}{}}\label{sec::tau_conf}
Another approach is to ignore any randomized data
and simple rely on observational data to construct an estimator for the CATE \citep[\eg,][]{Johansson2016,Shalit2017a, Alaa2017d}. Again, for fair comparison, we rely on neural networks for the function classes. As such, an estimator that only relies on observational data is identical to the first step of our procedure, but not the second step. For this, let $(\hat{\phi}, \hat{\mathbf{h}}^c)$ be the empirical loss minimizer of $\hat{\epsilon}_{\conf}(\cdot, \cdot)$ from \labelcref{eq:conf_loss_minimizer}. Then, let the CATE estimator which only uses observational data be
\begin{align}\label{eq:tau_hat_conf}
    \hat{\tau}_{\conf}(x) = (\hat{\textup{h}}^c_1 -  \hat{\textup{h}}^c_0) \circ \hat{\phi}(x).
\end{align}
Again, we derive finite sample error bounds for the estimator which only uses observational data, $\hat{\tau}_{\conf}$.
\begin{theorem}(Finite sample learning bound for $\tau_{\conf}$.)\label{thm_gen_bound_conf}
Let $(\hat{\phi}, \hat{\mathbf{h}}^c)$ be the empirical loss minimizer  from \labelcref{eq:conf_loss_minimizer} and $\hat{\tau}_{\conf}$ as in \labelcref{eq:tau_hat_conf}. Then, if Assumption~\labelcref{assum:realizability} holds true, we have that, with probability at least $1-p$,
\begin{align}
		\epsilon_{\textup{PEHE}}(\hat{\tau}_{\conf})
        \leq \tilde{\mathcal{O}}\bigg(\frac{\mathcal{C}_{\RC} + \mathcal{C}_{\HC}}{\sqrt{n^{\conf}}}\bigg)+2\Delta,
	\end{align}
	where $\mathcal{C}_{\RC}$ and $\mathcal{C}_{\HC}$ are the same constants depending on the complexity of the neural networks as in \Cref{thm_bound} and $\Delta = \E[((\delta_1-\delta_0)\circ\phi^\ast(X^{\conf}))^2]$ is the bias due to unobserved confounding.
\end{theorem}
\begin{proof}
	See \Cref{apx:thm_gen_bound_conf}.
\end{proof}
Since $n^{\conf}$ is usually large, the first term in the error bound, $\frac{\mathcal{C}_{\RC} + \mathcal{C}_{\HC}}{\sqrt{n^{\conf}}}$, is small. Hence, the estimation error of $\hat{\tau}_{\conf}$, which only uses observational data, is dominated by the bias term $\Delta$. When there is no bias due to unobserved confounding in the observational data, then $\Delta$ is zero. In this case, $\hat{\tau}_{\conf}$ may be more accurate than $\hat{\tau}_{\unc}$ from \labelcref{eq:tau_hat_naive_unc}, which explains the widespread use of observational data in previous works. However, if the bias is non-zero, then relying on observational data can lead to an irreducible estimation error.

\subsubsection{Averaging Estimator \texorpdfstring{($\tau_{\avg}$)}{}}\label{sec::tau_avg}
Another possible baselines alternative for combining observational and randomized data is to average the estimators from the randomized data, $\hat{\tau}_{\unc}$, and from the observational data, $\hat{\tau}_{\conf}$, \ie,
\begin{align}\label{eq:cate_hat_avg}
    \hat{\tau}_{\avg}(\lambda) = (1-\lambda)\, \hat{\tau}_{\unc} + \lambda\,\hat{\tau}_{\conf},
\end{align}
where $\lambda\in[0,1]$ is a hyperparameter. For $\lambda=0$, we recover $\hat{\tau}_{\unc}$, \ie, the estimator trained only on randomized data. Vice versa, for $\lambda=1$, we obtain $\hat{\tau}_{\conf}$, \ie, the estimator trained only on observational data. For $\lambda\in(0,1)$, this yields a weighted average of $\hat{\tau}_{\unc}$ and $\hat{\tau}_{\conf}$.

\begin{remark}
    Concurrent to our work, this approach was recently also proposed for combining observational and randomized data in \citet{cheng2021adaptive}. In particular, the authors propose to estimate kernel regressions separately on the observational and randomized data, yielding $\hat{\tau}_{\conf}$ and $\hat{\tau}_{\unc}$. Then, both estimators are combined using the parameter $\lambda$ via
    \begin{align}
        \hat{\tau} = \hat{\tau}_{\unc} + \lambda\,(\hat{\tau}_{\conf} - \hat{\tau}_{\unc}),
    \end{align}
    which is equivalent to the approach in \labelcref{eq:cate_hat_avg}. We show in the following that this baseline does (up to constants) not yield an improvement over best choice of $\hat{\tau}_{\conf}$ or $\hat{\tau}_{\unc}$. Moreover, in the experiments in \Cref{sec:experiments}, we show empirically that this approach does indeed not perform substantially better than the two standalone estimators.
\end{remark}
The finite sample learning bound for the averaging estimator, $\hat{\tau}_{\avg}(\lambda)$, is give as follows.
\begin{theorem}(Finite sample learning bound for $\tau_{\avg}$.)\label{thm_gen_bound_avg}
Let $\hat{\tau}_{\avg}(\lambda)$ be as in \labelcref{eq:cate_hat_avg} and $\lambda\in[0,1]$. Then, if Assumption~\labelcref{assum:realizability} holds true, we have that, with probability at least $1-p$,
\begin{align}
		&\epsilon_{\textup{PEHE}}(\hat{\tau}_{\avg}(\lambda))
        \leq \tilde{\mathcal{O}}\bigg((1-\lambda)\mathit{\bar{d}}_\infty(p^{\conf}_\phi\mid p^{\unc}_\phi)\frac{\mathcal{C}_{\RC} + \mathcal{C}_{\HC}}{\sqrt{n^{\unc}}}+\lambda\frac{\mathcal{C}_{\RC} + \mathcal{C}_{\HC}}{\sqrt{n^{\conf}}}\bigg)+2\lambda\Delta,
	\end{align}
    where $\mathcal{C}_{\RC}$ and $\mathcal{C}_{\HC}$ are the same constants depending on the complexity of the neural networks as in \Cref{thm_bound} and $\Delta = \E\left[((\delta_1-\delta_0)\circ\phi^\ast(X^{\conf}))^2\right]$ is the bias due to unobserved confounding.
\end{theorem}
\begin{proof}
	See \Cref{apx:thm_gen_bound_avg}.
\end{proof}
The above error bound is dependent on the hyperparameter $\lambda$. If, hypothetically, we chose the optimal hyperparameter $\lambda$ for every particular data instance, then we would obtain the following error bound, which is strictly smaller than the above result.
\begin{corollary}(Learning bound with optimal $\lambda$.)\label{cor_gen_bound_avg}
Let $\hat{\tau}_{\avg}(\lambda)$ be as in \labelcref{eq:cate_hat_avg}. Then, if Assumption~\labelcref{assum:realizability} holds true, we have that, with probability at least $1-p$,
\begin{align}
		&\min_{\lambda\in[0,1]}\epsilon_{\textup{PEHE}}(\hat{\tau}_{\avg}(\lambda))
        \leq \min\left(\tilde{\mathcal{O}}\bigg(\mathit{\bar{d}}_\infty(p^{\conf}_\phi\mid p^{\unc}_\phi)\frac{\mathcal{C}_{\RC} + \mathcal{C}_{\HC}}{\sqrt{n^{\unc}}}\bigg),\tilde{\mathcal{O}}\bigg(\frac{\mathcal{C}_{\RC} + \mathcal{C}_{\HC}}{\sqrt{n^{\conf}}}\bigg)+2\Delta\right),
	\end{align}
    where $\mathcal{C}_{\RC}$ and $\mathcal{C}_{\HC}$ are the same constants depending on the complexity of the neural networks as in \Cref{thm_bound} and $\Delta = \E\left[((\delta_1-\delta_0)\circ\phi^\ast(X^{\conf}))^2\right]$ is the bias due to unobserved confounding.
\end{corollary}
\begin{proof}
	The proof follows immediately  from \Cref{thm_gen_bound_avg} using $\lambda=\mathbf{1}\{\mathit{\bar{d}}_\infty(p^{\conf}_\phi\mid p^{\unc}_\phi)\frac{\mathcal{C}_{\RC} + \mathcal{C}_{\HC}}{\sqrt{n^{\unc}}}\geq \frac{\mathcal{C}_{\RC} + \mathcal{C}_{\HC}}{\sqrt{n^{\conf}}}+2\Delta\}$.
\end{proof}
Corollary~\labelcref{cor_gen_bound_avg} shows that, even with the optimal hyperparameter $\lambda$, the averaging estimator does not achieve more than a constant factor improvement over the best of $\hat{\tau}_{\unc}$ or $\hat{\tau}_{\conf}$. In particular, the error bound in Corollary~\labelcref{cor_gen_bound_avg} is, up to constant factors, exactly the minimum of the error bounds of the estimator on randomized data (see \Cref{lemma:unconfounded_gen_bnd}) and the estimator on observational data (see \Cref{thm_gen_bound_conf}). Because the averaging estimator spans both of these estimators (depending on the choice of $\lambda$), it is to be expected that the best possible averaging estimator does at least as well as either of these two
estimators; surprisingly, it does no better.

\subsubsection{Weighted Risk Estimator \texorpdfstring{($\tau_{\weight}$)}{}}\label{sec::tau_weight}
In the last section, we studied an estimator that combines two separate estimators, and now we consider and estimator, which is directly learned on a combination of observational and randomized data. That is, we use a weighted regression combining observational and randomized data, but assign a higher weight to randomized data. Formally, again using neural networks for fair comparison, this yields the following weighted risk minimization procedure
\begin{equation}\label{eq:weight_minimizer}
\hat{\phi}, \hat{\mathbf{h}} = \argmin_{\phi\in\Phi, \mathbf{h}\in\HC^{\otimes 2}} \hat{\epsilon}_{\weight}(\mathbf{h}, \phi), 
\end{equation}
where the weighted empirical risk is given by
\begin{align}
    \hat{\epsilon}_{\weight}(\mathbf{h}, \phi) = \frac{1}{\Lambda\, n^{\unc}+n^{\conf}} \bigg(&\Lambda\sm i {n^{\unc}} (\textup{h}_{t_i^{\unc}}\circ\phi(x_i^{\unc}) - y_i^{\unc})^2\notag\\
    &+\sm i {n^{\conf}} (\textup{h}_{t_i^{\conf}}\circ\phi(x_i^{\conf}) - y_i^{\conf})^2\bigg),
\end{align}
where $\Lambda\in[0,\infty)$ is a hyperparameter. For $\Lambda\rightarrow\infty$, we recover $\hat{\tau}_{\unc}$, \ie, the estimator which only uses randomized data. Vice versa, for $\Lambda=0$, we obtain $\hat{\tau}_{\conf}$, \ie, the estimator which only uses observational data.

Then, the CATE estimator which weighted risk minimization procedure is given by
\begin{equation}\label{eq:CATE_hat_weight}
    \hat{\tau}_{\weight}(\Lambda) = (\hat{\textup{h}}_1 - \hat{\textup{h}}_0)\circ\hat{\phi}(x).    
\end{equation}
We derive finite sample error bounds for $\hat{\tau}_{\weight}$ as follows.
\begin{theorem}(Finite sample learning bound for $\tau_{\weight}$.)\label{thm_gen_bound_weight}
Let $(\hat{\mathbf{h}}, \hat{\phi})$ be the empirical loss minimizer from \labelcref{eq:weight_minimizer} and $\hat{\tau}_{\weight}(\Lambda)$ as in \labelcref{eq:CATE_hat_weight}. Then, if Assumption~\labelcref{assum:realizability} holds true, we have that, with probability at least $1-p$,
\begin{align}
		&\epsilon_{\textup{PEHE}}(\hat{\tau}_{\weight}(\Lambda))
        \leq \tilde{\mathcal{O}}\bigg((1-\lambda)\mathit{\bar{d}}_\infty(p^{\conf}_\phi\mid p^{\unc}_\phi)\frac{\mathcal{C}_{\RC} + \mathcal{C}_{\HC}}{\sqrt{n^{\unc}}}+\lambda\frac{\mathcal{C}_{\RC} + \mathcal{C}_{\HC}}{\sqrt{n^{\conf}}}\bigg)+2\lambda\Delta,
	\end{align}
    where $\lambda = \frac{n^{\conf}}{\Lambda n^{\unc}+n^{\conf}}\in[0,1]$. Moreover, $\mathcal{C}_{\RC}$ and $\mathcal{C}_{\HC}$ are the same constants depending on the complexity of the neural networks as in \Cref{thm_bound} and $\Delta = \E\left[((\delta_1-\delta_0)\circ\phi^\ast(X^{\conf}))^2\right]$ is the bias due to unobserved confounding.
\end{theorem}
\begin{proof}
	See \Cref{apx:thm_gen_bound_weight}.
\end{proof}
The above error bound depends on the constant $\lambda = \frac{n^{\conf}}{\Lambda n^{\unc}+n^{\conf}}$. If, hypothetically, we could choose the optimal $\lambda$ (via the hyperparameter $\Lambda$) for every particular data instance, we would obtain the following, strictly smaller, error bound.
\begin{corollary}(Learning bound with optimal $\lambda$.)\label{cor_gen_bound_weight}
Let $\hat{\tau}_{\weight}(\Lambda)$ be as in \labelcref{eq:CATE_hat_weight}. Then, if Assumption~\labelcref{assum:realizability} holds true, we have that, with probability at least $1-p$,
\begin{align}
		&\min_{\Lambda\in[0,\infty)}\epsilon_{\textup{PEHE}}(\hat{\tau}_{\weight}(\Lambda))
        \leq \min\left(\tilde{\mathcal{O}}\bigg(\mathit{\bar{d}}_\infty(p^{\conf}_\phi\mid p^{\unc}_\phi)\frac{\mathcal{C}_{\RC} + \mathcal{C}_{\HC}}{\sqrt{n^{\unc}}}\bigg),\tilde{\mathcal{O}}\bigg(\frac{\mathcal{C}_{\RC} + \mathcal{C}_{\HC}}{\sqrt{n^{\conf}}}\bigg)+2\Delta\right),
	\end{align}
    where $\mathcal{C}_{\RC}$ and $\mathcal{C}_{\HC}$ are the same constants depending on the complexity of the neural networks as in \Cref{thm_bound} and $\Delta = \E\left[((\delta_1-\delta_0)\circ\phi^\ast(X^{\conf}))^2\right]$ is the bias due to unobserved confounding.
\end{corollary}
\begin{proof}
	The proof immediately follows from \Cref{thm_gen_bound_weight} using $\lambda=\mathbf{1}\{\mathit{\bar{d}}_\infty(p^{\conf}_\phi\mid p^{\unc}_\phi)\frac{\mathcal{C}_{\RC} + \mathcal{C}_{\HC}}{\sqrt{n^{\unc}}}\geq \frac{\mathcal{C}_{\RC} + \mathcal{C}_{\HC}}{\sqrt{n^{\conf}}}+2\Delta\}$.
\end{proof}
Corollary~\labelcref{cor_gen_bound_weight} shows that the weighted risk minimization achieves exactly the same bound as the averaging estimator (Corollary~\labelcref{cor_gen_bound_avg}). Thus, the weighted risk minimization also does not achieve more than a constant factor improvement over the best of $\hat{\tau}_{\unc}$ or $\hat{\tau}_{\conf}$. Similar to the averaging estimator, the weighted estimator spans both of these estimators (depending on the choice of $\lambda$) and we would expect that the best possible weighted estimator does at least as well as either of these two estimators; again, surprisingly, it does no better.

\subsubsection{Comparison with Baselines and Improvement Conditions}\label{sec:::comparison_baselines}
In this section, we derive conditions for our two-step procedure should be used. In particular, we compare the error bounds of our estimator against the ones of the baseline alternatives. Based on this, we derive conditions for when our two-step procedure improves upon the baselines and when we should rely on one of the baseline estimators.

For ease of comparison, we tabulate the error bounds (up to constants) in \Cref{tbl:comparison_bounds}. Recall that we are interested in the regime where $n^{\conf}$ is large and $n^{\unc}$ is small, since the latter originates from an RCT. Even with infinite observational samples, the error bound of $\hat{\tau}_{\conf}$ does not vanish due to its bias $\Delta$. The error bound of $\hat{\tau}_{\unc}$ can also be very large, especially when $n^{\unc}< \mathit{\bar{d}}_\infty(p^{\conf}_\phi\mid p^{\unc}_\phi)(\mathcal{C}_{\RC} + \mathcal{C}_{\HC})$. Moreover, since the number of randomized samples is small and fixed, the error bound of $\hat{\tau}_{\unc}$ does not converge to zero the more observational data is used. The estimators $\hat{\tau}_{\avg}$ and $\hat{\tau}_{\weight}$ do not improve upon this by more than a constant factor. Our estimator, $\hat{\tau}_{\cor}$, is the only approach which leverages observational data and enables its error bound to converge to zero the more data is used. 
\begin{table*}[ht!]
	\caption{\footnotesize Comparison of error bounds (up to constants and polylogarithmic factors) across estimators\label{tbl:comparison_bounds}. We denote observational and randomized samples as $n=(n^{\unc}, n^{\conf})$.}
	\begin{center}
				\begin{tabular}{lccc}
					\toprule
					\makecell[l]{Estimator} &Error bound&$\epsilon_{\textup{PEHE}}(\hat{\tau})\underset{n\rightarrow\infty}{\rightarrow} 0$\\\midrule
					\addlinespace[0.75ex]
					\makecell[l]{$\hat{\tau}_{\unc}$} &$\mathit{\bar{d}}_\infty(p^{\conf}_\phi\mid p^{\unc}_\phi)\frac{\mathcal{C}_{\RC} + \mathcal{C}_{\HC}}{\sqrt{n^{\unc}}}$&\cmark\\[0.2cm]
					\addlinespace[0.75ex]	\makecell[l]{$\hat{\tau}_{\conf}$} &$\frac{\mathcal{C}_{\RC} + \mathcal{C}_{\HC}}{\sqrt{n^{\conf}}}+2\Delta$&\xmark \\[0.2cm]
					\addlinespace[0.75ex]
					\makecell[l]{$\hat{\tau}_{\avg}(\lambda)$} &$(1-\lambda)\mathit{\bar{d}}_\infty(p^{\conf}_\phi\mid p^{\unc}_\phi)\frac{\mathcal{C}_{\RC} + \mathcal{C}_{\HC}}{\sqrt{n^{\unc}}}+\lambda\frac{\mathcal{C}_{\RC} + \mathcal{C}_{\HC}}{\sqrt{n^{\conf}}}+2\lambda\Delta$&\xmark\\[0.2cm]	
					\addlinespace[0.75ex]	\makecell[l]{$\hat{\tau}_{\weight}(\Lambda)$} &$(1-\lambda)\mathit{\bar{d}}_\infty(p^{\conf}_\phi\mid p^{\unc}_\phi)\frac{\mathcal{C}_{\RC} + \mathcal{C}_{\HC}}{\sqrt{n^{\unc}}}+2\lambda\frac{\mathcal{C}_{\RC} + \mathcal{C}_{\HC}}{\sqrt{n^{\conf}}}+2\lambda\Delta$&\xmark\\[0.2cm]	\addlinespace[0.75ex]
					\midrule
					\addlinespace[0.75ex]
					\makecell[l]{$\hat{\tau}_{\cor}$} &$\frac{\mathcal{C}_{\RC} + \mathcal{C}_{\HC}}{\sqrt{n^{\conf}}} + \frac{\mathit{d}_\infty(p^{\conf}_\phi\mid p^{\unc}_\phi)\,\mathcal{C}_{\mathcal{B}}}{\sqrt{n^{\unc}}}$& \cmark
					\\[0.2cm]
					\bottomrule
			\end{tabular}
	\end{center}
\end{table*}

Based on these finite sample error bounds, we now derive theoretical conditions for when our estimator yields superior performance compared to the baselines. We compare our estimator \emph{empirically} against the above baseline estimators and further baselines in \Cref{sec:experiments}. For the theoretical conditions, we focus on $\hat{\tau}_{\unc}$ and $\hat{\tau}_{\conf}$. In the following proposition, we provide conditions for when the error of our estimator can improve upon the baseline estimators.
\begin{proposition}(Improvement conditions.)\label{prop:condition}
The error of our estimator, $\hat{\tau}_{\cor}$, can be substantially lower than the error of $\hat{\tau}_{\unc}$, if
\begin{align}
       &\frac{\mathcal{C}_{\RC} + \mathcal{C}_{\HC}}{\mathit{\bar{d}}_\infty(p^{\conf}_\phi\mid p^{\unc}_\phi)(\mathcal{C}_{\RC} + \mathcal{C}_\mathcal{\HC}) -\mathit{d}_\infty(p^{\conf}_\phi\mid p^{\unc}_\phi) \mathcal{C}_\mathcal{B}} < \sqrt{\frac{n^{\conf}}{n^{\unc}}}\tag{C1},\label{eq:c1}
\end{align}
   and substantially lower than the error of $\hat{\tau}_{\conf}$, if
   \begin{align}
       &\frac{\mathit{d}_\infty(p^{\conf}_\phi\mid p^{\unc}_\phi)\,\mathcal{C}_{\mathcal{B}}}{\sqrt{n^{\unc}}} < 2\Delta\tag{C2}.\label{eq:c2}
    \end{align}
\end{proposition}
\begin{proof}
    See \Cref{apx:condition}.
\end{proof}
In the following, we will discuss both conditions.

\textbf{Condition~\labelcref{eq:c1}:} The benefit of combining observational and randomized data (\ie, $\hat{\tau}_{\cor}$) over only using \emph{randomized data} (\ie, $\hat{\tau}_{\unc}$) hinges again on three factors: (i)~the distributional discrepancy, $\mathit{d}_\infty(p^{\conf}_\phi\mid p^{\unc}_\phi)$, (ii)~the size of observational data, $n^{\conf}$, and (iii)~the complexity of the bias function, $\mathcal{C}_\mathcal{B}$:
\begin{enumerate}
    \item[(i)]A large distributional discrepancy acts in favor of satisfying the condition. This may sound counterintuitive at first. However, the larger the distributional discrepancy between observational and randomized data, the more useful observational data becomes for learning the representation, because the covariates in the observational data follow the distribution of interest, while the randomized data does not. In particular, if $\bar{\mathit{d}}_\infty(p^{\conf}_\phi\mid p^{\unc}_\phi)$ and $\mathit{d}_\infty(p^{\conf}_\phi\mid p^{\unc}_\phi)$ increases (but $\bar{\mathit{d}}_\infty(p^{\conf}_\phi\mid p^{\unc}_\phi)>\mathit{d}_\infty(p^{\conf}_\phi\mid p^{\unc}_\phi)$ and, usually, $\mathcal{C}_{\RC} + \mathcal{C}_{\HC}>\mathcal{C}_\mathcal{B}$), a larger distributional discrepancy results in a larger denominator in the left-hand side of \labelcref{eq:c1}. This makes it easier to satisfy condition.  
    \item[(ii)]When $n^{\conf}$ is sufficiently large, but $n^{\unc}$ is relatively small (\ie, precisely our setting), it is also very likely that the condition is satisfied. This is because the representation and the confounded hypotheses can be learned with the observational data and then debiased with the randomized data. In particular, a larger $n^{\conf}$ implies a larger $\sqrt{\frac{n^{\conf}}{n^{\unc}}}$, which makes it easier to satisfy the condition in \labelcref{eq:c1}. This justifies the use of large amounts of observational data in our procedure.
    \item[(iii)]A complex bias function does not act in favor of satisfying the condition. This is because a complex bias function would increase also the left-hand side of \labelcref{eq:c1} via $\mathcal{C}_{\mathcal{B}}$ and, hence, would make it less likely for the condition to be satisfied. If the complexity of the bias function is very large, we are better off only using randomized data. 
\end{enumerate}

\textbf{Condition~\labelcref{eq:c2}:} The benefit of combining observational with randomized data (\ie, $\hat{\tau}_{\cor}$) over only using \emph{observational data} (\ie, $\hat{\tau}_{\conf}$) hinges on the bias due to unobserved confounding, $\Delta$. In other words, if there is substantial unobserved confounding, which results in a large bias and, thus a large $\Delta$, then combining observational and randomized data in our estimator can yield substantially lower errors. Conversely, when the bias is small, using only observational data can be better.

In sum, in our regime of interest, \ie, there exist unobserved confounding, $n^{\conf}$ is large, and $n^{\unc}$ is small, the above conditions are likely be satisfied, and, hence, it is beneficial to combine observational data \emph{and} randomized data using our two-procedure.

\subsubsection{Sample Complexity}
The above results can also be interpreted in terms of sample complexity. That is, how many samples are required to ensure an estimation error below a particular threshold with high probability. To this end, let us consider a decision-maker that aims at a non-trivial estimation error of $\eta<\Delta$. First, note that $\hat{\tau}_{\conf}$ cannot achieve this error even with infinitely many observational samples, since it is lower bounded by $\Delta$. Based on the bounds in \Cref{tbl:comparison_bounds}, it is straightforward to show that the 
estimator on randomized data, $\hat{\tau}_{\unc}$, requires at least 
$\sqrt{n^{\unc}} = \tilde{\mathcal{O}}(\mathit{\bar{d}}_\infty(p^{\conf}_\phi\mid p^{\unc}_\phi)\frac{\mathcal{C}_{\RC} + \mathcal{C}_{\HC}}{\eta})$ regardless of the size of observational data. In contrast, if $n^{\conf}$ is large enough such that $\frac{\mathcal{C}_{\RC} + \mathcal{C}_{\HC}}{\sqrt{n^{\conf}}}\leq \mathit{d}_\infty(p^{\conf}_\phi\mid p^{\unc}_\phi)\frac{\mathcal{C}_{\mathcal{B}}}{\sqrt{n^{\unc}}}$, our two-step procedure only requires $\sqrt{n^{\unc}} = \tilde{\mathcal{O}}(\mathit{d}_\infty(p^{\conf}_\phi\mid p^{\unc}_\phi)\frac{2\mathcal{C}_{\mathcal{B}}}{\eta}$). In other words, if sufficient observational data is available and the bias function is less complex than the representation and hypotheses, \ie, $2\mathcal{C}_{\mathcal{B}}<\mathcal{C}_{\RC}+\mathcal{C}_{\HC}$, then the number of randomized samples required to achieve $\eta$ is substantially smaller. This is particularly interesting, since in order to achieve the same error as $\tau_{\unc}$, we need less randomized data. As a directly implication, less costly randomized data needs to be acquired and less subjects enrolled in an RCT, which could substantially accelerate drug safety and evaluation.

\section{CorNet: an Algorithm for Combining Observational and Randomized Data}\label{sec:algorithms}
We propose an instantiation of our two-step procedure for estimating CATE. Our algorithm for \textbf{C}ombining \textbf{O}bservational and \textbf{R}andomized data is based on neural \textbf{Net}works, and, thus, we call it \textbf{CorNet}. Guided by the theoretical insights from \Cref{sec:error_bounds}, we propose a sample-efficient algorithm of CorNet, which first learns a \emph{balanced}\footnote{A balanced representation is a representation in which the distribution of $\phi(X^\conf)$ and $\phi(X^\unc)$ are similar with respect to some distributional discrepancy measure.} shared representation and confounded hypothesis using observational data and, then, debiases the confounded hypotheses by learning the bias functions while regularizing its complexity. In particular, we let $\RC$ be parametrized by neural networks, while $\HC$ and $\mathcal{B}$ are parametrized by linear functions, \ie, $\textup{h}^c_{t}\circ\phi^\ast(x) = \phi^\ast(x)\,\mathbf{w}^c_t$ and $\textup{h}^u_{t}\circ\phi^\ast(x) = \phi^\ast(x)(\mathbf{w}^c_t+\boldsymbol{\delta}_t)$, where $\mathbf{w}^c_t, \boldsymbol{\delta}_t\in\Rl^{d_\phi}$. Moreover, we prove that this instantiation satisfies Condition~\labelcref{cond:rep_diff}.

\subsection{Learning a Balanced Representation} 
In the first step, we learn a shared representation and the confounded hypotheses using observational data as proposed in \Cref{sec::two_step_procedure}. As discussed in \Cref{sec::main_result}, the distributional discrepancy between observational and randomized covariate distributions is an important factor in the finite sample error bound (\Cref{thm_bound}), which should be accounted for. For this, we seek a representation $\phi$ and a confounded hypothesis $\mathbf{h}^c$ which minimize the trade-off between predictive accuracy and distributional discrepancy. That is, $\phi$ (together with $\mathbf{h}^c$) should be predictive of the confounded outcome $Y^{\conf}$, and the distribution of $X^{\unc}$ and $X^{\conf}$ in the representation space should be balanced.\footnote{In particular, this means that the push-forwards of the covariate distributions $p^{\unc}_x$ and $p^{\conf}_x$ through $\phi$ should be similar.} Balancing of the distributions is achieved by regularizing for a distributional metric (either by directly computing the metric, or by using adversarial learning) in the learning procedure. However, since we only observe a small sample from the randomized data, standard approaches are not directly applicable. That is because estimating the distributional metric based on scarce data is difficult. While any approach for balancing the distributions can be chosen, we propose to combine adversarial learning with data augmentation similar to \citet{wang2019semi}. We refer to \Cref{sec:practical_considerations} for an in-depth discussion.

\subsection{Regularizing the Complexity of the Bias Function Class} 
In the second step, we use the shared representation and confounded hypothesis from in the first step together with randomized data to debias the confounded hypotheses. That is, we learn the bias function $\boldsymbol{\delta}$. 

As discussed in \Cref{sec::main_result}, the complexity of the bias function class is an important factor in the finite sample error bound (\Cref{thm_bound}) and, thus, should be accounted for in our algorithm. For this, we regularize the bias function, which imposes structure on the bias function. Regularizing the bias has two advantages: (i)~it decreases the error bound in \Cref{thm_bound}, since we decrease the complexity of the bias function, $\mathcal{C}_{\mathcal{B}}$. As such, we may achieve faster lower error with less data, which is particularly favorable in our setting. (ii)~The second step of our procedure is a high-dimensional inference problem, since the representation is often high-dimensional and the randomized sample size is small. As such, the inference of the parameters $\boldsymbol{\delta}$ requires regularization, since the small sample size prohibits inference otherwise \citep{meinshausen2006high,buhlmann2011statistics}. In particular, we use the $L_1$-norm to regularize the bias function, which imposes sparsity on the bias function (in the representation) and renders the second step as a Lasso regression.

Two questions arise: (i)~why not directly learning the unconfounded hypotheses $\mathbf{w}^u$ and (ii)~why not using $L_2$-regularization to circumvent the high-dimensional inference problem? Both are discussed in the following.
\begin{itemize}
    \item[(i)]Directly learning the unconfounded hypotheses (without further structural restrictions) is prohibited by the small sample size of the randomized data. Therefore, some structural restrictions via a regularization have to be imposed. However, imposing such restrictions (\eg, sparsity) on the hypotheses directly may be more restrictive than imposing restrictions on the difference between the confounded and unconfounded hypotheses (\ie, the bias). That is since observational and randomized data are closely related, which may allow to impose more structure on the bias function.
    \item[(ii)]While $L_2$-regularization would also circumvent the high-dimensional inference problem, the choice of regularization (\ie, $L_1$- or $L_2$-regularization) depends on the implicit assumptions we make on the data-generating process. That is, whether we implicitly assume that unobserved confounding in the observational data affects a subset of entries of $\boldsymbol{\delta}$ (\ie, $\boldsymbol{\delta}$ is sparse in $\phi(X)$), and, hence, we use a $L_1$-regularization. On the contrary, the $L_2$-regularization implicitly assumes that all covariates in the representation $\phi(X)$ influence the bias function $\boldsymbol{\delta}$. However, we argue that there are often unobserved confounders that systematically affect a subset of the features, creating a bias between the observational and randomized data. For instance, the difference in the patient records from different hospitals (such as hospitals from the RCT and hospitals from the OS) may arise in the difference of how some diagnoses are recorded, but likely not how all diagnoses are recorded \citep{bastani2021predicting}. Therefore, the bias induces by these differences can be explained by the few covariates corresponding to the subset of diagnoses where differences arise. Note that both regularizations are feasible and the decision which one is used has to be made by the user. In the experiments in \Cref{sec:experiments}, we conduct an ablation study on the type of regularization and find that, indeed, for most datasets, an $L_1$-regularization yields better results, but in some cases, an $L_2$-regularization works better.
\end{itemize}

\subsection{CorNet Algorithm}\label{sec:theoretical_properties}
In this section, we present the actual algorithm for CorNet and discuss some of its theoretical properties. We defer a in-depth theoretical discussion to \Cref{apx:theoretical_properties_algo}. We present the algorithm of CorNet in Algorithm~\labelcref{alg:CORNet}. 
\begin{algorithm*}[ht]
	\SetKwInOut{Input}{Input}
	\SetKwInOut{Output}{Output}
	\SetAlgoLined
	\Input{Observational data $\{(x_i^{\conf}, t_i^{\conf}, y_i^{\conf})\}_{i=1}^{n^{\conf}}$ and randomized data $\{(x_i^{\unc}, t_i^{\unc}) y_i^{\unc}\}_{i=1}^{n^{\unc}}$}\vspace{0.1cm}
	\Output{$\hat{\phi},\, \hat{\mathbf{w}}^c = (\mathbf{w}^c_1, \mathbf{w}^c_0),\,\hat{\boldsymbol{\delta}}=(\hat{\boldsymbol{\delta}}_1, \hat{\boldsymbol{\delta}}_0)$}\break\hfill\break
	\underline{Step 1}. Compute
	\begin{align}
		\hat{\phi}, \hat{\mathbf{w}}^c = \argmin_{\phi\in\RC, \mathbf{w}^c\in{\Rl^{d_\phi\otimes 2}}} \frac1{n^{\conf}}\sm i {n^{\conf}} (\phi(x^{\conf}_i)\mathbf{w}^c_{t_i^{\conf}} - y^{\conf}_i)^2\\ 
		+\lambda_d\, \mathcal{D}(\{\phi(x_i^{\conf})\}_{i=1}^{n^{\conf}}, \{\phi(x_i^{\unc})\}_{i=1}^{n^{\unc}})
	\end{align}
	\underline{Step 2}. Compute
	\begin{align}
	\hat{\boldsymbol{\delta}} = \argmin_{\boldsymbol{\delta}\in\Rl^{d_\phi\otimes 2}} \frac1{n^{\unc}}\sm i {n^{\unc}} (\hat{\phi}(x_i^{\unc})(\hat{\mathbf{w}}^c_{t_i^{\unc}} + \boldsymbol{\delta}_{t_i^{\unc}})-y^{\unc}_i)^2+  \lambda_\delta \left\|\boldsymbol{\delta}\right\|_1
	\end{align}
	\caption{\textbf{CorNet algorithm}}\label{alg:CORNet}
\end{algorithm*}

Most theoretical results from \Cref{sec:error_bounds} remain in tact, since only the complexities of the function classes change (see \Cref{apx:theoretical_properties_algo} for details). Here, we prove that Condition~\labelcref{cond:rep_diff} is satisfied for CorNet. 
\begin{proposition}\label{cor:cond}
	For the function classes as in Algorithm~\labelcref{alg:CORNet}, \ie, neural networks for $\RC$ and linear functions for $\HC$ and $\mathcal{B}$, the condition $d_{\mathcal{B}, \boldsymbol{\delta}}(\hat{\phi}; \phi^\ast) \leq\gamma\, d_{\HC, \mathbf{h}^c}(\hat{\phi};\phi^\ast)$, for some $\gamma > 0$ is satisfied.
\end{proposition}
\begin{proof}
	See \Cref{apx:cor_cond}.
\end{proof}
Once we obtain $\hat{\phi}$, $\hat{\mathbf{w}}^c$, and $\hat{\boldsymbol{\delta}}$, the estimator for the CATE is given by
\begin{align}
    \hat{\tau}_{\textup{CorNet}}(x) = \hat{\phi}(x)\,(\hat{\mathbf{w}}^c_1+\hat{\boldsymbol{\delta}}_1 - \hat{\mathbf{w}}^c_0-\hat{\boldsymbol{\delta}}_0).
\end{align}
In Algorithm~\labelcref{alg:CORNet}, $\mathcal{D}$ denotes the distributional discrepancy between observational and randomized data. The specific expression depends on the method for balancing and can be chosen by the user. We detail our choice in \Cref{sec:practical_considerations}. Further, the hyperparameter $\lambda_d$ trades off the predictive accuracy and the distributional discrepancy in the first step, and $\lambda_\delta$ regularizes the complexity of the bias function in the second step.
\begin{remark}
    Note that, when $\lambda_\delta\rightarrow \infty$, we recover the estimator that only uses observational data. This can be seen by recognizing that $\hat{\boldsymbol{w}}^u = \hat{\boldsymbol{w}}^c + \boldsymbol{\delta}$ and $\boldsymbol{\delta}\rightarrow \mathbf{0}$ for $\lambda_\delta\rightarrow \infty$.
\end{remark}
\begin{remark}
    In Algorithm~\labelcref{alg:CORNet}, the hyperparameter $\lambda_d$ and $\lambda_\delta$ must be chosen. This can be done by hyperparameter optimization (\eg, grid search) using a validation set. However, since $n^{\unc}$ is small, we do not want to sacrifice any randomized samples for the hyperparameter optimization of $\lambda_d$ and $\lambda_\delta$. For the choice of $\lambda_\delta$, we recognize that the second step is a Lasso regression and, based on high-dimensional statistics theory, the optimal choice of the regularization parameter is $\lambda_\delta = \sqrt{c_0\log(d_\phi)/\log(n^{\unc})}$ \citep[\eg,][]{buhlmann2011statistics}. In particular, in our experiments, the constant $c_0$ is set to $1/{d_\phi}$. Moreover, we choose $\lambda_d = \lambda_\delta$, since the distributional discrepancy has the same impact on the finite sample error as the complexity of the bias functions. Hence, we attribute the same importance to both factors.
\end{remark}

\subsection{Practical Considerations}\label{sec::practical_considerations}
In this section, we discuss two practical considerations of our algorithm. First, we need to ensure that the bias function is identifiable from the randomized data, and, second, we detail how a balanced representation can be learned with only few randomized samples.

\subsubsection{Identifiability of the Bias Function}
Here, we discuss the identifiability of the bias function $\boldsymbol{\delta}$ in the second step. For this, we assume that the so-called ``compatibility conditions'', a condition from the theory of high-dimensional statistics \citep[\eg,][]{buhlmann2011statistics}, is satisfied. We discuss this condition in the following.

In the second step of our Algorithm~\labelcref{alg:CORNet}, we are given an estimate of the representation function $\hat{\phi}$. Based on this, we have to ensure that $\boldsymbol{\delta}_t$ can be identified from $\hat{\phi}(\mathbf{X}_t^{\unc})$, where $\mathbf{X}_t^{\unc}$ is the (unconfounded) design matrix for all samples with $T=t$ for $t\in\{0,1\}$. For this, we define the (unconfounded) sample covariance matrix
\begin{align}
    \boldsymbol{\Sigma}^{^{\unc}}_t({\hat{\phi}}) = \frac{1}{n^{^{\unc}}}\hat{\phi}(\mathbf{X}^{\unc}_t)^\top\hat{\phi}(\mathbf{X}^{\unc}_t).
\end{align}
In order to ensure that $\boldsymbol{\delta}_t$ is identifiable from $\hat{\phi}(\mathbf{X}^{\unc}_t)$, we could assume that $\hat{\phi}(\mathbf{X}^{\unc}_t)$ is positive-definite, \ie, for any vector $\mathbf{v}\in\Rl^{d_\phi}$, $\mathbf{v}^\top\hat{\phi}(\mathbf{X}^{\unc}_t)\mathbf{v}>0$. This, however, is infeasible in our setting, since $n^{\unc}$ is small and, hence, $n^{\unc}<d_\phi$. Therefore, there exists collinearity in features. The compatibility condition is a strictly weaker condition, since, for instance, it allows for collinearity in features outside of the index set\footnote{An index set $S$ is a set $S\subseteq\{1, \ldots,d\}$. Then, for any vector $\mathbf{v}\in\Rl^d$, the vector $\mathbf{v}_S\in\Rl^d$ is obtained by setting the entries in $\mathbf{v}$ which are not in $S$ to zero. Furthermore, $S^c$ denotes the complement of the index set.} $S_t=\text{supp}(\boldsymbol{\delta}_t)$\footnote{The support of a vector $\mathbf{v}\in\Rl^d$, denoted as supp$(\mathbf{v})\subseteq\{1, \ldots, d\}$, is the set of indices corresponding to nonzero entries of $\mathbf{v}$.}, which is often the case in a high-dimensional setting such as ours.

We assume that the following compatibility condition holds for the (unconfounded) sample covariance matrix, $\boldsymbol{\Sigma}^{^{\unc}}_t(\hat{\phi})$, and the index set $S_t=\text{supp}(\boldsymbol{\delta}_t)$, for each $t\in\{0,1\}$.

\textbf{Compatibility Condition.} The compatibility condition is satisfied if for an index set $S\subseteq\{1, \ldots, d\}$ and a matrix $\boldsymbol{\Sigma}\in \Rl^{d\times d}$, there exists $\rho>0$ such that, for all $\mathbf{v}\in\Rl^d$ satisfying $\lVert \mathbf{v}_{S^c}\rVert_1 \leq 3 \lVert \mathbf{v}_{S} \rVert_1$, it holds that
\begin{equation}\label{eq:comp_cond}
    \lVert \mathbf{v}_S\rVert^2_1 \leq \frac{\lvert S\rvert}{\rho^2}\mathbf{v}^\top \mathbf{v}.
\end{equation}
This assumption is standard in the high-dimensional statistics literature to ensure the convergence of high-dimensional estimators such as the Dantzig selector or Lasso \citep{candes2007dantzig, bickel2009simultaneous, buhlmann2011statistics}.

Note that we make this assumption for the bias function $\boldsymbol{\delta}_t$ and not the unconfounded hypothesis $\mathbf{w}_t^u$ directly. We do so, since we only need to estimate the bias vector $\boldsymbol{\delta}_t$ in the second step of our algorithm. In particular, if the bias is sparse enough, \ie, $\lVert \boldsymbol{\delta}_t\rVert_1$ is small, the compatibility condition in \labelcref{eq:comp_cond} is satisfied.

Furthermore, the compatibility condition is always satisfied if $\boldsymbol{\Sigma}^{^{\unc}}_t({\hat{\phi}})$ is positive-definite. This can be seen by letting $\psi>0$ the minimum eigenvalue of $\boldsymbol{\Sigma}^{^{\unc}}_t({\hat{\phi}})$. Then, the compatibility condition is satisfied with the constant $\rho_0=\sqrt{\psi}$ for an index set. Hence, the compatibility condition is strictly weaker than positive definiteness. This is particularly relevant, since positive definiteness for the unconfounded covariance matrix might not hold true due to the small samples size of the randomized data. As such, even if $\mathbf{w}^u_t$ is not identifiable, we may be able to identify the bias function $\boldsymbol{\delta}_t$.


\subsubsection{Balancing Representation via Augmented Distribution Alignment}\label{sec:practical_considerations}
In Algorithm~\labelcref{alg:CORNet}, we regularize for the distributional discrepancy between the covariate distributions in the representation space, $\mathcal{D}(\{\phi(x_i^{\conf})\}_{i=1}^{n^{\conf}}, \{\phi(x_i^{\unc})\}_{i=1}^{n^{\unc}})$. We do this, since the distributional discrepancy is a driving factor in the learning bound from \Cref{thm_bound}. Regularizing for the distributional discrepancy, so-called balancing, is not straightforward in our setting, since the small sample size of the randomized data may cause instability in optimization and leads to performance degradation. While several approaches for balancing exist from the fields of domain adaptation \citep{ganin2015unsupervised,ganin2016domain} and causal inference \citep[\eg,][]{Shalit2017a, zhang2020}, these approaches assume large sample sizes and, therefore, are not directly applicable in our setting.

As a remedy, we adopt an adversarial distribution alignment approach using augmented data \citep{wang2019semi}. For this, \citet{wang2019semi} propose to iteratively generate new samples by interpolating between the observational and the randomized samples similar to the \textit{mixup} approach for supervised learning \citep{zhang2017mixup}. This has two advantages: First, the new samples enlarge the training data set, which leads to a more stable optimization (particularly for neural networks). Second, since the new samples are generated by interpolating between observational and randomized samples, the distribution of the interpolated samples is expected to be closer to the observational distribution as proven in \citet{wang2019semi}. We use these new samples to augment the existing randomized samples and, based on this, minimize the distributional discrepancy using adversarial learning.

Mathematically, let $x^{\unc}$ and $x^{\conf}$ be observational and randomized covariate samples. Then, we sample from the following density function
\begin{equation}\label{eq:mixup_distribution}
    p_\mu(x\mid x^{\unc}, x^{\conf}) = \mathbf{1}_{\{\mu\, x^{\unc} + (1-\mu)\, x^{\conf}\}}(x),
\end{equation}
where $\mathbf{1}_{\{\cdot\}}(\cdot)$ is the indicator function and $\mu\sim$Beta($\alpha$, $\alpha$), with $\alpha>0$. Sampling from this density is equivalent to linearly interpolating between $x^{\unc}$ and $x^{\conf}$ via
\begin{align}
	\tilde{x} = \mu\, x^{\unc} + (1-\mu)\, x^{\conf},
\end{align}
where, for each sample $x^{\conf}$, a sample $x^{\unc}$ is randomly chosen from all randomized samples and $\mu\sim$ Beta($\alpha$, $\alpha$) determines the degree of interpolation. The interpolated sample is denoted by $\tilde{x}$.


Based on the samples from the above density, $\{\tilde{x}_i\}_{i=1}^m$, for large enough $m$, we balance the empirical distribution of $\tilde{\mathcal{U}} = \{\phi(\tilde{x}_i)\}_{i=1}^m$ and $\mathcal{U}^{\conf}=\{\phi(x_i^{\conf})\}_{i=1}^{n^{\conf}}$. Here, both $m$ and $n^{\conf}$ are large.\footnote{In the experiments, $m$ is set to $n^{\conf}$.} In order to measure the distributional discrepancy, we use the $H$-divergence between interpolated and confounded samples as inspired by domain adaptation \citep{ben2010theory}. We optimize the $H$-divergence, $\hat{d}_H(\tilde{\mathcal{U}}, \mathcal{U}^{\conf})$, using adversarial learning and a gradient reverse layer~(GRL) \citep{ganin2015unsupervised}.\footnote{For details on the optimization of the $H$-divergence, see \Cref{apx:h_divergence_details}.} This renders step 1 in Algorithm~\labelcref{alg:CORNet} as
\begin{align}
    \argmin_{\phi\in\RC, \mathbf{w}^c\in{\Rl^{d_\phi\otimes 2}}} \frac1{n^{\conf}}\sm i {n^{\conf}} (\phi(x^{\conf}_i)\mathbf{w}^c_{t_i^{\conf}} - y^{\conf}_i)^2
		+\lambda_d\, \hat{d}_H(\tilde{\mathcal{U}}, \mathcal{U}^{\conf}),
\end{align}
where $\lambda_d$ trades off the predictive accuracy and the distributional discrepancy. Therefore, augmented distribution alignment can be easily incorporated by appending a GRL and adding the proposed interpolated samples during the optimization.

\section{Experiments}\label{sec:experiments}
In this section, we perform extensive experiments which demonstrate the superior performance of our CorNet. For this, we conduct simulation studies in which we verify the theoretical properties of CorNet (see \Cref{sec:sim_study}) and experiments using real-world data in which we compare the estimation performance against a variety of baselines (see \Cref{sec:real_world_data}). During this section, we refer to the estimator yield by the unregularized variant of CorNet (\ie, $\lambda_d=\lambda_\delta=0$) by $\tau_{\textup{CorNet}}$ and to its regularized variant (\ie, $\lambda_d,\lambda_\delta>0$) by $\tau_{\textup{CorNet}^+}$.

\subsection{Simulation Studies: Empirical Verification of Theoretical Properties}\label{sec:sim_study}
Here, we empirically verify the theoretical properties of our procedure derived in \Cref{sec:error_bounds}. For this, we use simulated data, since this allows to manipulate the data-generating processes. In particular, we need to manipulate the data-generating process in order to study the impact of different factors as in \Cref{sec:error_bounds}. Note that for this, we use the \emph{unregularized} variant, $\tau_{\textup{CorNet}}$, since regularizations may avoid observing the effect of certain factors.

We conduct the following three simulation studies. In \Cref{sec::sim_learning_bounds}, we empirically  the finite sample properties derived in \Cref{thm_bound}, \ie, the impact of the size of observational data, the distributional discrepancy, and the complexity of the bias on the error of our algorithm. In \Cref{sec:::sim_condition}, we empirically verify the conditions derived in \Cref{prop:condition} for when our algorithm should be used and when is should not. In \Cref{sec:::sim_assum_violations}, we study the robustness of our algorithm against alterations of our setup (\ie, the representation is not shared, the observational data is unconfounded, and no population overlap). We find that our algorithm yields robust results even when the underlying setup is altered.

\subsubsection{Finite Sample Properties}\label{sec::sim_learning_bounds}
Here, we empirically investigate the finite sample properties of our algorithm CorNet. As we discussed in \Cref{sec::main_result}, there are three driving factors of interest in the finite sample error bound: (i)~the size of observational data, $n^{\conf}$, (ii)~the distributional discrepancy between observational and randomized data, $\mathit{d}_\infty(p^{\conf}_\phi\mid p^{\unc}_\phi)$, and (iii)~the complexity of the bias, $\mathcal{C}_{\mathcal{B}}$. In particular, we have seen that from a theoretical point of view: more observational data is beneficial and large distributional discrepancy and large complexity of the bias function are  not beneficial for the error of our algorithm. Therefore, we conduct three simulation studies to show the impact of these three factors on the finite sample error. Note that, in order to control the bias due to unobserved confounding, we do not introduce an unobserved confounder that affects treatment assignment and outcome. Instead, we directly manipulate the data-generating processes of the randomized and observational outcomes, $Y^{\unc}, Y^{\conf}$. This is permitted, since the unobserved confounding leads to a difference in $Y^{\unc}, Y^{\conf}$ (see \Cref{sec::dgp}), which gives direct access to controlling the bias.

We generate observational and randomized data following the data-generating process from \Cref{sec::dgp}. In particular, let $X^{\unc} \sim \mathcal{N}(\mathbf{0}, \sigma_u^2\mathbf{1})$ and $X^{\conf} \sim \mathcal{N}(\mathbf{0}, \mathbf{1})$, where $\sigma_u^2$ controls the distributional discrepancy between the observational and randomized data. Furthermore, let the treatment assignment\footnote{Note that $T^{\conf}$ does not depend on the covariates $X^{\conf}$, since the impact of selection bias (\ie, the covariate shift between treatment and control group) is orthogonal to this work. See discussion in \Cref{rmk:selection_bias} and \Cref{apx:selection_bias} for details.}$^{,}$\footnote{Note that we do not introduce bias due to unobserved confounding by conditioning $T^{\conf}$ on the potential outcomes or an unobserved confounder. Instead we introduced bias by directly manipulating the DGP of $Y^{\unc}, Y^{\conf}$ (see discussion at beginning of \Cref{sec:sim_study}).} be $T^{\unc}, T^{\conf}\sim \text{Bernoulli}(\frac{1}{2})$ and the outcomes be 
\begin{align}\label{eq:informal_bound}
    Y^{\conf} \sim \phi^\ast(X^{\conf})\,\mathbf{w}_{T^{\conf}}^c + \epsilon,\\
    Y^{\unc} \sim \phi^\ast(X^{\unc})\,\mathbf{w}_{T^{\unc}}^u + \epsilon,
\end{align}
where $\epsilon\sim\mathcal{N}(0, \sigma_\epsilon^2)$. We use randomly chosen weights for $\phi^\ast$ and $\mathbf{w}_t^c$ and control the bias between observational and randomized data via $\mathbf{w}_t^u = \mathbf{w}_t^c + \boldsymbol{\delta}_t$, by changing $\lVert \boldsymbol{\delta}_t\rVert = \beta$.
\begin{figure}
	\centering 
	\scalebox{0.55}{\includegraphics{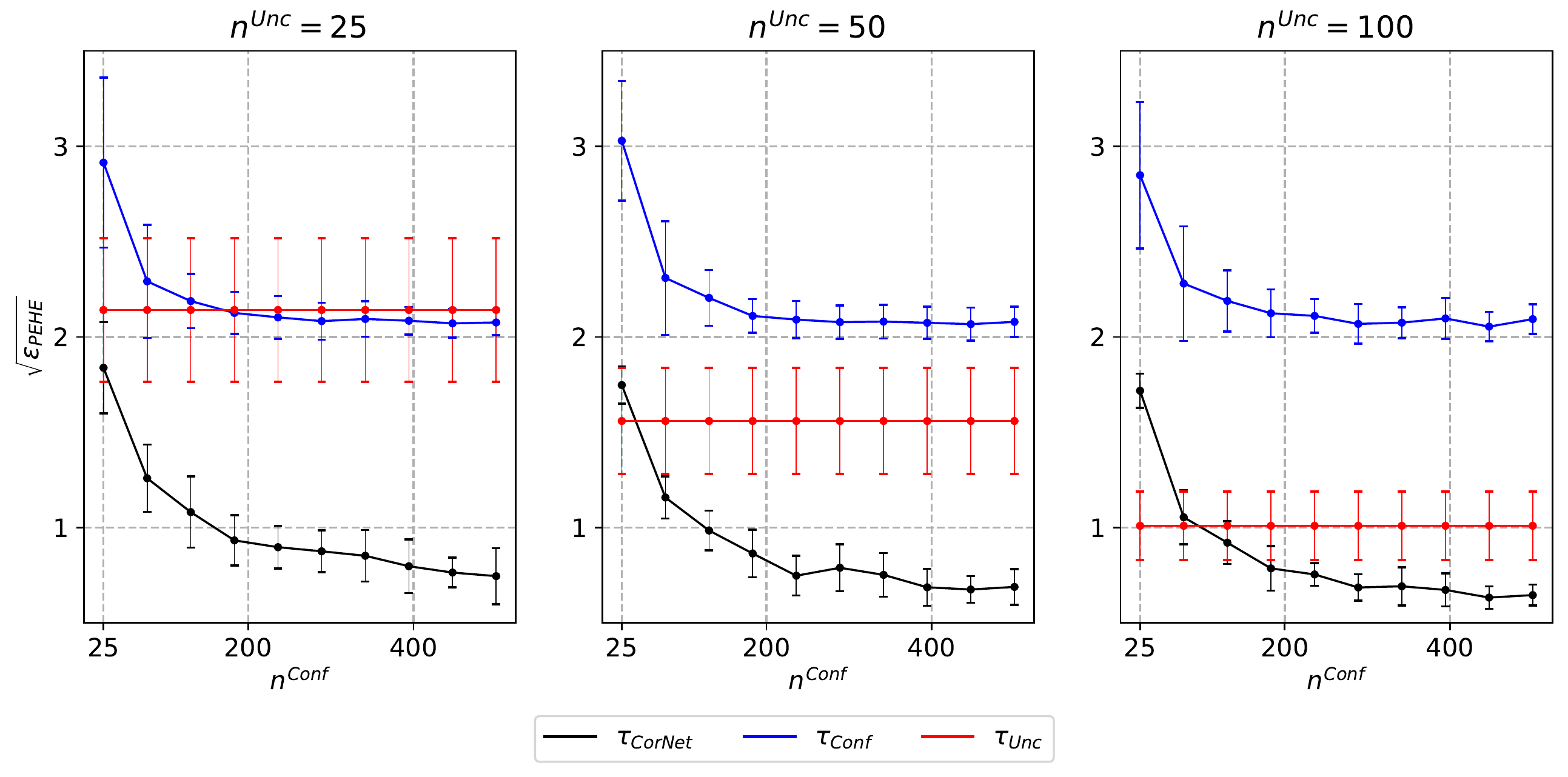}}
    \caption{\footnotesize Simulation study on the impact of the size of observational data on the error, $\sqrt{\epsilon_{\textup{PEHE}}}$. We compare our algorithm, $\tau_{\textup{CorNet}}$, against the algorithm that only uses randomized data, $\tau_{\unc}$, and the algorithm that only uses observational data, $\tau_{\conf}$. The size of observational data is increased ($x$-axis). All other parameters are fixed. The mean and standard deviation (over 10 samples of the data-generating process) are shown. We observe that, for $\tau_{\textup{CorNet}}$, $\sqrt{\epsilon_{\textup{PEHE}}}$ decreases in the number of observational samples, \ie, $n^{\conf}$.}\label{fig:sim_n_conf}
\end{figure}
\begin{enumerate}
    \item[\textbf{(i)}]\textbf{Impact of size of observational data.} Here, we study the impact of the size of observational data on the finite sample error of our algorithm. In order to isolate the impact of the observational data, we fix all parameters and vary $n^{\conf}$. In particular, we set $n^{\unc} \in \{25, 50, 100\}$ and $d_\infty(p^{\conf}\mid p^{\unc}) = 1$ (\ie, $\sigma_u^2=1$). Although we are interested in the error across different values for $n^{\conf}$, we show the result across difference values for $n^{\unc}$ as a sensitivity check. The parameter $\beta$ is chosen such that the bias in the observational data is constant across the different settings, \ie, $\sqrt{\epsilon_{\textup{PEHE}}(\tau_{\conf})} = 2$.

    In \Cref{fig:sim_n_conf}, we present the error of our method (\ie, $\tau_{\textup{CorNet}}$), the estimator which only uses randomized data (\ie, $\tau_{\unc}$), and the estimator which only uses observational data (\ie, $\tau_{\conf}$) for increasing $n^{\conf}$. We are interest in whether more observational data (\ie, larger $n^{\conf}$) is beneficial for our algorithm.

    We make the following three observations. (1)~The insight derived from the result in \Cref{thm_bound} holds true empirically: using observational data is beneficial, since the error decreases the more observational data is used (\ie, larger $n^{\conf}$). This is not obvious, since the observational data is biased. (2)~Compared to $\tau_{\unc}$, our method achieves smaller error for large enough $n^{\conf}$. Moreover, the standard deviation of our method, $\tau_{\textup{CorNet}}$, is much smaller than the standard deviation of $\tau_{\unc}$. This is because the sample size of the randomized data (\ie, $n^{\unc}$) is small. Our method achieves much smaller standard deviation, since we incorporate observational data, which is available in larger quantities and, hence, reduce the variance. If the sample size of the randomized data is increased, for instance, from the far left plot (with $n^{\unc}=25$) to the far right plot (with $n^{\unc}=100$), the benefit of observational data remains. However, the error of $\tau_{\unc}$ decreases (due to more data) and, in order to outperform $\tau_{\unc}$, our method requires more observational data. Hence, since observational data is usually available in large quantities, it remains beneficial to combine observational and randomized data, even when more randomized data is acquired. However, this effect may diminish at some point when the size of randomized data becomes very large. (3)~Compared to $\tau_{\conf}$, our method achieves smaller error independent of $n^{\conf}$. Moreover, the error of $\tau_{\conf}$ stagnates at $\sqrt{\epsilon_{\textup{PEHE}}(\tau_{\conf})} = 2$, which is exactly the square root of the bias in the observational data (see earlier in this section). If the sample size of the randomized data is increased (far left plot to far right plot), the error of $\tau_{\conf}$ does not change, since $\tau_{\conf}$ ignores any randomized data.

\item[\textbf{(ii)}]\textbf{Impact of distributional discrepancy.} Here, we study the impact of the distributional discrepancy, $d_\infty(p^{\conf}\mid p^{\unc})$, on the finite sample error. Such a setting arises in practice when the covariate distribution of the subjects in the RCT does not fit the covariate distribution of the subject in the target population. For this, again, we fix all parameters (\ie, $n^{\unc}=50$ and $\beta$ such that $\sqrt{\epsilon_{\textup{PEHE}}(\tau_{\conf})} = 2$). Then, we compare the error in the case in which there is no distributional discrepancy (\ie, $d_\infty(p^{\conf}\mid p^{\unc})=1$) to the case in which there is distributional discrepancy (\ie, $d_\infty(p^{\conf}\mid p^{\unc})>1$). For the former case, we set $\sigma_u=1$, \ie, $X^{\unc}\sim\mathcal{N}(0, 1\cdot\mathbf{1})$, and, for the latter case, we set $\sigma_u=0.1$, \ie, $X^{\unc}\sim\mathcal{N}(0, 0.1^2\cdot\mathbf{1})$.

The results are presented in \Cref{sim_res_dist}. We make the following three observations. (1)~If $d_\infty(p^{\conf}\mid p^{\unc}) > 1$ (\ie, $\sigma_u=0.1$), the error converges slower for our algorithm compared to when $d_\infty(p^{\conf}\mid p^{\unc}) = 1$. This is expected from the finite sample error bound, as the distributional discrepancy is a driving factor in \Cref{thm_bound}. Hence, as our theory suggests, we observe that there is a negative impact of the distributional discrepancy between observational and randomized covariate distributions on the error. (2)~If $d_\infty(p^{\conf}\mid p^{\unc}) > 1$, the performance of the estimator which only uses randomized data, $\tau_{\unc}$, worsens. This is expected from the finite sample error bound of $\tau_{\unc}$ (see \Cref{lemma:unconfounded_gen_bnd}), in which the distributional discrepancy is a driving factor. In particular, we observe that the impact of distributional discrepancy is more pronounced for $\tau_{\unc}$ than for our algorithm, $\tau_{\textup{CorNet}}$ (the error of $\tau_{\unc}$ worsens more than the error of $\tau_{\textup{CorNet}}$). This is because the distributional discrepancy impact the whole finite sample error bound of $\tau_{\unc}$ (see \Cref{lemma:unconfounded_gen_bnd}), but only parts of the finite sample error bound of $\tau_{\textup{CorNet}}$ (see \Cref{thm_bound}). (3)~The error of $\tau_{\conf}$ is not affected by the distributional discrepancy. Again, this is suggested by our theory (see \Cref{thm_gen_bound_conf}), since the term $d_\infty(p^{\conf}\mid p^{\unc})$ does not occur in the finite sample error bound.
\begin{figure}
	\centering 
	\scalebox{0.5}{\includegraphics{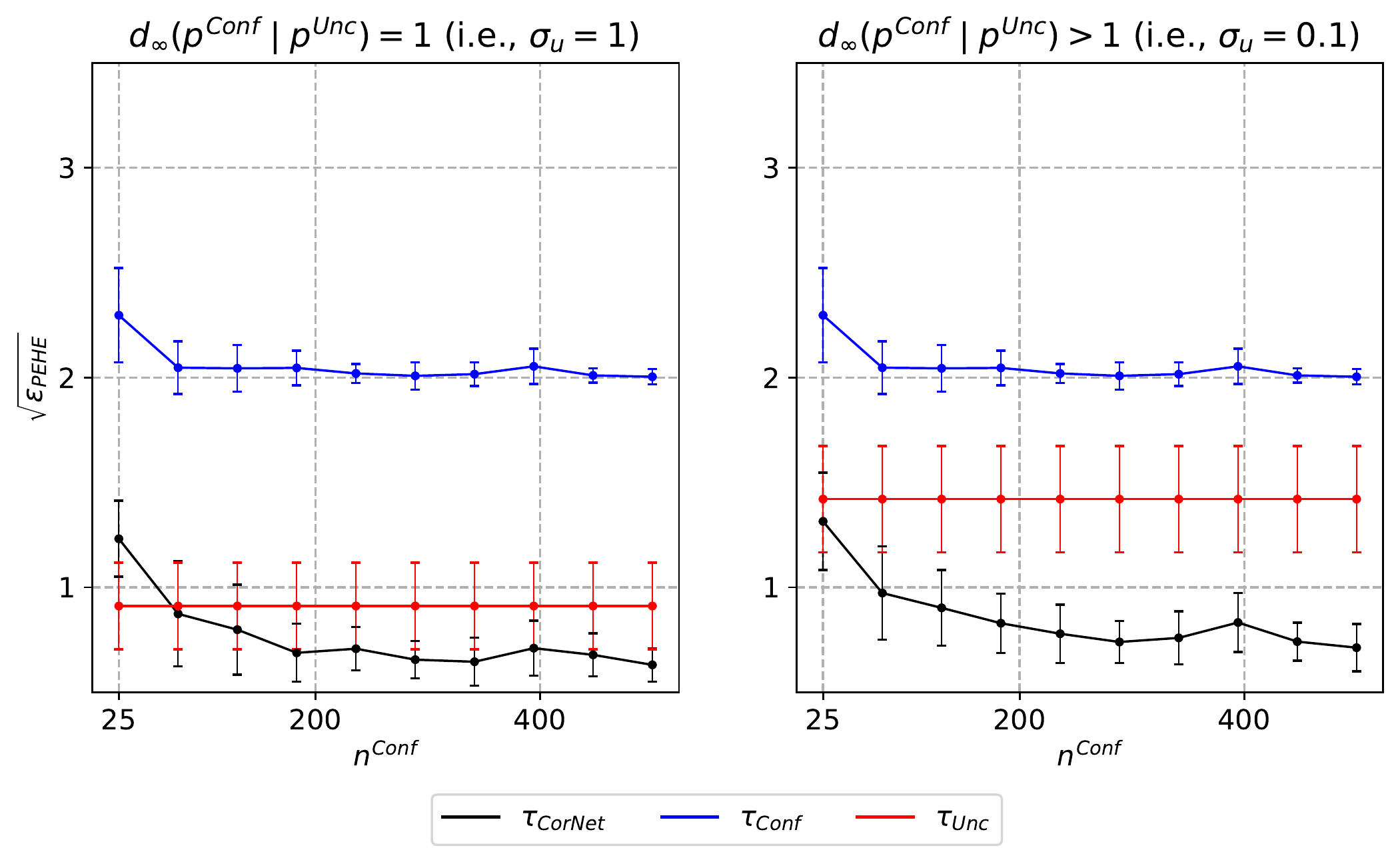}}
	\caption{\footnotesize Simulation study on the impact of the distributional discrepancy on the error, $\sqrt{\epsilon_{\textup{PEHE}}}$ (left plot: no discrepancy, right plot: discrepancy). We compare our algorithm, $\tau_{\textup{CorNet}}$, against the algorithm that only uses randomized data, $\tau_{\unc}$, and the algorithm that only uses observational data, $\tau_{\conf}$. The size of observational data is increased ($x$-axis). All other parameters are fixed. The mean and standard deviation (over 10 samples of the data-generating process) are shown. We observe that, for $\tau_{\textup{CorNet}}$ and $\tau_{\unc}$, the distributional discrepancy negatively impacts the convergence of $\sqrt{\epsilon_{\textup{PEHE}}}$.}\label{sim_res_dist}
\end{figure}
\item[\textbf{(iii)}]\textbf{Impact of the complexity of bias.} Finally, we study the impact of the complexity of the bias function, $\mathcal{C}_{\mathcal{B}}$, on the finite sample error. For this, we fix all parameters (\ie, $n^{\unc}=50$ and $\mathit{d}_\infty(p^{\conf}\mid p^{\unc})=1$, \ie, no distributional discrepancy). Then, we compare the error in the case in which the complexity is small (\ie, $\mathcal{C}_{\mathcal{B}}$ small) to the case in which the complexity is large (\ie, $\mathcal{C}_{\mathcal{B}}$ large). The complexity of the bias function is measured in terms of its Gaussian complexity and, hence, can be controlled by varying $\lVert \boldsymbol{\delta}_t\rVert=\beta$ (\ie, larger $\beta$ corresponds to larger Gaussian complexity).

The results are presented in \Cref{fig:sim_res_bias}. We make the following three observations. (1)~The error of our algorithm, $\tau_{\textup{CorNet}}$, increases in the complexity of the bias function (from the left to the right plot). This is in line with our theoretical results in \Cref{thm_bound}, in which the complexity of the bias function, $\mathcal{C}_{\mathcal{B}}$, is one of the driving factors. Moreover, the convergence of the error is substantially slower the more complex the bias function is. In particular, for large bias complexity, the error remains substantially larger than in the case with smaller bias complexity. This observation is, again, expected by our finite sample error bound in \Cref{thm_bound}. For both settings, our method improves upon the estimators $\tau_{\unc}$ and $\tau_{\conf}$. (2)~The estimator which only uses randomized data, \ie, $\tau_{\unc}$, is not impacted by the complexity of the bias, since it is trained on only randomized (\ie, unbiased) data. It, therefore, remains identical across the two plots. (3)~The estimator which only uses observational data, \ie, $\tau_{\conf}$, achieves smaller error in the case in which the bias complexity is small compared to the case in which bias complexity is large. This is due to the fact that, when the complexity of the bias is increased (by scaling $\beta$), the bias itself also becomes larger.
\begin{figure}
	\centering 
	\scalebox{0.5}{\includegraphics{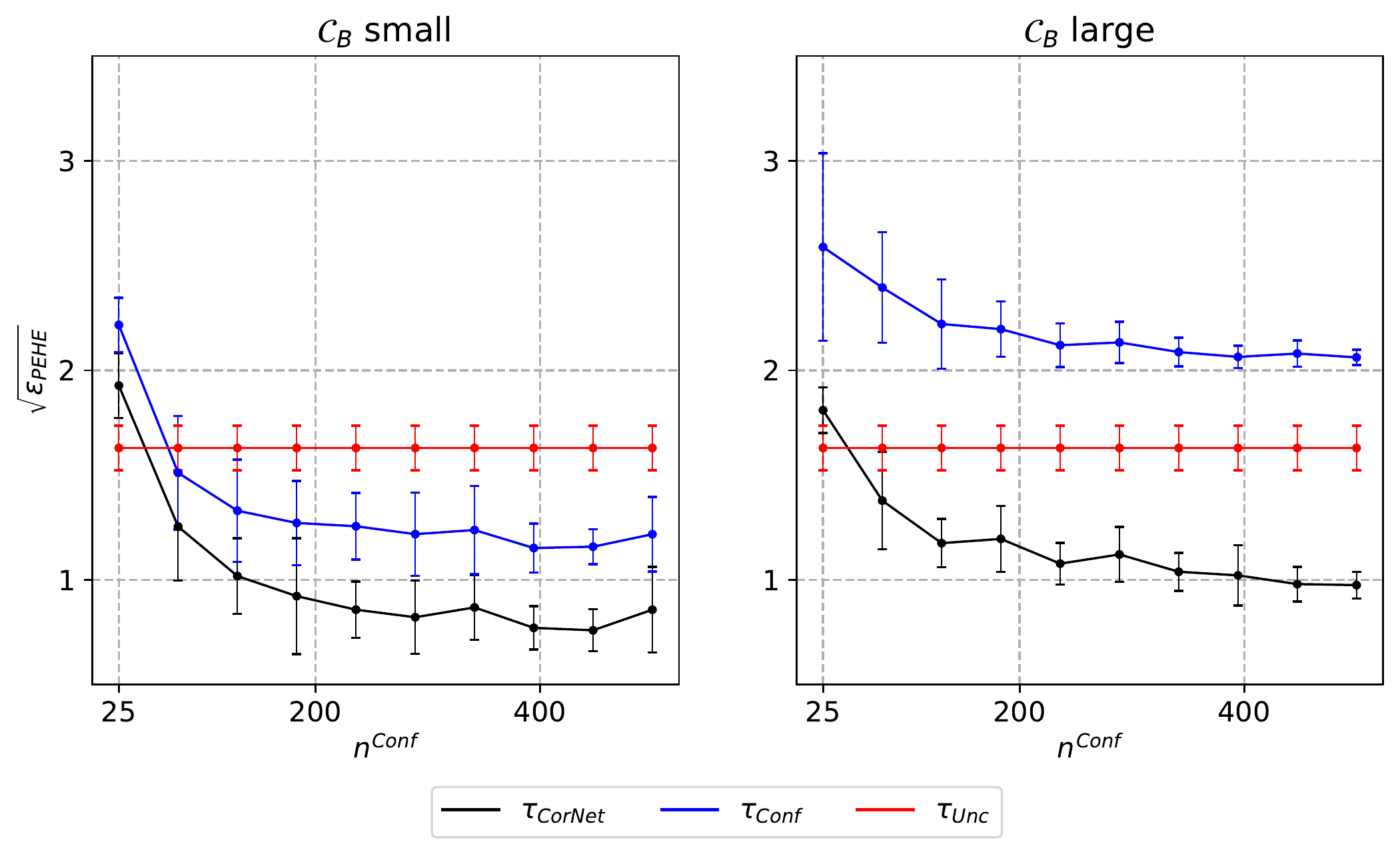}}
	\caption{\footnotesize Simulation study on the impact of the complexity of the bias function on the error, $\sqrt{\epsilon_{\textup{PEHE}}}$ (left plot: small complexity, right plot: large complexity). We compare our algorithm, $\tau_{\textup{CorNet}}$, against the algorithm that only uses randomized data, $\tau_{\unc}$, and the algorithm that only uses observational data, $\tau_{\conf}$. The size of observational data is increased ($x$-axis). All other parameters are fixed. The mean and standard deviation (over 10 samples of the data-generating process) are shown. We observe that the complexity of the bias negatively impacts the convergence of $\sqrt{\epsilon_{\textup{PEHE}}}$.}\label{fig:sim_res_bias}
\end{figure}
\end{enumerate}
\subsubsection{Improvement Conditions for our Algorithm}\label{sec:::sim_condition}
In this section, we empirically verify the theoretical conditions derived in Proposition~\labelcref{prop:condition} in \Cref{sec:::comparison_baselines} for when our method, $\tau_{\textup{CorNet}}$, should be used over $\tau_{\unc}$ and $\tau_{\conf}$ and when it should not. For this, we compare $\epsilon_{\textup{PEHE}}(\tau_{\textup{CorNet}})$ against $\epsilon_{\textup{PEHE}}(\tau_{\unc})$ and $\epsilon_{\textup{PEHE}}(\tau_{\textup{CorNet}})$ against $\epsilon_{\textup{PEHE}}(\tau_{\conf})$ across different settings. 

For the comparison against $\epsilon_{\textup{PEHE}}(\tau_{\unc})$, we are interested in when our method is more beneficial than $\hat{\tau}_{\unc}$ and how the size of the observational data, the distributional discrepancy, and the complexity of the bias affect this. For this, we compare across three settings: (1)~no distributional discrepancy and small bias complexity (\ie, $d_\infty(p^{\conf}\mid p^{\unc}) = 1$, $\mathcal{C}_{\mathcal{B}}$ small), (2)~distributional discrepancy and small bias complexity (\ie, $d_\infty(p^{\conf}\mid p^{\unc}) > 1$, $\mathcal{C}_{\mathcal{B}}$ small), and (3)~no distributional discrepancy and large bias complexity (\ie, $d_\infty(p^{\conf}\mid p^{\unc}) = 1$, $\mathcal{C}_{\mathcal{B}}$ large). The distributional discrepancy and the complexity of the bias are controlled similarly as in the simulation studies conducted earlier in this section.


In \Cref{fig:sim_res_cond}, we present the results. We note that we observe empirically what we found theoretically in \Cref{sec:::comparison_baselines}. We make the following three observations. First, the larger the size of the observational data (and, therefore, the larger the ratio $\sqrt{\frac{n^{\conf}}{n^{\unc}}}$), the more beneficial (in terms of error) our method is. We can observe this in \Cref{fig:sim_res_cond} (top), where the vertical dashed line indicates the point when our method, $\tau_{\textup{CorNet}}$, improve upon $\tau_{\unc}$. Second, the larger the distributional discrepancy, the more beneficial our method is. We can observe this in \Cref{fig:sim_res_cond} (middle), where the vertical dashed line moves to the left compared to the plot in the first row (\ie, less data required). Third, the larger the complexity of the bias function, the more observational data our algorithm needs to offset its negative effect and to be beneficial over $\tau_{\unc}$. We can observe this in \Cref{fig:sim_res_cond} (bottom), where the vertical dashed line moves to the right (\ie, more data required).
\begin{figure}[ht]
	\centering 
	\scalebox{0.85}{\includegraphics{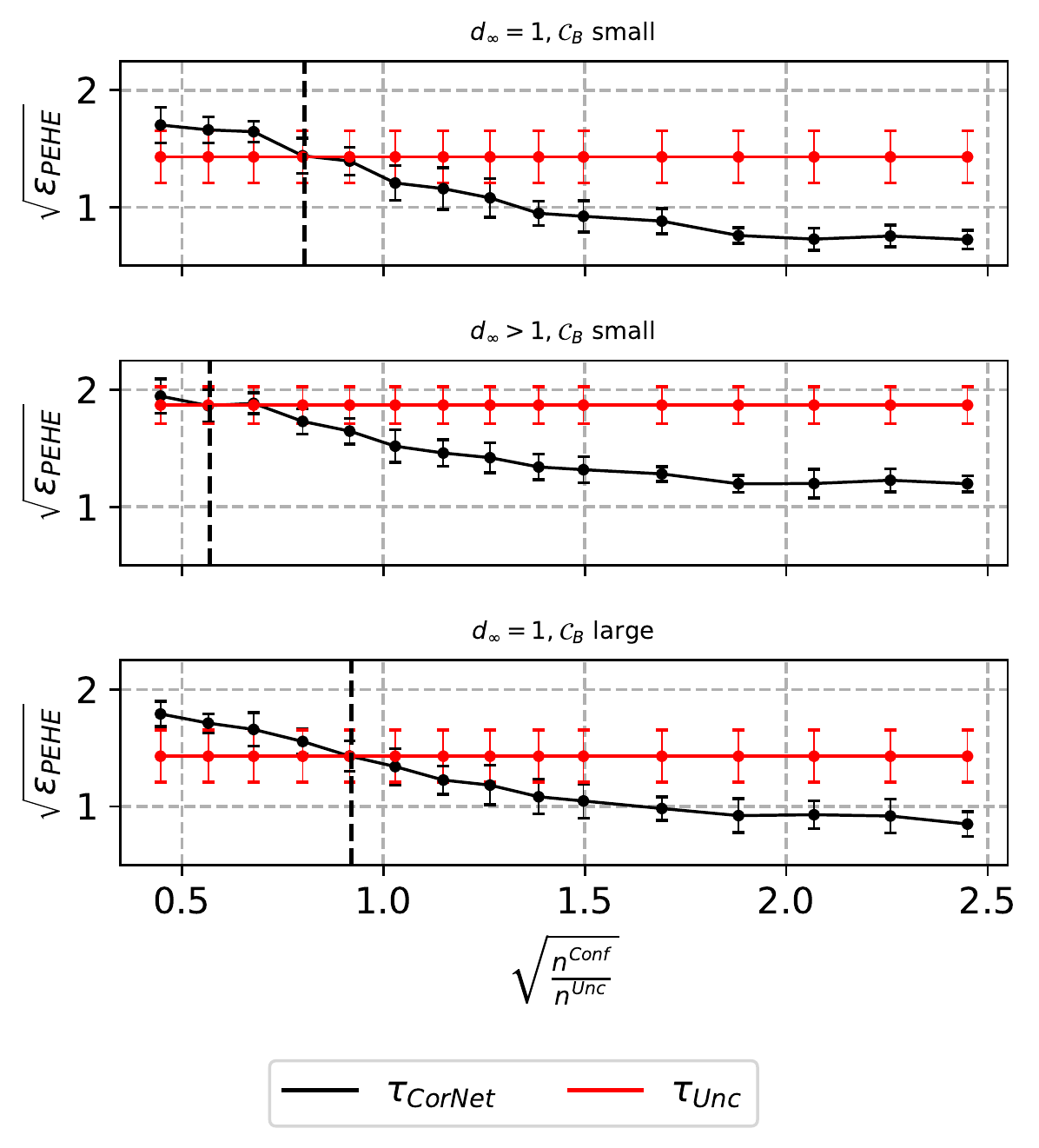}}
	\caption{\footnotesize Simulation study on the conditions for when our algorithm, $\tau_{\textup{CorNet}}$, is beneficial over the algorithm which only uses randomized data, $\tau_{\unc}$. Top:~$d_\infty(p^{\conf}\mid p^{\unc}) = 1$ and $\mathcal{C}_{\mathcal{B}}$ small. Middle:~$d_\infty(p^{\conf}\mid p^{\unc}) > 1$ and $\mathcal{C}_{\mathcal{B}}$ small. Bottom:~$d_\infty(p^{\conf}\mid p^{\unc}) = 1$ and $\mathcal{C}_{\mathcal{B}}$ large. The ratio $\sqrt{\frac{n^{\conf}}{n^{\unc}}}$ is increased. All other parameters are fixed. We observe that: (1)~the larger the size of the observational data, the more beneficial our algorithm becomes. (2)~The larger the distributional discrepancy, the fewer observational samples are needed to improve upon $\tau_{\unc}$. (3)~The larger the bias, the more observational data is required to improve upon $\tau_{\unc}$.}\label{fig:sim_res_cond}
\end{figure}

For the comparison against $\epsilon_{\textup{PEHE}}(\tau_{\conf})$, we are interested in when our method, $\tau_{\textup{CorNet}}$, is more beneficial than $\tau_{\conf}$ and how the size of the randomized data, the distributional discrepancy, and the bias affect this. Similar to before, we compare across three settings: (1)~no distributional discrepancy and small bias (\ie, $d_\infty(p^{\conf}\mid p^{\unc}) = 1$, $\Delta=1$), (2)~distributional discrepancy and small bias (\ie, $d_\infty(p^{\conf}\mid p^{\unc}) > 1$, $\Delta=1$), and (3)~no distributional discrepancy and larger bias (\ie, $d_\infty(p^{\conf}\mid p^{\unc}) = 1$, $\Delta = 1.5$). The distributional discrepancy and the bias are controlled similarly as in the simulation studies conducted earlier in this section.

In \Cref{fig:sim_res_cond_conf}, we present the results. We note that we again observe empirically what we found theoretically in \Cref{sec:::comparison_baselines}. First, the larger the size of the randomized data (and, therefore, the smaller the ratio $\sqrt{1/n^{\unc}}$), the more beneficial (in terms of error) our method becomes. We can observe this in \Cref{fig:sim_res_cond_conf} (top), where the vertical dashed line indicates the point when our method, $\tau_{\textup{CorNet}}$, improves upon $\tau_{\conf}$. Second, the larger the distributional discrepancy, the more randomized data is required to offset the negative effect from the distributional discrepancy and to be beneficial over $\tau_{\conf}$. We can observe this in the \Cref{fig:sim_res_cond_conf} (middle), where the vertical dashed line moves to the right compared to the plot at the top (\ie, more data required). Third, the larger the bias, the more beneficial our algorithm becomes over $\tau_{\conf}$. We can observe this in \Cref{fig:sim_res_cond_conf} (bottom), where the vertical dashed line moves to the left (\ie, less data required).
\begin{figure}[ht]
	\centering 
	\scalebox{0.85}{\includegraphics{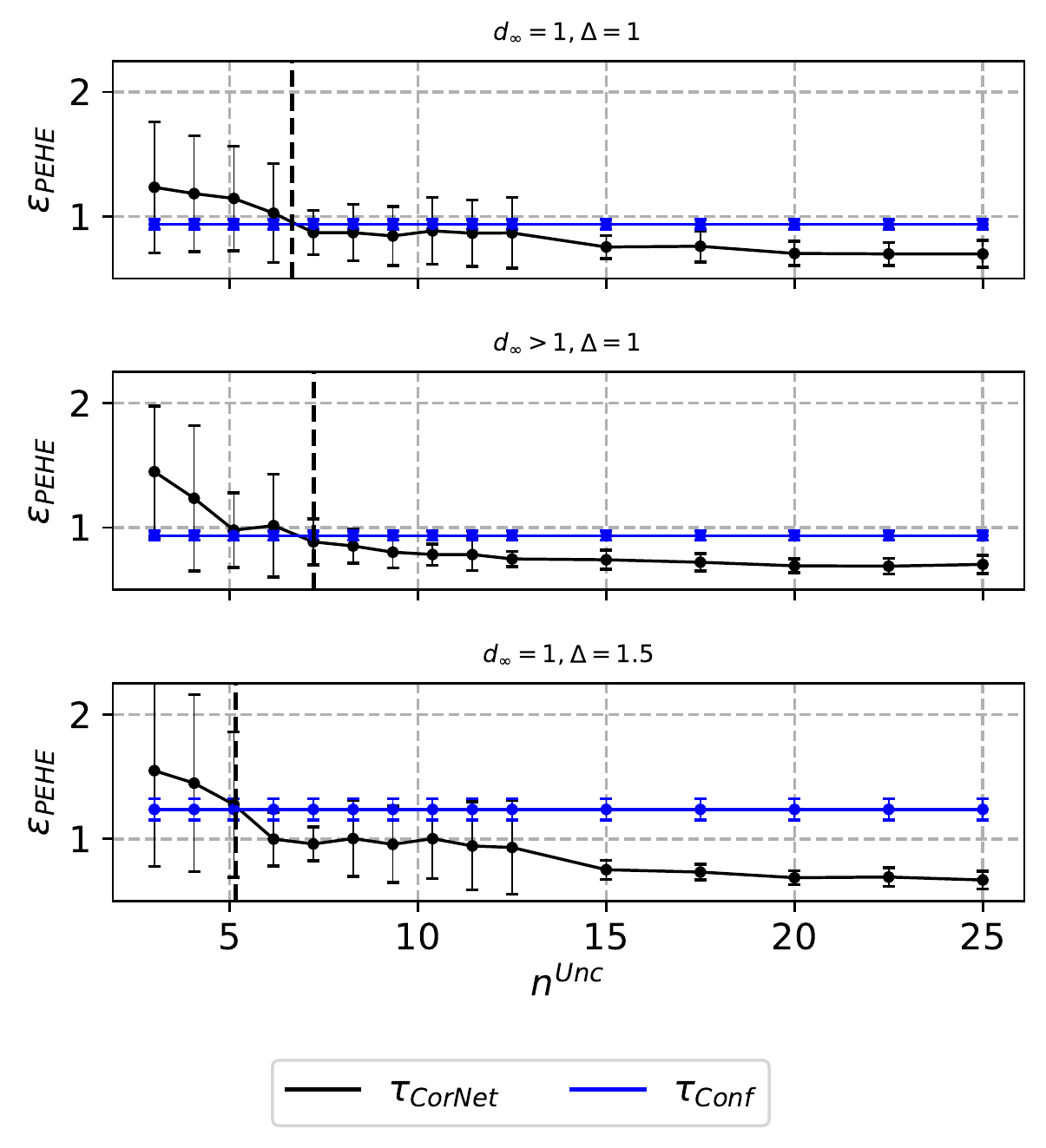}}
	\caption{\footnotesize Simulation study on the conditions for when our algorithm, $\tau_{\textup{CorNet}}$, is beneficial over the algorithm which only uses observational data, $\tau_{\conf}$. Top:~$d_\infty(p^{\conf}\mid p^{\unc}) = 1$ and $\Delta=1$ Middle:~$d_\infty(p^{\conf}\mid p^{\unc}) > 1$ and $\Delta=1$. Bottom:~$d_\infty(p^{\conf}\mid p^{\unc}) = 1$ and $\Delta=1.5$. The number of randomized samples $n^{\unc}$ is increased. All other parameters are fixed. We observe that: (1)~the larger the size of the randomized data, the more beneficial our algorithm becomes (2)~The larger the distributional discrepancy, the more randomized samples are needed to improve upon $\tau_{\conf}$. (3)~The larger the bias, the fewer randomized data is required to improve upon $\tau_{\conf}$.}\label{fig:sim_res_cond_conf}
\end{figure}

\subsubsection{Robustness against alterations of our setup}\label{sec:::sim_assum_violations}
In this section, we study the robustness of our algorithm against alteration of the underlying setup. In particular, we investigate three different settings, in which our setup is altered: (i)~the representation between observational and randomized data is \emph{not} shared, (ii)~the observational data is unconfounded, and (iii)~the observational and randomized covariate distributions do not fully overlap.

\textbf{(i)~No shared representation:} Here, we generate a observational and randomized dataset, whose data-generating processes do \emph{not} possess a shared representation.

For this, we generate observational and randomized samples as follows: $X^{\conf}, X^{\unc} \sim \mathcal{N}(\boldsymbol{0}, \mathbf{1})$, $T^{\conf}, T^{\unc}\sim \text{Bernoulli}(\frac{1}{2})$. The outcomes are sampled from two different distributions, with two different representations $\phi^c$ and $\phi^u$:
\begin{align}
    Y^{\conf} \sim \phi^c(X^{\conf})\,\mathbf{w}_{T^{\conf}}^c + \epsilon, \\
    Y^{\unc} \sim \phi^u(X^{\unc})\,\mathbf{w}_{T^{\unc}}^u + \epsilon,
\end{align}
where $\epsilon\sim\mathcal{N}(0, \sigma_\epsilon^2)$. If both observational and randomized data share the representation, then $\phi^\ast = \phi^c = \phi^u$. We break the shared representation assumption by choosing $\phi^c \neq \phi^u$. In particular, we use randomly chosen weights for $\phi^c$. Then, we vary by how much $\phi^u$ differs from $\phi^c$ using $\lVert \mathbf{W}^u_K - \mathbf{W}^c_K\rVert_{1, \infty} = \beta_\phi$, where $\mathbf{W}^u_K$ and $\mathbf{W}^c_K$ are the last weight matrices of the corresponding representations. For $\beta_\phi=0$, the representation is shared, \ie, $\phi^\ast = \phi^c = \phi^u$ and the larger $\beta_\phi$, the larger the difference between $\phi^c$ and $\phi^u$. We sample three different settings for $\beta_\phi$: (1)~$\beta_\phi=0$, \ie, the representation is shared, (2)~$\beta_\phi$ small, \ie, there is a small difference between $\phi^c$ and $\phi^u$, and (3)~$\beta_\phi$ large, \ie, there is a large difference between $\phi^c$ and $\phi^u$.

In \Cref{fig:sim_assum_shared_rep}, we compare the estimation error of $\tau_{\textup{CorNet}}$ against $\tau_{\unc}$ and $\tau_{\conf}$ for all three settings as described above. We make the following three observations. (1)~When there is a shared representation, \ie, $\beta_\phi=0$, our method substantially improves upon $\tau_{\unc}$ and $\tau_{\conf}$. (2)~When we break the assumption of a shared representation (\ie, $\beta_\phi$ small and $\beta_\phi$ large), the error of $\tau_{\conf}$ increases, which is due to a larger bias. At the same time, the error of $\tau_{\unc}$ remains unchanged across the different settings, since $\tau_{\unc}$ only uses randomized (\ie, unbiased) data. (3)~When we break the assumption of a shared representation (\ie, $\beta_\phi$ small and $\beta_\phi$ large), the error of our method, $\tau_{\textup{CorNet}}$, remains consistent across the different settings. That means, our method is robust to violations of the ``shared representation'' assumption. This may be the case, since the bias function is directly learned in the second step of our CorNet. As such, any difference between the representations that has not been correctly learned in the first step, may be learned in the second step as part of the bias function.
\begin{figure}
	\centering 
	\scalebox{0.55}{\includegraphics{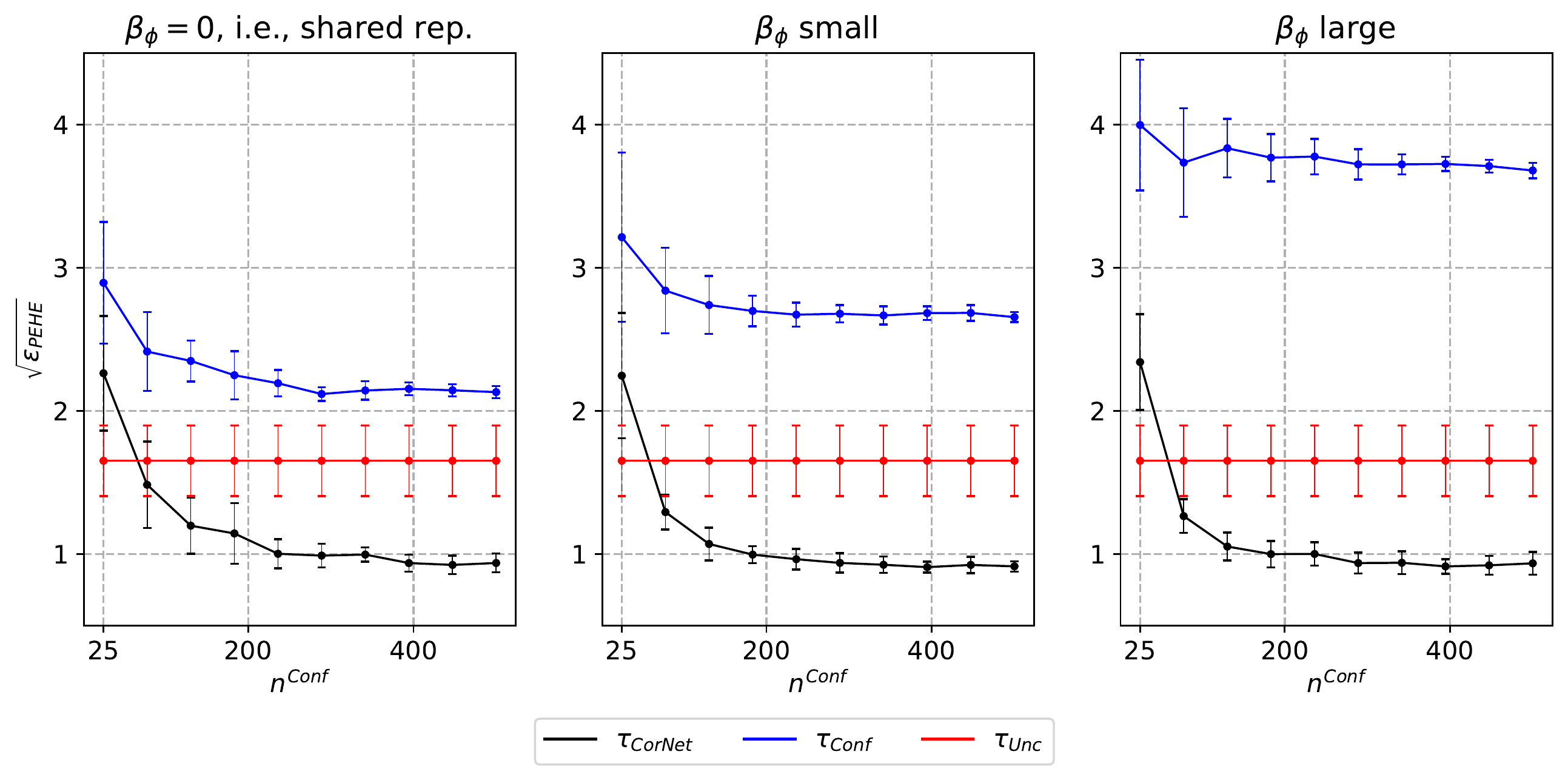}}
	\caption{\footnotesize Simulation study for the robustness of our algorithm, if data does not posses a shared representation. The difference to the true representation is measured by $\lVert \mathbf{W}_K^u - \mathbf{W}_K^c\rVert_{1, \infty} = \beta_\phi$, where $\beta_\phi$ is increased in order to study the impact on the estimation error. We observe that our algorithm, $\tau_{\textup{CorNet}}$, remains robust with respect to violation of the shared representation assumption.}\label{fig:sim_assum_shared_rep}
\end{figure}

\textbf{(ii)~Unconfounded observational data:} In this setting, we generate an observational dataset which is unconfounded, \ie, it does not yield biased estimates. We are interested in whether, in this case, it would be more beneficial to rely only on observational data.

For this, we generate observational and randomized samples as follows: $X^{\conf}, X^{\unc} \sim \mathcal{N}(\mathbf{0}, \mathbf{1})$, $T^{\conf}, T^{\unc}\sim \text{Bernoulli}(\frac{1}{2})$. Then, the outcomes are sampled from two different distributions:
\begin{align}
    Y^{\conf} \sim \phi^\ast(X^{\conf})\,\mathbf{w}_{T^{\conf}}^c + \epsilon, \\
    Y^{\unc} \sim \phi^\ast(X^{\unc})\,\mathbf{w}_{T^{\unc}}^u + \epsilon,
\end{align}
where $\epsilon\sim\mathcal{N}(0, \sigma_\epsilon^2)$. Then, the overall bias due to unobserved confounding can be controlled via $\boldsymbol{\delta}_t = \mathbf{w}_t^u-\mathbf{w}_t^c$, since $\Delta = \mathbb{E}\left[(\phi^\ast(X^{\conf})\,(\boldsymbol{\delta}_1-\boldsymbol{\delta}_0))^2\right]$. Hence, by varying $\boldsymbol{\delta}_t$, we vary the bias due to unobserved confounding. In particular, we use randomly chosen weights for $\mathbf{w}_t^u$ and, then, we vary by how much $\mathbf{w}_t^c$ differs from $\mathbf{w}_t^u$ using $\lVert \boldsymbol{\delta}_t\rVert_1 = \beta$. Hence, for $\beta=0$, the observational data is unconfounded (\ie, $\mathbf{w}_t^c=\mathbf{w}_t^u$ and, therefore, $\Delta=0$). The larger $\beta$, the larger the difference between $\mathbf{w}_t^c$ and $\mathbf{w}_t^u$, and, therefore, the larger $\Delta$. We sample three different settings for $\Delta$: (1)~$\Delta=4$ , \ie, the observational data is heavily confounded, (2)~$\Delta=1$ , \ie, the observational data is moderately confounded, and (3)~$\Delta = 0$, \ie, the observational data is unconfounded.

In \Cref{sim_assum_os_bias}, we compare the estimation error of $\tau_{\textup{CorNet}}$ against $\tau_{\unc}$ and $\tau_{\conf}$ for all three settings as described above. We make the following two observations. First, when there is confounding in the observational data, \ie, $\Delta=4$  or $\Delta=1$, our method consistently improves upon the $\tau_{\unc}$ and, especially, upon $\tau_{\conf}$, where the latter estimator only uses observational data. Second, when the observational data is unconfounded (\ie, $\Delta=0$), the error of $\tau_{\conf}$ decreases below the error of $\tau_{\unc}$ and to a similar level as the error of our method, $\tau_{\textup{CorNet}}$. Notably, the error of $\tau_{\conf}$ does not decrease below the error of our method. This is particularly important, since this implies that there is no drawback from using our method, since it yields superior results even if the observational data is unconfounded. That is, even if the observational data is unconfounded, our method, $\tau_{\textup{CorNet}}$, is on par with methods that only use observational data. The reason for this is that our method uses observational data in the first step and randomized data in the second step. If both types of data are unconfounded, then, our method uses unconfounded data in every step. On the other hand, $\tau_{\conf}$ only uses observational data. Hence, if the observational data is unconfounded, then, $\tau_{\conf}$ uses unconfounded data in every ``step'' as well, which explains the similar, but not superior, error of $\tau_{\conf}$.
\begin{figure}
	\centering 
	\scalebox{0.55}{\includegraphics{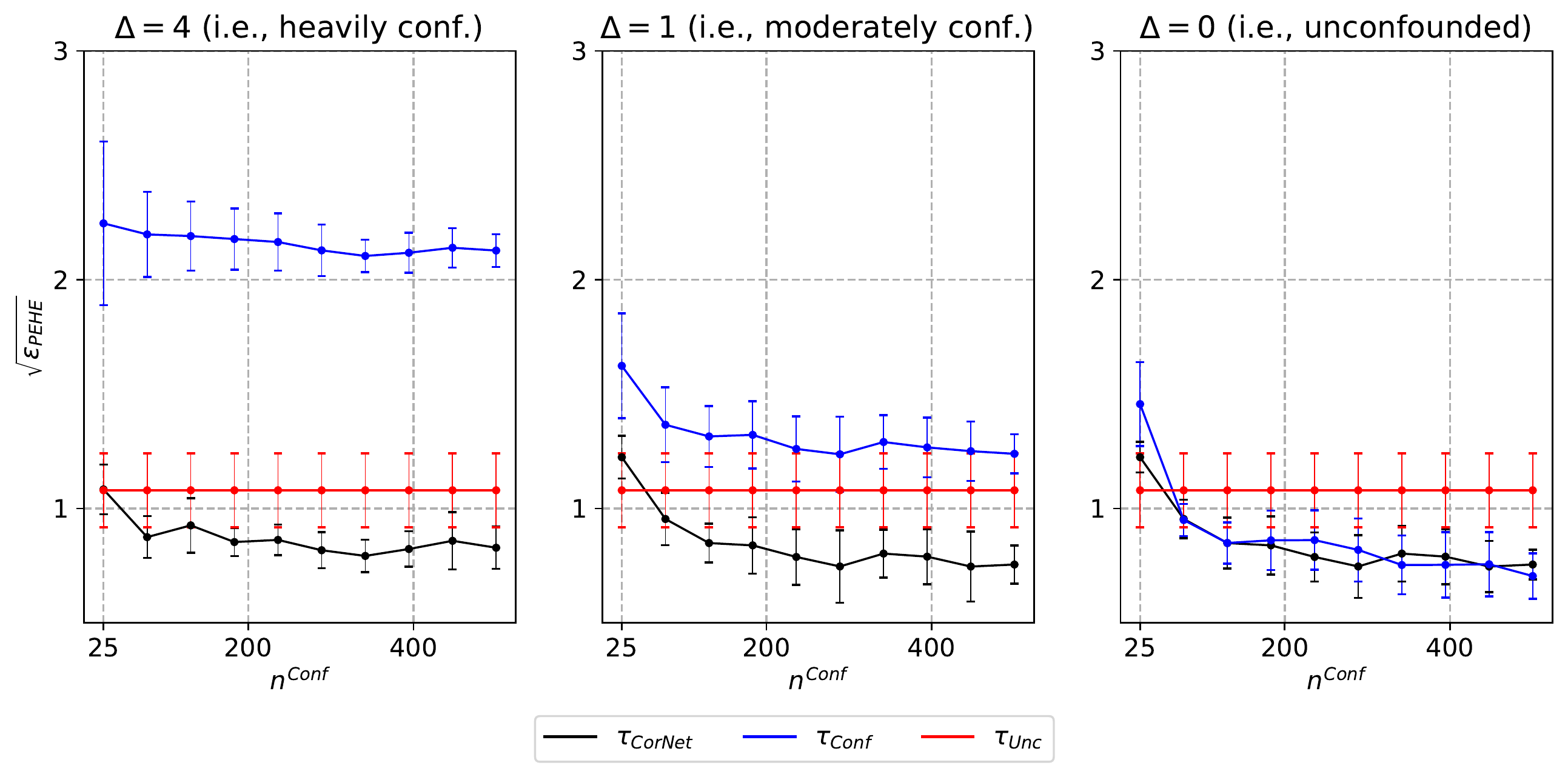}}
	\caption{\footnotesize Simulation study for robustness of our algorithm if the observational data is unconfounded. The amount of bias due to unobserved confounding is measured by $\Delta$, which is decreased in order to study the impact of no unobserved confounding on the estimation error. We observe that our algorithm, $\tau_{\textup{CorNet}}$, remains superior even if the observational data is confounded.}\label{sim_assum_os_bias}
\end{figure}

\textbf{(iii)~Limited overlap between observational and randomized data:} Here, we investigate whether our method yields robust results if the population overlap assumption (see (iii)~in Assumption~\labelcref{assum:standard}) is violated. In particular, this assumption implies that $p_x^{\unc}$ and $p_x^{\conf}$ fully overlap, which ensures that the distributional discrepancy, $\mathit{d}_\infty(p^{\conf}_\phi\mid p^{\unc}_\phi)$, in \Cref{thm_bound} is well-defined. Here, we study the impact on our method if this is not satisfied.


For this, we generate observational and randomized samples as follows: $X^{\conf} \sim \mathcal{N}(\mathbf{0}, \mathbf{1})$ and $T^{\conf}, T^{\unc}\sim \text{Bernoulli}(\frac{1}{2})$. The outcomes are sampled from two different distributions:
\begin{align}
    Y^{\conf} \sim \phi^\ast(X^{\conf})\mathbf{w}_{T^{\conf}}^c + \epsilon, \\
    Y^{\unc} \sim \phi^\ast(X^{\unc})\mathbf{w}_{T^{\unc}}^u + \epsilon,
\end{align}
where $\epsilon\sim\mathcal{N}(0, \sigma_\epsilon^2)$. In order to study what happens if $p_x^{\unc}$ and $p_x^{\conf}$ do not fully overlap, we sample the randomized covariates as follows: $X^{\unc} \sim \mathcal{U}([-\textup{a}, \textup{a}]^d)$. As such, $p^{\unc}$ is zero outside of the hypercube $[-\textup{a}, \textup{a}]^d$ and, therefore, violates the overlap assumption. By varying `$\textup{a}$', we control by how much the population overlap assumption is violated. In particular, we study three different settings for $X^{\unc} \sim \mathcal{U}([-\textup{a}, \textup{a}]^d)$: $\textup{a}=3$, $\textup{a}=1$, and $\textup{a}=\frac{1}{2}$.

In \Cref{sim_assum_overlap} (top), we compare the estimation error of $\tau_{\textup{CorNet}}$ against $\tau_{\unc}$ and $\tau_{\conf}$ for all three settings for $X^{\unc} \sim \mathcal{U}([-\textup{a}, \textup{a}]^d)$ as described above. We also display the covariate distribution of the observational (in blue) and randomized (in red) data for one dimension, which illustrates the violation of the overlap. 

We make the following three observations. (1)~We observe that the error of $\tau_{\conf}$ remains constant across the settings, since it uses only observational data. (2)~The error of $\tau_{\unc}$ increases the less the two distributions overlap. (3)~The error of $\tau_{\textup{CorNet}}$ also increases when the overlap becomes less, but, it increases substantially less than the error of $\tau_{\unc}$. Moreover, the error of $\tau_{\textup{CorNet}}$ remains substantially below $\tau_{\unc}$ and $\tau_{\conf}$ across all settings. However, we cannot conclude from this that the violation of the population overlap assumption harms the performance of our method, since, by varying $\textup{a}$ in $\mathcal{U}([-\textup{a}, \textup{a}]^d)$, we also vary the distributional discrepancy. And, as seen in \Cref{thm_bound} and \Cref{sim_res_dist}, varying the distributional discrepancy also impacts the error.
\begin{figure}[ht!]
	\centering 
	\scalebox{0.75}{\includegraphics{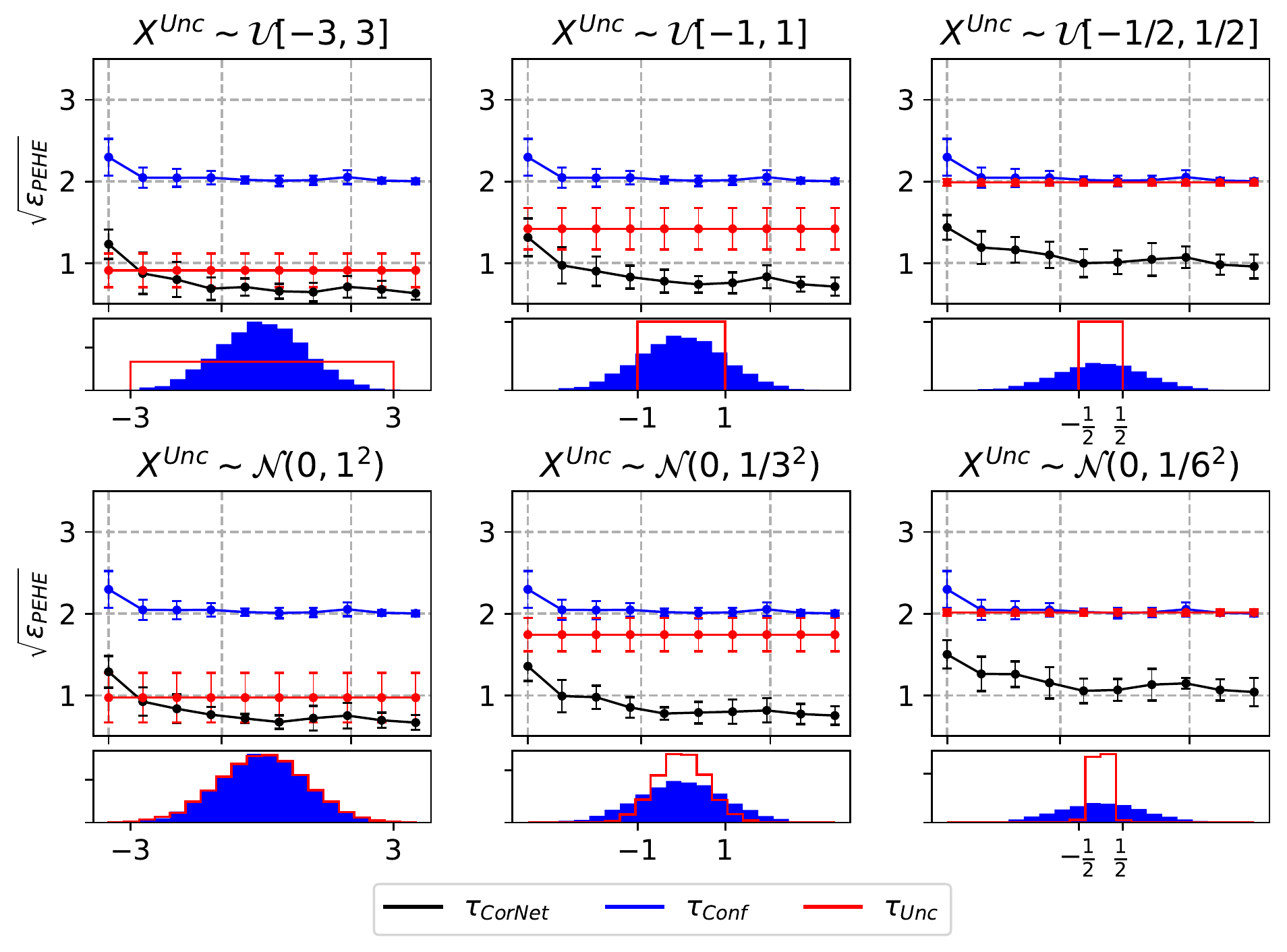}}
	\caption{\footnotesize Simulation study for the violation of the population overlap assumption. The degree of overlap is measured by the hypercube $[-\textup{a}, \textup{a}]^d$, which is decreased in order to study the impact of assumption violation on the estimation error. Top:~the error for $\tau_{\textup{CorNet}}$, $\tau_{\conf}$, and $\tau_{\unc}$ for different hypercubes and for different values of $n^{\conf}$ ($x$-axis). The distribution of the covariates, $X^{\conf}$ and $X^{\unc}$, and their limited overlap is illustrated below the error plot. Bottom:~the errors for the same methods, but without violating the overlap assumption in order to distinguish the contribution of the overlap violations and the distributional discrepancy to the error. The distribution of the covariates, $X^{\conf}$ and $X^{\unc}$, are presented below the error plot. Since the errors increase across the settings similarly in both rows, we can conclude that our algorithm, $\tau_{\textup{CorNet}}$, remains robust with respect to violation of the population overlap assumption.}\label{sim_assum_overlap}
\end{figure}

In order to isolate the impact of the violation of the population overlap assumption, we sample randomized data, which fulfills the overlap assumption, but has a similar distribution as $X^{\unc} \sim \mathcal{U}([-\textup{a}, \textup{a}]^d)$ (\ie, a sample is with high probability in $[-\textup{a}, \textup{a}]^d$). This means, we sample randomized data with similar distributional discrepancy, but without violating the population overlap assumption. As such, we can distinguish the contribution of the overlap violations and the distributional discrepancy to the error.

For this, for each of the above settings, we sample $X^{\unc}\sim\mathcal{N}(\mathbf{0}, \sigma_u^2\mathbf{1})$, which satisfies the overlap assumption. Moreover, we choose $\sigma_u$ for each setting such that $X^{\unc}$ is, with high probability,\footnote{To be precise, with high probability, we mean that we choose $\sigma_u$ such that $\mathbb{P}(X^{\unc}\in [-\textup{a}, \textup{a}]^d)=0.9973$ in \emph{all} settings.} in $[-\textup{a}, \textup{a}]^d$. This yields, for each of the above settings, the following choices: for $\textup{a}=3$, we choose $\sigma_u=1$; for $\textup{a}=1$, we choose, $\sigma_u=\frac{1}{3}$; and, for $\textup{a}=\frac{1}{2}$, we choose $\sigma_u=\frac{1}{6}$.


In \Cref{sim_assum_overlap} (bottom), we present the same results as in the first row, but for $X^{\unc} \sim \mathcal{N}(\mathbf{0}, \sigma_u^2\mathbf{1})$ with $\sigma_u\in\{1, \frac{1}{3}, \frac{1}{6}\}$. 

By comparing the results at the top against the results at the bottom in \Cref{sim_assum_overlap}, we see how much of the changes in error are contributed to the violation of the population overlap assumption and how much to the distributional discrepancy. We find that the errors change similarly across the settings for our method, $\tau_{\textup{CorNet}}$, at the top \emph{and} at the bottom in \Cref{sim_assum_overlap}. As such, the errors at the bottom change similarly (from left to right) as at the top. However, at the bottom, the population overlap assumption is \emph{not} violated. Hence, we can conclude that, also at the top, the change in the errors is coming from the distributional discrepancy and \emph{not} from the violation of the population overlap assumption. As a consequence, our method remains robust against violations of the population overlap assumptions.


\subsection{Real-World Data}\label{sec:real_world_data}
In this section, we validate our method against an extensive set of baselines on real-world data. Validating causal inference method on real-world data is challenging, since we do not have access to the true potential outcomes. We circumvent this challenge by using real-world data from three large experiments: the Tennessee Student/Teacher Achievement Ratio (STAR) experiment \citep{word1990state, krueger1999experimental}, the AIDS Clinical Trial Group (ACTG) study 175 \citep{hammer1996trial}, and the National Supported Work Demonstration \citep{LaLonde1986, Smith2005}. We discuss in \Cref{sec:::exp_ground_truth} how a ground truth can be constructed from this real-world data. 

\subsubsection{Summary of the real-world data}\label{sec:::exp_data}
We briefly describe all three real-world datasets that we used for our experiments in the following. More details on each of the datasets can be found in \Cref{apx:data_sets_details}.

\textbf{Tennessee Student/Teacher Achievement Ratio (STAR) experiment:} The Tennessee Student/Teacher Achievement Ratio (STAR) experiment is a randomized randomized experiment, which started in 1985 with the objective to study the effect of class size (\ie, treatment) on students' standardized test scores (\ie, outcome). This dataset was previously also used in \citet{kallus2018removing} for removing bias due to unmeasured confounding in observational data. Following their setup yields a randomized samples of 8 covariates and  4,139 students: 1,774 students assigned to treatment (\ie, small class, $T=1$) and 2,365 students assigned to control (\ie, regular class, $T=0$).

\textbf{AIDS Clinical Trial Group (ACTG) study 175:} The AIDS Clinical Trial Group (ACTG) study 175 is a clinical trial with the goal of studying the effect of different treatments on subjects with human immunodeficiency virus type 1 (HIV-1), whose CD4 counts were 200–500 cells/mm$^3$ \citep{hammer1996trial}. The ACTG study was used in \citet{hatt2021generalizing} for learning policies that generalize to the target population. Even though this is a different aim, the study is particularly suited for evaluating our method since HIV-positive females tend to be underrepresented in clinical trials, which makes these studies not representative of the target population (\ie, the HIV-positive population) \citep{gandhi2005eligibility, greenblatt2011priority}. The outcome $Y$ is defined as the difference between the cluster of differentiation 4 (CD4) cell counts at the beginning of the study and the CD4 counts after $20\pm5$ weeks. The average treatment effects on the male and female subgroups are $-8.97$ and $-1.39$, respectively \citep{hatt2021generalizing}, which suggests a large discrepancy in the treatment effects between both subgroups. Following the setup in \citet{hatt2021generalizing} yields a total of 1,056 patients and 12 covariates, which are detailed in \Cref{apx:data_sets_details}.

\textbf{National Supported Work~(NSW) Demonstration:}
The National Supported Work~(NSW) Demonstration was a randomized experiment investigating the effect of job training on income \citep{LaLonde1986}.\footnote{The study by \citet{LaLonde1986} is a widely used dataset in the causal inference literature and is also known as the ``Jobs'' dataset \citep[\eg,][]{Shalit2017a, hatt2021estimating}.} Following \citet{Smith2005}, we combine randomized samples of 465 subjects (297 treated, 425 control) with the 2,490 PSID observational controls to create an observational dataset and include 8 covariates, which are detailed in \Cref{apx:data_sets_details}.

\subsubsection{Estimating ground truth treatment effects}\label{sec:::exp_ground_truth}
Evaluating causal inference methods is generally challenging, since observational data is often confounded and the size of randomized data is usually too small. Fortunately, our datasets, STAR, ACTG, and NSW, originated from large-scale experiment. As such, these datasets are unconfounded, since the treatment assignment was controlled by the investigator of the experiment. Moreover, these experiments have the appropriate size to estimate treatment effect heterogeneity. As a result, the CATE is identifiable in all three datasets via standard causal inference methods. 

Therefore, we can estimate the potential outcomes $\mathbb{E}[Y(t)\mid X=x]$ for each $t\in\{0, 1\}$. Note, that is possible, since the data is unconfounded and, therefore, $\mathbb{E}[Y(t)\mid X=x] = \mathbb{E}[Y\mid X=x, T=t]$. In addition, since we have access to large-scale experiments, we have enough data to estimate treatment effect heterogeneity. For this, we use neural networks; in particular, we use an architecture as proposed for TarNet \citep{Shalit2017a}, which uses a shared representation across potential outcomes and two independent hypotheses. We use neural networks is two-fold. First, we implemented all methods\footnote{Strictly speaking, there are two baselines which are not based on neural networks. However, for these two baselines, we also implement a neural network-based version to guarantee fair comparison. See \Cref{sec:::exp_baselines} for details.} using neural networks and, therefore, a consistent use of the same underlying method enables fair comparisons across all baselines. Second, we require a function class with sufficiently capacity such that the true functions are contained in this function class with high confidence. Due to their large capacity, neural networks are a favorable choice for this.\footnote{As a sensitivity check, we also ran all experiments using random forest, tree-based methods, and kernel methods as estimators for the ground truth treatment effect, which did not alter the final outcome of our experiments.} Based on this, we estimate the ground truth CATE by $\hat{\tau}_{\text{GT}}(x) = \hat{\mathbb{E}}[Y(1)\mid X=x] - \hat{\mathbb{E}}[Y(0)\mid X=x]$. Note that this is only feasible, since we have access to the data of uncommonly large experiments, which allows to estimate complex functions directly on the randomized data. This, however, is rarely the case in practice, but enables us to evaluate our proposed method. 

\subsubsection{Generating small randomized and large observational data}\label{sec:::exp_os_rct}
Following the same procedure as in \citet{kallus2018removing}, we artificially generate a large, but confounded observational dataset and an unconfounded, but small randomized datasets using the real-world data described in \Cref{sec:::exp_data}. In particular, the randomized data is generated over a different population than the observational data, which represents the limited capabilities of RCTs to generalize to the population of interest. 

For this, we follow the same protocol for STAR, ACTG, and NSW: First, from the original dataset, we sample a unconfounded, but small randomized dataset. To do so, we sample $n^{\unc}=2d$ samples. We introduce distributional discrepancy between the randomized and observational data by selecting subjects into the randomized dataset based on one of the covariates (``birthday'' for STAR, ``gender'' for ACTG, and ``age'' for NSW). Second, we generate the observational dataset by introducing unobserved confounding, \ie, the treatment and control group must be systematically different in their potential outcomes. Similar to \citet{kallus2018removing}, we sample the following subjects from whose who were not selected into the randomized dataset: we take the controls ($T=0$) whose outcomes were particularly low (\ie, $y_i<\E[Y\mid T=0]-c\cdot \sigma_{Y\mid T=0}$, where $\sigma_{Y\mid T=0}$ is the standard deviation of the outcomes in the control). Then, we take the treated (\ie, $T=1$) whose outcomes were particularly high (\ie, $y_i>\E[Y\mid T=1]+c\cdot \sigma_{Y\mid T=1}$, where $\sigma_{Y\mid T=1}$ is the standard deviation of the outcomes among treated). The constant $c$ is used to control how many subjects are selected into the observational data (to ensure it remains large) and depends on the size of the original dataset (we choose $c=1$, $c=0$, $c=0.25$ for STAR, ACTG, and NSW, respectively). This procedure yields confounding by including control subjects with lower outcomes and treated subjects with higher outcomes selectively into control and treatment group of the observational data. Hence, a na{\"i}ve estimate which only uses observational data will be biased. Moreover, since this selection is based on the outcome variable, it makes it impossible to control for this confounding.

\subsubsection{Baselines}\label{sec:::exp_baselines}
We compare our algorithm against seven baselines: (i)~The estimator on randomized data, $\hat{\tau}_{\unc}$, as described in \Cref{sec::tau_unc}; (ii)~the estimator on observational data, $\hat{\tau}_{\conf}$, as described in \Cref{sec::tau_conf}, which we have seen is biased; (iii)~the averaging estimator, $\hat{\tau}_{\avg}$, as described in \Cref{sec::tau_avg}, which is a convex combination of $\hat{\tau}_{\unc}$ and $\hat{\tau}_{\conf}$. This is also the estimator that was proposed recently in \citet{cheng2021adaptive}; (iv)~the weighting estimator, $\hat{\tau}_{\weight}$, as described in \Cref{sec::tau_weight}, which joins both data sets, but assigns a larger weight to randomized samples. For all the above baselines, the chosen functions classes are neural networks, which allows for comparison between them. Moreover, we compare against the two-step approach proposed in \citet{kallus2018removing}, which first estimates a biased model and then tries to remove the bias using the randomized data (see \Cref{apx:baseline_improvement_discussion} for a detailed description of the method). The authors use ridge regression and random forest (RF) for estimating the biased model. In addition, for fair comparison, we also use neural networks (NN) for estimating this biased model. Hence, this yields three further baselines from \citet{kallus2018removing}: (v)~{2-step-ridge}, (vi)~{2-step-RF}, (vii)~{2-step-NN}.

\subsubsection{Evaluation Metric}\label{sec:::exp_metrics}
Our method is evaluated against various baselines for predicting the CATE. We evaluate how well the CATE estimate matches the ground truth estimate ($\tau_{\text{GT}}$ from \Cref{sec:::exp_ground_truth}) in terms of the empirical version of the $\epsilon_{\textup{PEHE}}$. For this, an unconfounded version of the observational dataset, which is exactly the original dataset (as described in \Cref{sec:::exp_data}) minus the randomized dataset. In sum, the metric for evaluation is: $\hat{\epsilon}_{\textup{PEHE}}(\hat{\tau}) = \frac1{n^{\text{test}}} \sm i {n^{\text{test}}} (\hat{\tau}(x^{\text{test}}_i) - \tau_{\text{GT}}(x^{\text{test}}_i))^2$.

\subsubsection{Results}
We run extensive experiments across all three real-world datasets. We find that, across all datasets and compared to all baselines, our algorithms, \ie, $\tau_{\textup{CorNet}}$ and its regularized variant, $\tau_{\textup{CorNet}^+}$, perform superior. 

The results are presented in \Cref{tbl:exp_results}. We make the following five observations about the results presented in \Cref{tbl:exp_results}. (1)~Our methods (\ie, $\tau_{\textup{CorNet}}$, $\tau_{\textup{CorNet}^+}$) outperform all other methods substantially. The regularized two-step procedure, $\tau_{\textup{CorNet}^+}$, improves upon the unregularized two-step procedure, $\tau_{\textup{CorNet}}$, on STAR and ACTG, but not on NSW. The reason for this may be that the bias function is not sparse for this dataset. We investigate this in more detail in the ablation study in \Cref{sec:::exp_ablation_study}. (2)~The baselines ``2-step-ridge'', ``2-step-RF'', and ``2-step-NN'' from \citet{kallus2018removing} perform comparatively poorly due to the re-weighting of the outcomes. We discuss the reasons for this at length in Remark~\labelcref{rmk:poor_baseline_performance}. (3)~The error of $\tau_{\conf}$, which only uses observational data, depends on the magnitude of the bias due to unobserved confounding and displays consistently high errors across all datasets. (4)~The method $\tau_{\unc}$, which only uses randomized data, displays also high errors across all datasets. Moreover, $\tau_{\unc}$ displays substantially higher variance than the other methods across all datasets. Both observations, the high error and high variance, are due to the small sample size of randomized data. (5)~The averaging and weighted risk estimators, $\tau_{\avg}$ and $\tau_{\weight}$, yield errors of a similar magnitude than $\tau_{\unc}$ and $\tau_{\conf}$. This is expected from the theoretical comparison in \Cref{sec:::comparison_baselines}, since these estimators are combinations of $\tau_{\unc}$ and $\tau_{\conf}$. Between $\tau_{\avg}$ and $\tau_{\weight}$, there is no clearly preference: On STAR and NSW, $\tau_{\avg}$ achieves lower error, whereas, on ACTG, $\tau_{\weight}$ achieves lower errors.
\begin{table*}[ht!]
	\caption{\footnotesize Results for the experiments on the real-world datasets STAR, ACTG, and NSW. Results are obtained via 10 runs. Lower is better.}\label{tbl:exp_results}
	\begin{center}
				\begin{tabular}{lcccc}
					\multicolumn{4}{l}{\bf{Results: Experiments on three real-world datasets}}\\
					\toprule\addlinespace[0.75ex] &\multicolumn{3}{c}{\bf{$\sqrt{\hat{\epsilon}_{\textup{PEHE}}}$ (Mean $\pm$ Std)}}\\
					\cmidrule{2-4}\addlinespace[0.75ex]
					\tikz{\node[below left, inner sep=1pt] (est) {Estimator};%
      \node[above right,inner sep=1pt] (set) {Dataset};%
      \draw (est.north west|-set.north west) -- (est.south east-|set.south east);}
                    & \bf{STAR}
					&\bf{ACTG}
					&\bf{NSW}\\
					\hline
					\addlinespace[0.75ex]
					\makecell[l]{2-step ridge} 
					&$\text{3.01} \pm \text{0.01}$ 
					& $\text{1.51} \pm \text{0.01}$
					& $\text{2.82} \pm \text{0.02}$\\
					\addlinespace[0.75ex]
					\makecell[l]{2-step RF} 
					&$\text{3.14} \pm \text{0.03}$ 
					& $\text{1.58} \pm \text{0.07}$
					& $\text{3.10} \pm \text{0.12}$\\
					\addlinespace[0.75ex]
					\makecell[l]{2-step NN} 
					&$\text{3.03} \pm \text{0.02}$ 
					& $\text{1.60} \pm \text{0.02}$
					& $\text{2.82} \pm \text{0.02}$\\
                    \addlinespace[0.75ex]
					\makecell[l]{$\tau_{\conf}$} &$\text{2.66} \pm \text{0.01}$ 
					& $\text{1.08} \pm \text{0.04}$
					& $\text{0.85} \pm \text{0.04}$\\
					\addlinespace[0.75ex]
					\makecell[l]{$\tau_{\unc}$} &$\text{1.33} \pm \text{0.22}$ 
					& $\text{1.10} \pm \text{0.11}$
					& $\text{0.52} \pm \text{0.15}$\\
					\addlinespace[0.75ex]
				\makecell[l]{$\tau_{\avg}$} 
					&$\text{1.68} \pm \text{0.35}$ 
					& $\text{0.91} \pm \text{0.29}$
					& $\text{0.51} \pm \text{0.20}$\\
					\addlinespace[0.75ex]
			   	\makecell[l]{$\tau_{\weight}$} 
						&$\text{2.26} \pm \text{0.09}$ 
					& $\text{0.80} \pm \text{0.10}$
					& $\text{0.76} \pm \text{0.07}$\\\hline
					\addlinespace[0.75ex]
					\makecell[l]{$\tau_{\textup{CorNet}}$ (ours)} &$\text{0.59} \pm \text{0.01}$ 
					& $\text{0.42} \pm \text{0.06}$
					& $\textbf{0.14} \pm \textbf{0.07}$\\
					\addlinespace[0.75ex]
		       	\makecell[l]{$\tau_{\textup{CorNet}^+}$ (ours)} &$\textbf{0.38} \pm \textbf{0.07}$ 
					& $\textbf{0.27} \pm \textbf{0.03}$
					& $\text{0.21} \pm \text{0.08}$\\
					\bottomrule	
			\end{tabular}
	\end{center}
\end{table*}

\begin{remark}\label{rmk:poor_baseline_performance}
    In \Cref{tbl:exp_results}, we observe that the baselines ``2-step-ridge'', ``2-step-RF'', and ``2-step-NN'' from \citet{kallus2018removing} perform comparatively poorly. This is due to the re-weighting of the outcomes using the inverse propensity score.\footnote{See \Cref{apx:baseline_improvement_discussion} or Eq.~(2) in \citet{kallus2018removing} for more details.} It is well-known that this yields high-variance estimates, particularly in higher dimensions. Hence, these baselines improve when $n^{\unc}$ is increased, which can be observed in \Cref{sensitivity_results_n}. However, their performance remains inferior to na{\"i}ve approaches such as $\tau_{\unc}$ and, in particular, inferior to our methods. In order to show that the poor performance originates from re-weighting, we modify these baselines to avoid re-weighting, which improves their performance by a substantial margin. For this, we apply their approach to both of the outcomes rather than the CATE directly. This circumvents the need of re-weighting with the inverse propensity score. Although these modifications improve the baselines, they remain inferior to our method. A detailed discussion on the modification and the numerical results for the modified baselines can be found in \Cref{apx:baseline_improvement_discussion}.
\end{remark}

In sum, we conclude that our methods, $\tau_{\textup{CorNet}}$ and $\tau_{\textup{CorNet}^+}$, outperform a variety of baselines across different real-world datasets. 

\subsubsection{Sensitivity with respect to \texorpdfstring{$n^{\conf}$}{} and \texorpdfstring{$n^{\unc}$}{}}\label{sec:::exp_sensitivity_n}
In this section, we study the sensitivity of the previous results for different numbers of $n^{\conf}$ and $n^{\unc}$. In the previous experiments, we chose an explicit size of the observational and randomized dataset. Specifically, we chose $n^{\conf}$ to be the maximum available samples for the observational dataset after the processing described in \Cref{sec:::exp_os_rct}. In particular, this yields $n^{\conf}\in\{643, 552, 482\}$ for the datasets STAR, ACTG, and NSW.\footnote{The sizes of the observational datasets, $n^{\conf}$, are smaller than the sizes of the original datasets used to construct the observational datasets. That is, since we introduce confounding by only including treated subjects with high outcomes and control subjects with low outcomes. As such, we exclude many of the subjects in the original dataset.} Moreover, we chose $n^{\unc}=2\,d$, where $d$ is the number of covariates. We did so, since the randomized datasets are usually data-scarce. Here, we study the robustness of our results from the previous section when $n^{\conf}$ and $n^{\unc}$ are varied.

\textbf{Sensitivity to $n^{\conf}$:} We compare our method against the baselines across different numbers of observational samples used, \ie, we vary $n^{\conf}$. In particular, we choose $n^{\conf} \in \{300, 400, 500, 600, 643\}$ for STAR, $n^{\conf} \in \{200, 300, 400, 500, 552\}$ for ACTG, and $n^{\conf} \in \{100, 200, 300, 400, 482\}$ for NSW. The results are presented in \Cref{sensitivity_results_n} (top). We observe that all results remain consistent across different numbers of observational samples. In particular, our methods, $\tau_{\textup{CorNet}}$ and $\tau_{\textup{CorNet}^+}$, achieve consistently superior results.

\textbf{Sensitivity to $n^{\unc}$:} We compare our method against the baselines across different numbers of randomized samples used, \ie, we vary $n^{\unc}$. In particular, for $X^{\unc}\in\Rl^d$, we choose $n^{\unc} \in \{1\cdot d, 2\cdot d, 3\cdot d, 4\cdot d, 6\cdot d, 8\cdot d\}$ for all datasets. The results are presented in \Cref{sensitivity_results_n} (bottom). We observe that the results remain consistent across different numbers of randomized samples. In particular, our methods, $\tau_{\textup{CorNet}}$ and $\tau_{\textup{CorNet}^+}$, achieve the lowest error for any number of randomized samples. We make two further observations. First, the larger the number of randomized samples, the closer the error of $\tau_{\unc}$ becomes to the error of our methods. This is because, for a large randomized dataset, which is not data-scarce anymore, only using randomized data becomes a competitive alternative. However, notably, our method remains superior for any number of randomized samples across all real-world datasets. Second, turning to the baselines ``2-step ridge'', ``2-step RF'', and ``2-step NN'', we observe that their error substantially decreases for increasing number of randomized samples. However, for any number of randomized samples, they remain inferior to the na{\"i}ve approach $\tau_{\unc}$ and, in particular, to our methods.
\begin{figure}[ht]
	\centering 
	\scalebox{0.425}{\includegraphics{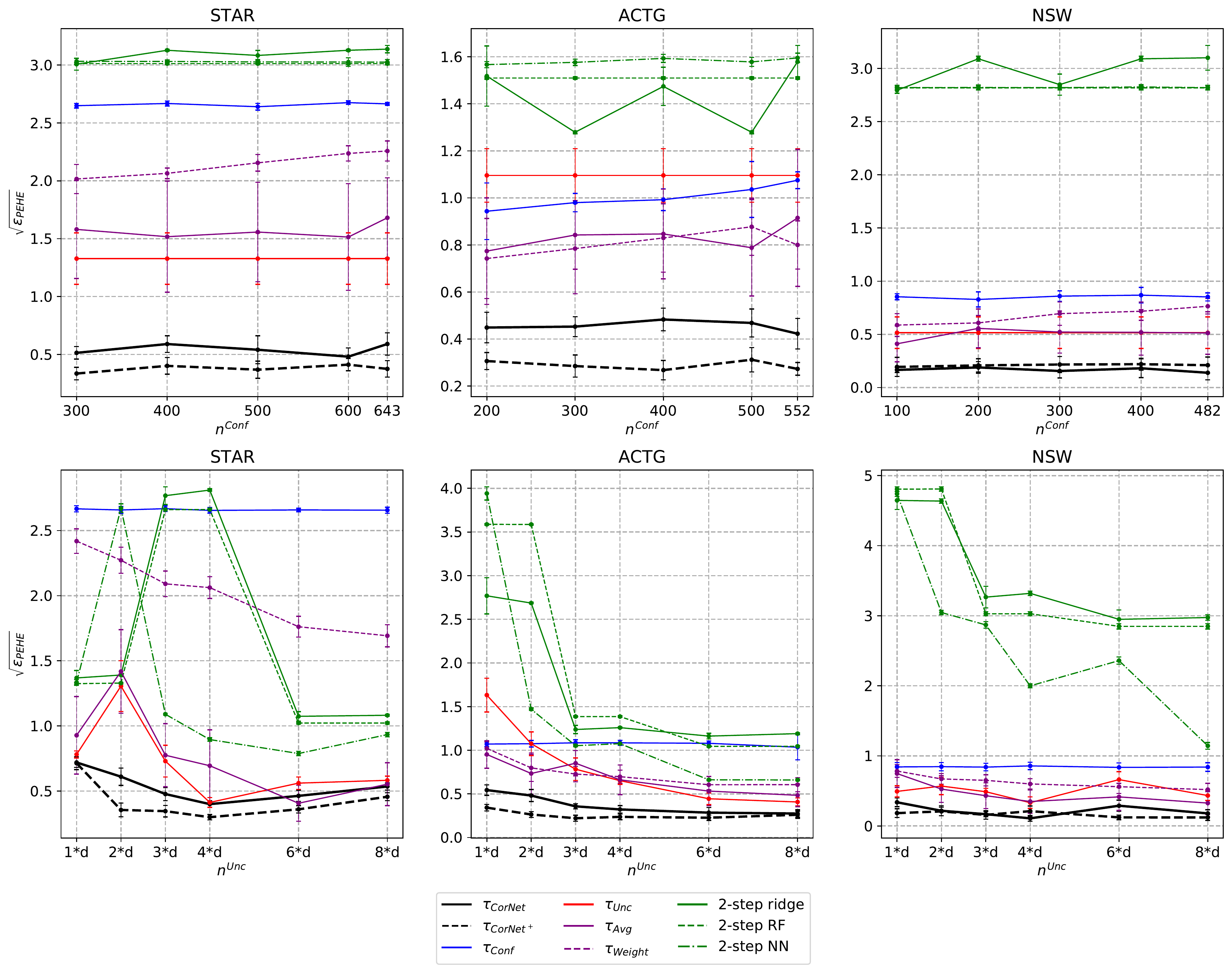}}
	\caption{\footnotesize Sensitivity study of the results for different numbers of $n^{\conf}$ (top row) and $n^{\unc}$ (bottom row) across all real-world datasets STAR, ACTG, and NSW. The results are consistent across all numbers of observational (\ie, $n^{\conf}$) and randomized samples (\ie, $n^{\unc}$). For large number of randomized samples, the error of $\tau_{\unc}$ approaches the error of our methods, $\tau_{\textup{CorNet}}$ and $\tau_{\textup{CorNet}^+}$. This is expected, since once the randomized dataset is not small anymore, only using randomized data may be a reasonable alternative. However, our methods remain superior across all numbers of observational and randomized samples.}\label{sensitivity_results_n}
\end{figure}

\subsubsection{Ablation Study}\label{sec:::exp_ablation_study}
In this section, we conduct an ablation study to investigate the contribution of the different components of our method, $\tau_{\textup{CorNet}}$, to its superior performance. We also investigate different extensions and variants of our method based on multi-task learning, which we discuss in depth in \Cref{apx:variant_study}.

For the ablation study, we take our $\tau_{\textup{CorNet}}$ as a starting point. Then, we compare the performance of variants, for which we add different components as follows: $\tau_\textup{CorNet}$ with (i)~no regularizations, (ii)~$L_1$-regularization on the bias function $\boldsymbol{\delta}$, (iii)~$L_2$-regularization on $\boldsymbol{\delta}$, (iv)~balancing the covariate distributions, \ie, $\lambda_d>0$, and (v)~$L_1$-regularization on $\boldsymbol{\delta}$ and balancing (which is $\tau_{\textup{CorNet}^+}$). In addition, we also compare to variants, which estimate the unconfounded hypotheses, $\mathbf{w}^u$, in the second step instead of the bias function. We denote this as $\tau_{\textup{CorNet}}$ + $\mathbf{w}^u$ and, then, add different components as follows: (vi)~no regularization on $\mathbf{w}^u$, (vii)~$L_1$-regularization on $\mathbf{w}^u$, (viii)~$L_2$-regularization on $\mathbf{w}^u$, (ix)~$L_1$-regularization on $\mathbf{w}^u$ and balancing, (x)~$L_2$-regularization on $\mathbf{w}^u$ and balancing. 

The results are presented in \Cref{tbl:ablation_study_results}. We make the following observations. The $L_1$-regularization on the bias improves the performance for STAR and ACTG, but not for NSW. This indicates that the bias function may not be sparse in the case of NSW. $L_2$-regularization performs worse than $L_1$-regularization for STAR and ACTG, but better for NSW. Again, this indicates that the bias functions may not be sparse on NSW. Learning the hypotheses, $\mathbf{w}^u_t$, (rather than the bias function) yields worse performance across all datasets. On top of this, regularizing the hypotheses (either $L_1$ or $L_2$) can yield improvement. Balancing the covariates can, but does not always lead to improvement compared to the unbalanced variants. Finally, $L_1$-regularization on the bias function and balancing of the covariates (\ie, $\tau_{\textup{CorNet}^+}$) improves upon only regularizing the bias function for STAR and NSW, but not for ACTG. This may be due to extreme underrepresentation of women in these trials. Hence, counteracting this large discrepancy may cause the predictive accuracy to decrease by a substantial margin. Similarly to the sensitivity study for $n^{\conf}$ and $n^{\unc}$ in \Cref{sec:::exp_sensitivity_n}, we run the identical sensitivity study for the ablation study. The findings are consistent across different numbers of observational and randomized samples and can be found in \Cref{apx:sensitivity_ablation}.
\begin{table*}[ht]
	\caption{\footnotesize Results for the ablation study on the real-world datasets STAR, ACTG, and NSW. $L_1$, $L_2$, and $\lambda_d$ indicate whether $L_1$-regularization, $L_2$-regularization, or balancing was used. Results obtained via 10 runs. Lower is better.}\label{tbl:ablation_study_results}
	\begin{center}
				\begin{tabular}{llcccccc}
					\multicolumn{7}{l}{\bf{Results: Ablation study on three real-world datasets}}\\
					\toprule\addlinespace[0.75ex] 
					&&&&&\multicolumn{3}{c}{\bf{$\sqrt{\hat{\epsilon}_{\textup{PEHE}}}$ (Mean $\pm$ Std)}}\\
					\cmidrule{6-8}\addlinespace[0.75ex]
					\multicolumn{2}{l}{Estimator}
                    & $L_1$ & $L_2$ & $\lambda_d$ & \bf{STAR} &\bf{ACTG} &\bf{NSW}\\

                    \midrule
					\addlinespace[0.75ex]

					\multirow{4}{*}{
					    \rotatebox[origin=c]{90}{$\tau_{\textup{CorNet}}$}
					}&
					\makecell[l]{(i)} 
					& \xmark
					& \xmark
					& \xmark
					& $\text{0.59} \pm \text{0.10}$
					& $\text{0.42} \pm \text{0.06}$
					& $\text{0.14} \pm \text{0.07}$\\
                    \addlinespace[0.75ex]
					&\makecell[l]{(ii)} 
					& \cmark
					& \xmark
					& \xmark
					& $\text{0.41} \pm \text{0.07}$
					& $\text{0.25} \pm \text{0.04}$
					& $\text{0.23} \pm \text{0.04}$\\
					\addlinespace[0.75ex]
					&\makecell[l]{(iii)} 
					& \xmark
					& \cmark
					& \xmark
					& $\text{0.57} \pm \text{0.09}$
					& $\text{0.40} \pm \text{0.06}$
					& $\text{0.10} \pm \text{0.05}$\\
					\addlinespace[0.75ex]
					&\makecell[l]{(iv)} 
					& \xmark
					& \xmark
					& \cmark
					&$\text{0.54} \pm \text{0.09}$
					& $\text{0.47} \pm \text{0.04}$
					& $\text{0.19} \pm \text{0.05}$\\
										\addlinespace[0.75ex]					
					&\makecell[l]{(v) $\tau_{\textup{CorNet}^+}$} 
					& \cmark
					& \xmark
					& \cmark
					& $\text{0.38} \pm \text{0.07}$
					& $\text{0.27} \pm \text{0.03}$
					& $\text{0.21} \pm \text{0.08}$\\
					\midrule

					\addlinespace[0.75ex]
					
					\multirow{5}{*}{
					    \rotatebox[origin=c]{90}{$\tau_{\textup{CorNet}}$ + $\mathbf{w}^u$}
					}
					&\makecell[l]{(vi)} 
					& \xmark
					& \xmark
					& \xmark
					&$\text{0.60} \pm \text{0.10}$
					& $\text{0.48} \pm \text{0.05}$
					& $\text{0.31} \pm \text{0.09}$\\
					\addlinespace[0.75ex]
				    &\makecell[l]{(vii)} 
					& \cmark
					& \xmark
					& \xmark
				    & $\text{0.57} \pm \text{0.10}$
					& $\text{0.47} \pm \text{0.05}$
					& $\text{0.29} \pm \text{0.08}$\\
					\addlinespace[0.75ex]
					&\makecell[l]{(viii)} 
					& \xmark
					& \cmark
					& \xmark
					& $\text{0.33} \pm \text{0.08}$
					& $\text{0.36} \pm \text{0.03}$
					& $\text{0.14} \pm \text{0.06}$\\
					\addlinespace[0.75ex]
					&\makecell[l]{(ix)} 
					& \cmark
					& \xmark
					& \cmark
					&$\text{0.31} \pm \text{0.04}$
					& $\text{0.39} \pm \text{0.04}$
					& $\text{0.14} \pm \text{0.04}$\\
					\addlinespace[0.75ex]
					&\makecell[l]{(x)} 
					& \xmark
					& \cmark
					& \cmark
					&$\text{0.53} \pm \text{0.07}$
					& $\text{0.51} \pm \text{0.06}$
					& $\text{0.33} \pm \text{0.06}$\\
					\bottomrule	
			\end{tabular}
	\end{center}
\end{table*}

\section{Discussion}
Individual-level decision-making is of great importance in many domains such as marketing \citep{hatt2020early}, economics \citep{Heckman1997}, or medicine \citep{ozyurt2021attdmm}. For this, heterogeneous treatment effects have to be estimated in order to personalize treatment decisions. In this paper, we propose a two-step procedure for estimating heterogeneous treatment effects, which combines observational and randomized data. In the first step, we use observational data to learn a shared representation and a confounded (\ie, biased) hypothesis. In the second step, we use the randomized data to debias the confounded hypothesis.

We prove finite sample learning bounds, which offer useful insights into which factors drive the estimation error when the size of randomized data is small. In particular, this reveals three driving factors affecting the error bound: (i)~The size of observational data affects the error bound positively (\ie, the more observational data the better). (ii)~The discrepancy between the covariate distributions in the observational and randomized data as well as (iii)~the complexity of the bias function both negatively affect the error bound. We also prove finite sample learning bounds for several natural baselines which use either observational or randomized data or both. This allows us to compare our approach to these natural baselines and to derive conditions for when it if beneficial to combine observational and randomized data and when it is not. In particular, we find that the same factors as above affect these conditions. For instance, we reveal that if the distributional discrepancy is large, combining observational and randomized data is more beneficial than only using randomized data. This is particularly useful for pharmaceutical and medical sciences. If an RCT is poorly design in the sense that it may not generalize to the population of interest, combining the data from the RCT and observational data may substantially benefit the estimation of heterogeneous treatment effects. We also interpret these conditions in terms of sample complexity. This reveals that, under certain conditions, when combining observational and randomized data, we required less randomized data in order to achieve a predefined estimation error. This is particularly interesting, since RCTs are costly, and, hence, if we can reduce the required number of randomized samples, we can also reduce the costs of RCTs and accelerate drug approval.
Based on these theoretical insights, we propose a sample-efficient algorithm based on our two-step procedure, called CorNet. We perform extensive simulation studies, which empirically verify the finite sample properties of CorNet as discussed above. Moreover, we study the robustness of CorNet against alterations of our setup. We find that even if observational and randomized data are generated from different representations (not a shared representation), CorNet yields robust results. Moreover, if observational data is unconfounded, CorNet yields similar results compared to only using observational data. This suggests that there is no drawback of using CorNet over method that only use observational data, even if observational data is unconfounded. Finally, if the population overlap assumption is violated, we find that the performance degradation of CorNet is due to the larger distributional discrepancy, but not the assumption violation itself. We further use multiple real-world datasets to compare CorNet against a variety of baseline methods. We find that CorNet outperforms all baseline methods by a substantial margin. This finding is consistent across different sizes of the observational data. Once we can acquire large amounts of randomized data, methods that only use randomized data perform similarly to CorNet. However, this means that we leave the setting for which CorNet was design (namely, small randomized data as in RCTs). Finally, in an ablation study, we find that the different components of CorNet (and extensions of it) improve upon the unregularized CorNet depending on the data at hand.

\clearpage
\bibliography{library}

\clearpage
\appendix
\onecolumn

\section{Model Complexity}
We first introduce measure for the complexity of function classes. We define the Gaussian complexity, which is used as a measure for the complexity of a function class. For a function class $\mathcal{Q}$ with functions $\textup{q}:\mathbb{R}^d \rightarrow \mathbb{R}^r$ and $n$ data points, the empirical Gaussian complexity is defined as
\begin{equation}
	\hat{\mathcal{G}}_{\bar{X}}(\mathcal{Q})=\E_{\mathbf{g}}\left[\sup_{\textup{q}\in\mathcal{Q}} \frac1n\sm k r \sm i n g_{ki}\,q_k(\mathbf{x}_i)\right],
\end{equation}
where $g_{ki}\overset{\text{iid}}{\sim}\mathcal{N}(0, 1)$. The corresponding population quantity is defined as $\mathcal{G}_n(\mathcal{Q})=\E_{X}[\hat{\mathcal{G}}_{\bar{X}}(\mathcal{Q})]$. The first expectation is taken over the distribution of the data samples, and the second expectation is taken over the Gaussian random variables. We further define the worst-case Gaussian complexity as
\begin{align}
	\WGC{n}{\mathcal{H}} = \max_{\bar{Z}\in\mathcal{Z}}\hat{\mathcal{G}}_{\bar{Z}}(\HC),\quad \mathcal{Z} = \{(\phi(x_1), \ldots, \phi(x_n)); \phi\in\Phi, x_i\in\mathcal{X} \text{ for all } i \in \{1, \ldots, n\}\}.
\end{align}
Moreover, we define the empirical Rademacher complexity as
\begin{align}
    \hat{\mathcal{R}}_n(\mathcal{Q}) = \E_\sigma\left[\sup_{q\in\mathcal{Q}} \left\lvert \frac1n \sm i n \sigma_i\,q(x_i)\right\rvert\right],
\end{align}
where $\sigma_i$ are Rademacher variables (\ie, taking the values $\{-1, +1\}$ equiprobably). The corresponding population quantity is defined by taking the expectation over $X$, \ie, $\mathcal{R}_n(\mathcal{Q}) = \E_X[\hat{\mathcal{R}}_n(\mathcal{Q})] = \E_{X, \sigma}\left[\sup_{q\in\mathcal{Q}} \left\lvert \frac1n \sm i n \sigma_iq(x_i)\right\rvert\right]$.

\section{Feedforward Neural Networks}\label{apx:ffnn}
The class of feedforward neural network $\mathcal{N}$ is defined by recursion:
\begin{equation}\label{eq:nn}
\mathcal{N} = \{f_d:\,f_1(x) = \mathbf{W_1^\top} x, \qquad f_i(x) = \mathbf{W_i^\top} \sigma \fprns{f_{i-1}(x)}, \quad i=2,\ldots,d\},
\end{equation}
for $d\in\mathbb N_{\geq 1}$, matrices $\{\mathbf{W_i}\}_{i=1}^d$ of appropriate dimensions, activation function $\sigma$, which applies element-wise, and $d$ is the depth of the neural network.

Then, we define the function classes for the representation, $\RC$, confounded hypotheses, $\HC$, and bias functions, $\mathcal{B}$ as follows:
\begin{align}\label{eq:nn_rep}
\Phi = \{\phi \in\mathcal{N};\, &\lVert \mathbf{W_k} \rVert_{1, \infty}\leq \Omega_k \text{ for } k\in\{1, \ldots, K-1\},\nonumber\\ 
&\max(\lVert \mathbf{W_K} \rVert_{1, \infty}, \lVert \mathbf{W_K} \rVert_{2\rightarrow\infty})\leq \Omega,\nonumber\\
&\text{ s.t. } \sigma_r(\E[\phi(X^{\conf})\phi(X^{\conf})^\top])\geq \Omega(1)\},
\end{align}
\begin{align}\label{eq:nn_hypo}
\mathcal{H} = \{h \in\mathcal{N};\, &\lVert \mathbf{W_d} \rVert_{1, \infty}\leq \Omega_d \text{ for } d\in\{1, \ldots, D\}\},
\end{align}
\begin{align}\label{eq:nn_bias}
\mathcal{B} = \{\delta \in\mathcal{N};\, &\lVert \mathbf{W_b} \rVert_{1, \infty}\leq \Omega_b \text{ for } b\in\{1, \ldots,B\}\},
\end{align}
where for a matrix $\mathbf{W}\in \mathbb{R}^{p\times q}$, we let $\lVert \mathbf{W} \rVert_{1, \infty} = \max_{i}\sum_{j}\lvert\mathbf{W}_{ij}\rvert$ and $\lVert \mathbf{W} \rVert_{2\rightarrow\infty}$ the induced $\infty$-to-$2$ operator norm.

\Cref{fig:procedure_illlustration} illustrates the neural network architecture of our two-step procedure, which is used in \Cref{sec:error_bounds}.
\begin{figure}[ht]
	\centering 
	\scalebox{0.4}{\includegraphics{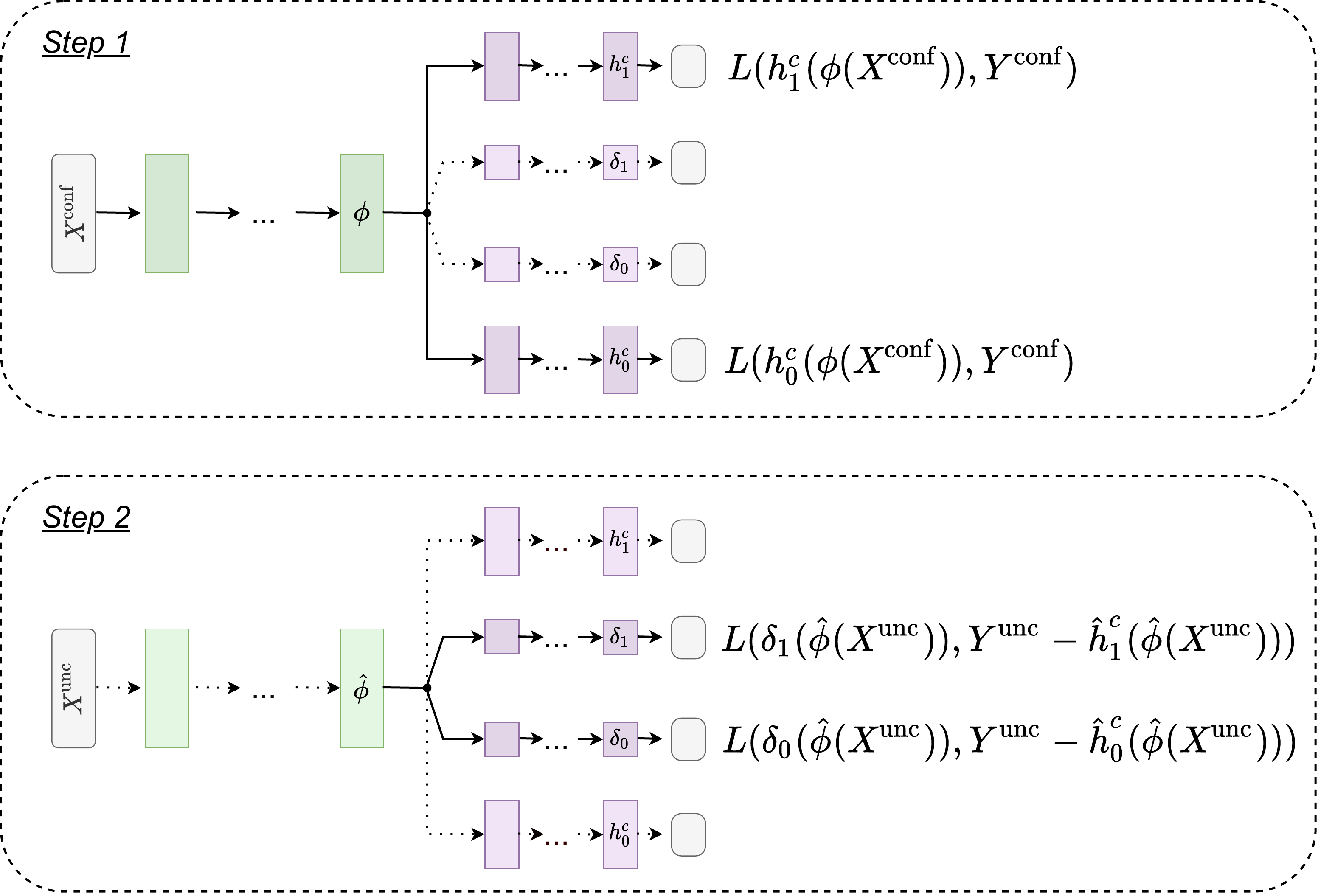}}
	\caption{\footnotesize Illustration of our two-step procedure using neural networks. In \underline{Step 1}, the observational data is used to learn a representation, $\hat{\phi}$, and confounded hypotheses, $\hat{\mathbf{h}}^c$. In \underline{Step 2}, the randomized data, together with the $\hat{\phi}$ and $\hat{\mathbf{h}}^c$ form the first step, is used to learn the bias, $\boldsymbol{\delta}$, in the confounded hypotheses due to unobserved confounding.}\label{fig:procedure_illlustration}
\end{figure}

\section{Proofs in \texorpdfstring{\Cref{sec:error_bounds}}{}}
\subsection{Proof of \texorpdfstring{\Cref{thm_bound}}{}}\label{apx:thm_gen_bound}
In this section, we prove \Cref{thm_bound}. To do so, we first prove two lemmas that are needed for the proof. Then, we prove a version of \Cref{thm_bound} for arbitrary function classes $\RC$, $\HC$, and $\mathcal{B}$. Finally, we can then prove \Cref{thm_bound} for feedforward neural networks.


\begin{lemma}\label{lemma:unc_upper_bnd}
Let $\phi\in\RC$ and $\mathbf{h}\in\HC^{\otimes 2}$. Then, for $\emerrorUC{\boldsymbol{\delta}}{\mathbf{h}}{\phi}$ as in \labelcref{eq:unc_risk} and 
\begin{align}
    \epsilon(\boldsymbol{\delta}, \mathbf{h}, \phi)=\sum_{t=0}^{1}\left[((\textup{h}_t+\delta_t)\circ\phi(X^{\conf}) - Y(t))\right],
\end{align}
we have that, if Assumption~\labelcref{assum:realizability} holds, 
\begin{align}
    \sup_{\boldsymbol{\delta}\in\mathcal{B}^{\otimes 2}}\lvert \epsilon(\boldsymbol{\delta}, \mathbf{h}, \phi) - \emerrorUC{\boldsymbol{\delta}}{\mathbf{h}}{\phi} \rvert \leq 8L\mathit{d}_\infty(p^{\conf}_\phi\mid p^{\unc}_\phi) \GC{n^{\unc}}{\mathcal{B}} + 4\mathit{d}_\infty(p^{\conf}_\phi\mid p^{\unc}_\phi)B\sqrt{\frac{\log(1/p)}{n^{\unc}}},
\end{align}
with probability at least $1-p$.
\end{lemma}
\begin{proof}



The proof relies on a modified version of Theorem 4.10 in \citep{wainwright2019high} and uses a modification of the bounded differences inequality and a standard symmetrization argument. For ease of notation, we use $l_{\mathbf{h}, \phi}(\boldsymbol{\delta}, x,y,t) = ((\textup{h}_t + \delta_t)\circ\phi(x) - y)^2$.

Then, note that for, any $\boldsymbol{\delta}\in\mathcal{B}^{\otimes 2}, \mathbf{h}\in\HC^{\otimes 2}$, and $\phi\in\RC$,
\begin{align}
    \epsilon(\boldsymbol{\delta}, \mathbf{h}, \phi) &= \sum_{t=0}^{1}\E\left[((\textup{h}_t+\delta_t)\circ\phi(X^{\conf}) - Y(t))^2\right]\\ 
    &= \sum_{t=0}^{1}\mathbb{P}(T^{\unc}=t)\E\left[\E[((\textup{h}_t+\delta_t)\circ\phi(X^{\conf}) - Y^{\unc})^2\mid T^{\unc}=t, X^{\conf}]\right]\\
    &=\E\left[\E[((\textup{h}_{T^{\unc}}+\delta_{T^{\unc}})\circ\phi(X^{\conf}) - Y^{\unc})^2\mid T^{\unc}, X^{\conf}]\right]\\&=\E\left[((\textup{h}_{T^{\unc}}+\delta_{T^{\unc}})\circ\phi(X^{\conf}) - Y^{\unc})^2\right]\\
    &=\E[l_{\mathbf{h},\phi}(\boldsymbol{\delta}, X^{\conf}, Y^{\unc}, T^{\unc})].
\end{align}

This yields 
\begin{align}
    \sup_{\boldsymbol{\delta}\in\mathcal{B}^{\otimes 2}}\lvert &\epsilon(\boldsymbol{\delta}, \mathbf{h}, \phi) - \emerrorUC{\boldsymbol{\delta}}{\mathbf{h}}{\phi} \rvert\\
    &= \sup_{\boldsymbol{\delta}\in\mathcal{B}^{\otimes 2}}\left\lvert \frac1{n^{\unc}} \sm i {n^{\unc}} l_{\mathbf{h},\phi}(\boldsymbol{\delta}, x^{\unc}_i, y^{\unc}_i, t^{\unc}_i)  - \E[l_{\mathbf{h},\phi}(\boldsymbol{\delta}, X^{\conf}, Y^{\unc}, T^{\unc})]\right\rvert.
\end{align}
We further simplify notation using $\bar{l}_{\mathbf{h},\phi}(\boldsymbol{\delta}, x_i^{\unc}, y_i^{\unc}, t_i^{\unc}) = l_{\mathbf{h},\phi}(\boldsymbol{\delta},x_i^{\unc}, y_i^{\unc}, t_i^{\unc})  - \E[l_{\mathbf{h},\phi}(\boldsymbol{\delta}, X^{\conf}, Y^{\unc}, T^{\unc})]$. Then, the above can be written as
\begin{equation}
    \sup_{\boldsymbol{\delta}\in\mathcal{B}^{\otimes 2}}\left\lvert \frac1{n^{\unc}} \sm i {n^{\unc}} \bar{l}_{\mathbf{h},\phi}(\boldsymbol{\delta}, x_i^{\unc}, y_i^{\unc}, t_i^{\unc})\right\rvert.
\end{equation}
Thinking of the samples as fixed for the moment, we consider the function
\begin{align}
    G(d_1, \ldots, d_{n^{\unc}}) := \sup_{\boldsymbol{\delta}\in\mathcal{B}^{\otimes 2}}\left\lvert \frac1{n^{\unc}} \sm i {n^{\unc}} \bar{l}_{\mathbf{h},\phi}(\boldsymbol{\delta}, x_i^{\unc}, y_i^{\unc}, t_i^{\unc})\right\rvert,
\end{align}
where $d_i=(x_i^{\unc}, y_i^{\unc}, t_i^{\unc})$. We prove that $G$ is Lipschitz such that we can apply the bounded difference inequality (see Corollary 2.21 in \citet{wainwright2019high}). Since the function $G$ is invariant to permutation of its coordinates, it suffices to bound the difference when the first coordinate $d_1$ is perturbed. For this, we define the sample vector $(w_1, \ldots, w_{n^{\unc}})$ such that $w_i = d_i$ for $i\neq 1$. Then, we seek a bound for the difference $\lvert G(d_1, \ldots, d_{n^{\unc}}) - G(w_1, \ldots, w_{^{\unc}})\rvert$. For any $\mathbf{h}\in\HC^{\otimes 2}, \phi\in\Phi, \boldsymbol{\delta}, \mathbf{f}\in\mathcal{B}^{\otimes 2}$, and $\bar{l}_{\mathbf{h},\phi}(\boldsymbol{\delta}, x_i^{\unc}, y_i^{\unc}, t_i^{\unc})$, we have that
\begin{align}
    &\left\lvert\frac1{n^{\unc}}\sm i {n^{\unc}} \bar{l}_{\mathbf{h},\phi}(\boldsymbol{\delta}, d_i)\right\rvert -\sup_{\mathbf{f}\in\mathcal{B}^{\otimes 2}}\left\lvert\frac1{n^{\unc}}\sm i {n^{\unc}} \bar{l}_{\mathbf{h},\phi}(\mathbf{f}, w_i)\right\rvert\\
    &\leq \left\lvert\frac1{n^{\unc}}\sm i {n^{\unc}} \bar{l}_{\mathbf{h},\phi}(\boldsymbol{\delta}, d_i)\right\rvert -\left\lvert\frac1{n^{\unc}}\sm i {n^{\unc}} \bar{l}_{\mathbf{h},\phi}(\boldsymbol{\delta}, w_i)\right\rvert\\
    &\leq \frac1{n^{\unc}} \left\lvert \bar{l}_{\mathbf{h},\phi}(\boldsymbol{\delta}, d_1) - \bar{l}_{\mathbf{h},\phi}(\boldsymbol{\delta}, w_1) \right\rvert=\frac1{n^{\unc}} \left\lvert l_{\mathbf{h},\phi}(\mathbf{h}, d_1) - l_{\mathbf{h},\phi}(\mathbf{h}, w_1) \right\rvert\\
    &\leq \frac{2B}{{n^{\unc}}},
\end{align}
where the last inequality follows from the boundedness of the squared loss. Since the above holds for any function $\boldsymbol{\delta}\in\mathcal{B}^{\otimes 2}$, we can take the supremum over $\boldsymbol{\delta}\in\mathcal{B}^{\otimes 2}$ on each sides. This yields $G(d_1, \ldots, d_{n^{\unc}}) - G(w_1, \ldots, w_{n^{\unc}})\leq\frac{2B}{{n^{\unc}}}$. Since the same argument can be applied with $(d_1, \ldots, d_{n^{\unc}})$ and $(w_1, \ldots, w_{n^{\unc}})$ reversed, we conclude that $\lvert G(d_1, \ldots, d_{n^{\unc}}) - G(w_1, \ldots, w_{n^{\unc}})\rvert\leq\frac{2B}{{n^{\unc}}}$. Hence, we can apply the bounded differences inequality (see Corollary 2.21 in \citep{wainwright2019high}), which yields
\begin{align}
    \sup_{\boldsymbol{\delta}\in\mathcal{B}^{\otimes 2}}\left\lvert \frac1{n^{\unc}} \sm i {n^{\unc}} \bar{l}_{\mathbf{h},\phi}(\boldsymbol{\delta}, x_i^{\unc}, y_i^{\unc}, t_i^{\unc})\right\rvert - &\E\bigg[    \sup_{\boldsymbol{\delta}\in\mathcal{B}^{\otimes 2}}\bigg\lvert \frac1{n^{\unc}} \sm i {n^{\unc}} \bar{l}_{\mathbf{h},\phi}(\boldsymbol{\delta}, x_i^{\unc}, y_i^{\unc}, t_i^{\unc}) \bigg\rvert\bigg]\\
    &\leq B\sqrt{\frac{2\log(1/p)}{n^{\unc}}},
\end{align}
with probability at least $1-p$.

We have now proven that, when 
$\mathcal{B}$ is uniformly bounded, then the random variable $\sup_{\boldsymbol{\delta}\in\mathcal{B}^{\otimes 2}}\left\lvert \frac1{n^{\unc}} \sm i {n^{\unc}} \bar{l}_{\mathbf{h},\phi}(\boldsymbol{\delta}, x_i^{\unc}, y_i^{\unc}, t_i^{\unc})\right\rvert$ is sharply concentrated around its expected value. It remains to upper bound $\E\left[    \sup_{\boldsymbol{\delta}\in\mathcal{B}^{\otimes 2}}\left\lvert \frac1{n^{\unc}} \sm i {n^{\unc}} \bar{l}_{\mathbf{h},\phi}(\boldsymbol{\delta}, x_i^{\unc}, y_i^{\unc}, t_i^{\unc})\right\rvert\right]$. We show that this quantity is upper bounded using classical symmetrization arguments.

For this, let $\{(X^{\unc}_1, Y^{\unc}_1, T^{\unc}_1)$, \ldots, $(X^{\unc, \prime}_{n^{\unc}}, Y^{\unc}_{n^{\unc}}, T^{\unc}_{n^{\unc}})\}$ and $\{(X^{\conf}_1, Y^{\unc, \prime}_1, T^{\unc, \prime}_1)$, \ldots, $(X^{\conf}_{n^{\unc}}, Y^{\unc, \prime}_{n^{\unc}}, T^{\unc, \prime}_{n^{\unc}})\}$ be a second \textit{i.i.d.} sequence independent of the first.

Further, let $w_\phi(z) = \frac{p^{\conf}_\phi(z)}{p^{\unc}_\phi(z)}$ and $\Psi$ be the inverse of $\phi$. Then, together with Lemma~\ref{lemma:change_of_variable}, this yields
\begin{align}
    &\E\left[    \sup_{\boldsymbol{\delta}\in\mathcal{B}^{\otimes 2}}\left\lvert \frac1{n^{\unc}} \sm i {n^{\unc}} \bar{l}_{\mathbf{h},\phi}(\boldsymbol{\delta}, x_i^{\unc}, y_i^{\unc}, t_i^{\unc})\right\rvert\right]\\
    &=\E\left[\sup_{\boldsymbol{\delta}\in\mathcal{B}^{\otimes 2}}\left\lvert \frac1{n^{\unc}} \sm i {n^{\unc}} l_{\mathbf{h},\phi}(\boldsymbol{\delta},x_i^{\unc}, y_i^{\unc}, t_i^{\unc})  - \E[l_{\mathbf{h},\phi}(\boldsymbol{\delta}, X_i^{\conf}, Y_i^{\unc, \prime}, T_i^{\unc, \prime})]\right\rvert\right]\\   
    &=\E\left[\sup_{\boldsymbol{\delta}\in\mathcal{B}^{\otimes 2}}\left\lvert \frac1{n^{\unc}} \sm i {n^{\unc}} l_{\mathbf{h},\phi}(\boldsymbol{\delta},\phi(\Psi(z_i^{\unc})), y_i^{\unc}, t_i^{\unc})  - \underbrace{\E[l_{\mathbf{h},\phi}(\boldsymbol{\delta}, \phi(\Psi(Z_i^{\conf})), Y_i^{\unc, \prime}, T_i^{\unc, \prime})]}_{=\E[w(\phi(\Psi(Z_i^{\unc})))l_{\mathbf{h},\phi}(\boldsymbol{\delta}, \phi(\Psi(Z_i^{\unc})), Y_i^{\unc, \prime}, T_i^{\unc, \prime})]}\right\rvert\right]\\
    &=\E\left[\sup_{\boldsymbol{\delta}\in\mathcal{B}^{\otimes 2}}\left\lvert \frac1{n^{\unc}} \sm i {n^{\unc}} l_{\mathbf{h},\phi}(\boldsymbol{\delta},z_i^{\unc}, y_i^{\unc}, t_i^{\unc})  - \underbrace{\E[w(\phi(\Psi(Z_i^{\unc})))l_{\mathbf{h},\phi}(\boldsymbol{\delta}, \phi(\Psi(Z_i^{\unc})), Y_i^{\unc, \prime}, T_i^{\unc, \prime})]}_{=\E[w(Z_i^{\unc})l_{\mathbf{h},\phi}(\boldsymbol{\delta}, Z_i^{\unc}, Y_i^{\unc, \prime}, T_i^{\unc, \prime})]}\right\rvert\right]\\
    &=\E\left[\sup_{\boldsymbol{\delta}\in\mathcal{B}^{\otimes 2}}\left\lvert \E[\frac1{n^{\unc}} \sm i {n^{\unc}} l_{\mathbf{h},\phi}(\boldsymbol{\delta},z_i^{\unc}, y_i^{\unc}, t_i^{\unc, \prime})  - {w(z_i^{\unc})l_{\mathbf{h},\phi}(\boldsymbol{\delta}, z_i^{\unc}, y_i^{\unc, \prime}, t_i^{\unc, \prime})]}\right\rvert\right]\\
    &\leq\E\left[\sup_{\boldsymbol{\delta}\in\mathcal{B}^{\otimes 2}}\left\lvert \frac1{n^{\unc}} \sm i {n^{\unc}} l_{\mathbf{h},\phi}(\boldsymbol{\delta},z_i^{\unc}, y_i^{\unc}, t_i^{\unc})  - {w(z_i^{\conf})l_{\mathbf{h},\phi}(\boldsymbol{\delta}, z_i^{\unc}, y_i^{\unc, \prime}, t_i^{\unc, \prime})}\right\rvert\right].
\end{align}
Let $(\sigma_1, \ldots, \sigma_{n^{\unc}})$ be \textit{i.i.d.} Rademacher variables, independent of other random variables. Then, the random vector with entries $\sigma_i(l_{\mathbf{h},\phi}(\boldsymbol{\delta},z_i^{\unc}, y_i^{\unc}, t_i^{\unc})  - {w(z_i^{\unc})l_{\mathbf{h},\phi}(\boldsymbol{\delta}, z_i^{\unc}, y_i^{\unc, \prime}, t_i^{\unc, \prime})})$ has the same joint distribution as the random vector with entries $l_{\mathbf{h},\phi}(\boldsymbol{\delta},z_i^{\unc}, y_i^{\unc}, t_i^{\unc})  - {w(z_i^{\unc})l_{\mathbf{h},\phi}(\boldsymbol{\delta}, z_i^{\unc}, y_i^{\unc, \prime}, t_i^{\unc, \prime})}$. Thus, 
\begin{align}
    &\E\left[\sup_{\boldsymbol{\delta}\in\mathcal{B}^{\otimes 2}}\left\lvert \frac1{n^{\unc}} \sm i {n^{\unc}} l_{\mathbf{h},\phi}(\boldsymbol{\delta},z_i^{\unc}, y_i^{\unc}, t_i^{\unc})  - {w(z_i^{\unc})l_{\mathbf{h},\phi}(\boldsymbol{\delta}, z_i^{\unc}, y_i^{\unc, \prime}, t_i^{\unc, \prime})}\right\rvert\right]\\
    &\leq\E_{\sigma}\left[\sup_{\boldsymbol{\delta}\in\mathcal{B}^{\otimes 2}}\left\lvert \frac1{n^{\unc}} \sm i {n^{\unc}} \sigma_i(l_{\mathbf{h},\phi}(\boldsymbol{\delta},z_i^{\unc}, y_i^{\unc}, t_i^{\unc})  - {w(z_i^{\unc})l_{\mathbf{h},\phi}(\boldsymbol{\delta}, z_i^{\unc}, y_i^{\unc, \prime}, t_i^{\unc, \prime})})\right\rvert\right]\\
    &\leq 2\mathit{d}_\infty(p^{\conf}_\phi\mid p^{\unc}_\phi)\E_{\sigma}\left[\sup_{\boldsymbol{\delta}\in\mathcal{B}^{\otimes 2}}\left\lvert \frac1{n^{\unc}} \sm i {n^{\unc}} \sigma_i\,l_{\mathbf{h},\phi}(\boldsymbol{\delta},z_i^{\unc}, y_i^{\unc}, t_i^{\unc})\right\rvert\right]\\
    & = 2\mathit{d}_\infty(p^{\conf}_\phi\mid p^{\unc}_\phi) \mathcal{R}_{n^{\unc}}(l(\mathcal{B}^{\otimes 2})),
\end{align}
using $w_\phi(z) = \frac{p^{\conf}_\phi(z)}{p^{\unc}_\phi(z)} \leq \sup_{z}\frac{p^{\conf}_\phi(z)}{p^{\unc}_\phi(z)} = \mathit{d}_\infty(p^{\conf}_\phi\mid p^{\unc}_\phi)$.

Since the term $\mathcal{R}_{n^{\unc}}(l(\mathcal{B}^{\otimes 2}))$ includes the loss function, we have to further decompose it. Since the square loss is uniformly bounded by $B$, we can center the square loss via $l_{\mathbf{h},\phi}^i(\boldsymbol{\delta}, x_i, y_i, t_i) = ((\textup{h}_{t_i}+\delta_{t_i})\circ \phi(x_i) - y_i)^2 - (\textup{h}_{t_i}\circ \phi(x_i) - y_i)^2$. Then, using that $(y_i - \textup{h}_{t_i}\circ\phi(x_i))^2 = l_{\mathbf{h}, \phi}(\boldsymbol{\delta}, x, y, t)$ with $\delta_t(z)=\mathbf{0}$ and, therefore, $(y - \textup{h}_t\circ\phi(x))^2\leq B$ and using the constant-shift property of Rademacher complexity (\cite{wainwright2019high}; Exercise 4.7c) yields 
\begin{align}
&\E_{\sigma}\left[\sup_{\boldsymbol{\delta}\in\mathcal{B}^{\otimes 2}} \frac1{n^{\unc}} \sm i {n^{\unc}} \sigma_i l_{\mathbf{h},\phi}(\boldsymbol{\delta}, z_i^{\unc}, y_i^{\unc}, t_t^{\unc})\right]\\
&\leq \E_{\sigma}\left[\sup_{\boldsymbol{\delta}\in\mathcal{B}^{\otimes 2}} \frac1{n^{\unc}} \sm i {n^{\unc}} \sigma_i l_{\mathbf{h}, \phi}^i(\boldsymbol{\delta}, z_i^{\unc}, y_i^{\unc}, t_t^{\unc})\right] + \frac{B}{\sqrt{n^{\unc}}}.
\end{align}
Since the loss is L-Lipschitz and, now, also centered, we can apply contraction principle (\cite{ledoux2013probability}; Theorem~4.12), which yields
\begin{align}\label{eq:ledoux_contraction}
E_{\sigma}\left[\sup_{\boldsymbol{\delta}\in\mathcal{B}^{\otimes 2}} \frac1{n^{\unc}} \sm i {n^{\unc}} \sigma_i l_{\mathbf{h}, \phi}^i(\boldsymbol{\delta}, z_i^{\unc}, y_i^{\unc}, t_i^{\unc})\right]
\leq 2L\mathcal{R}_{n^{\unc}}(\mathcal{B}^{\otimes 2}).
\end{align}
Finally, note that $\mathcal{R}_{n^{\unc}}(\mathcal{B}^{\otimes 2}) = \mathcal{R}_{n^{\unc}}(\mathcal{B})$, since in \labelcref{eq:ledoux_contraction}, the supremum is over $\delta_1\in\mathcal{B}$ if $t_i=1$ and over $\delta_0\in\mathcal{B}$ if $t_i=0$. Hence, the complexity is equivalent to the supremum over $\mathcal{B}$. 
Putting everything together, we obtain 
\begin{align}
    &\sup_{\boldsymbol{\delta}\in\mathcal{B}^{\otimes 2}}\lvert \epsilon(\boldsymbol{\delta}, \mathbf{h}, \phi) - \emerrorUC{\boldsymbol{\delta}}{\mathbf{h}}{\phi} \rvert = \sup_{\boldsymbol{\delta}\in\mathcal{B}^{\otimes 2}}\left\lvert \frac1{n^{\unc}} \sm i {n^{\unc}} \bar{l}_{h,\phi}(\boldsymbol{\delta}, x_i^{\unc}, y_i^{\unc}, t_i^{\unc})\right\rvert\\
    &\leq  \E\left[    \sup_{\boldsymbol{\delta}\in\mathcal{B}^{\otimes 2}}\left\lvert \frac1{n^{\unc}} \sm i {n^{\unc}} \bar{l}_{h,\phi}(\boldsymbol{\delta}, x_i^{\unc}, y_i^{\unc}, t_i^{\unc})\right\rvert\right] + B\sqrt{\frac{2\log(1/p)}{n^{\unc}}}\\
    &\leq  2\mathit{d}_\infty(p^{\conf}_\phi\mid p^{\unc}_\phi) \mathcal{R}_{n^{\unc}}(l(\mathcal{B}^{\otimes 2})) + B\sqrt{\frac{2\log(1/p)}{{n^{\unc}}}}\\
    &\leq 2\mathit{d}_\infty(p^{\conf}_\phi\mid p^{\unc}_\phi) (2L\mathcal{R}_{n^{\unc}}(\mathcal{B})+ \frac{B}{\sqrt{{n^{\unc}}}}) + B\sqrt{\frac{2\log(1/p)}{{n^{\unc}}}}\\
    &\leq 4L\mathit{d}_\infty(p^{\conf}_\phi\mid p^{\unc}_\phi) \mathcal{R}_{n^{\unc}}(\mathcal{B}) + 2\mathit{d}_\infty(p^{\conf}_\phi\mid p^{\unc}_\phi)B\sqrt{\frac{2\log(1/p)}{{n^{\unc}}}}
\end{align}

The Rademacher complexity can be upper bounded by the Gaussian complexity (\cite{ledoux2013probability}; p. 97): $\E\left[\hat{\mathcal{R}}_{Z^{\unc}}(\mathcal{B})\right]\leq \sqrt{\frac{\pi}{2}}\E\left[\hat{\mathcal{G}}_{Z^{\unc}}(\mathcal{B})\right]$. Combining this with the result above yields
\begin{align}
    &\sup_{\boldsymbol{\delta}\in\mathcal{B}^{\otimes 2}}\lvert \epsilon(\boldsymbol{\delta}, \mathbf{h}, \phi) - \emerrorUC{\boldsymbol{\delta}}{\mathbf{h}}{\phi} \rvert\\
    &\leq 4L\mathit{d}_\infty(p^{\conf}_\phi\mid p^{\unc}_\phi) \sqrt{\frac{\pi}{2}}\GC{n^{\unc}}{\mathcal{B}} + 2\mathit{d}_\infty(p^{\conf}_\phi\mid p^{\unc}_\phi)B\sqrt{\frac{2\log(1/p)}{{n^{\unc}}}}
\end{align}
with probability at least $1-p$. 



% ...

\end{proof}
\begin{lemma}\label{lemma:conf_upper_bnd}
 Let $\hat{\mathbf{h}}, \hat{\phi} = \argmin_{\mathbf{h}\in\HC^{\otimes 2}, \phi\in\Phi}\emriskC{\mathbf{h}}{\phi}$ and $\tilde{\mathbf{h}} = \argmin_{\mathbf{h}\in\HC^{\otimes 2}}\riskC{\mathbf{h}}{\hat{\phi}}$. Then, under the assumptions of \Cref{thm_bound}, with probability at least $1-p$,
\begin{align}
	\riskC{\hat{\mathbf{h}}^c}{\hat{\phi}} - \riskC{\mathbf{h}^c}{\phi^\ast} \leq 2048L\bigg(\frac{D_{\mathcal{X}}}{{(n^{\conf}})^2} + \log(n^{\conf})(L(\HC) &\GC{n^{\conf}}{\RC}+\WGC{n^{\conf}}{\HC})\bigg)\notag\\
	&+ 8B\sqrt{\frac{\log(1/p)}{n^{\conf}}}.
\end{align}
\end{lemma}
\begin{proof}
It follows that,
\begin{align}
    &\riskC{\hat{\mathbf{h}}^c}{\hat{\phi}} - \riskC{\mathbf{h}^c}{\phi^\ast}\\ &= \riskC{\hat{\mathbf{h}}^c}{\hat{\phi}} - \emriskC{\hat{\mathbf{h}}^c}{\hat{\phi}} + \underbrace{\emriskC{\hat{\mathbf{h}}^c}{\hat{\phi}} -\emriskC{\mathbf{h}^c}{\phi^\ast}}_{\leq 0} +\emriskC{\mathbf{h}^c}{\phi^\ast} - \riskC{\mathbf{h}^c}{\phi^\ast}\\
    &\leq 2\sup_{\mathbf{h}\in\HC^{\otimes 2}, \phi\in\RC}\left\lvert \riskC{\mathbf{h}}{\phi} - \emriskC{\mathbf{h}}{\phi} \right\rvert\\
    &\leq 8L\mathcal{R}_{n^{\conf}}(\HC(\RC)) + \frac{8B\sqrt{2\log(1/p)}}{\sqrt{n^{\conf}}},
\end{align}
with probability at least $1-p$. The first inequality follows be the same symmetrization argument as in the proof of Lemma~\ref{lemma:unc_upper_bnd} (although without the distributional discrepancy).

The Rademacher complexity can be upper bounded by the Gaussian complexity (\cite{ledoux2013probability}, p. 97): $\E[\hat{\mathcal{R}}_{X^{\conf}}(\HC(\RC))]\leq \sqrt{\frac{\pi}{2}}\E[\hat{\mathcal{G}}_{X^{\conf}}(\HC(\RC))]$. For the last step we use the chain rule for Gaussian complexity from Theorem 7 in \cite{tripuraneni2020theory}. This yields
\begin{align}
    \hat{\mathcal{G}}_{X^{\conf}}(\HC(\RC))\leq 128\left(\frac{2D_{\mathcal{X}}}{({n^{\conf}})^2}+C(\HC(\RC))\log(n^{\conf})\right),
\end{align}
where $C(\HC(\RC)) = L(\HC)\hat{\mathcal{G}}_{X^{\conf}}(\RC) + \max_{\bar{Z}\in\mathcal{Z}}\hat{\mathcal{G}}_{\bar{Z}}(\HC)$ and $\mathcal{Z} = \{(\phi(x_1), \ldots, \phi(x_{n^{\conf}})); \phi\in\RC\}$.
Combining this with the result above and taking expectations yields
\begin{align}
	\riskC{\hat{\mathbf{h}}^c}{\hat{\phi}} 
	- \riskC{\mathbf{h}^c}{\phi^\ast} \leq 
	2048L \Bigg(
	\frac{D_{\mathcal{X}}}{({n^{\conf}})^2}+\log(n^{\conf}) (L(\HC)&\GC{n^{\conf}}{\RC} + 
	\WGC{n^{\conf}}{\HC})\Bigg)\\
	&+ 8B\sqrt{\frac{2\log(1/p)}{n^{\conf}}},
\end{align}
with probability $1-p$.
\end{proof}

We are now equipped to prove a version \Cref{thm_bound} for an arbitrary representation function class $\RC$, hypothesis function class $\HC$, and bias function class $\mathcal{B}$.

For this, we assume the following standard, mild regularity conditions on the loss function $l(\cdot,\cdot)$, the hypothesis class $\HC$, and the function class of representations $\Phi$. Note that these assumptions hold true for neural networks, which we prove later.\\
\textbf{Assumption A1} (Regularity Conditions.) We assume the following regularity conditions hold:
\begin{enumerate}
	\item The loss-function $l(x,y)=(x-y)^2$ is $B$-bounded and $l(\cdot, y)$ is L-Lipschitz for all $y\in\Rl$.
	\item The function $\textup{h}_t$ is L($\mathcal{H}$)-Lipschitz with respect to the $l_2$ distance, for any $\textup{h}_t\in\HC$.
	\item The composition function $\textup{h}_t\circ \phi$ is bounded, \ie, $\sup_{x\in\mathbb{R}^d}\abs{\textup{h}_t\circ \phi(x)}\leq D_{X}$ for any $\textup{h}_t\in\mathcal{H}$ and $\phi\in\Phi$. 
\end{enumerate}
\begin{theorem}\label{thm:upper_bnd_general}
		Let $(\hat{\mathbf{h}}^c, \hat{\phi})$ be the empirical loss minimizer of $\hat{\epsilon}_{\conf}(\cdot, \cdot)$ from \labelcref{eq:conf_loss_minimizer} over function classes $\RC$ and $\HC$, and let  $\boldsymbol{\delta}$ be the empirical loss minimizer of $\hat{\epsilon}_{\unc}(\cdot, \hat{\mathbf{h}}^c, \hat{\phi})$ from \labelcref{eq:unc_loss_minimizer} over a function class $\mathcal{B}$. Further, let $\hat{\tau}_{\textup{CorNet}}$ be the resulting CATE estimator from \labelcref{eq:tau_hat}. Then, if Assumption~\labelcref{assum:realizability}, Assumption~A1, and Condition~\labelcref{cond:rep_diff} hold true, we have that, with probability at least $1-p$,
	\begin{align}
		\epsilon_{\textup{PEHE}}(\hat{\tau}_{\cor})\leq \tilde{O}\bigg(&L(\HC) \GC{n^{\conf}}{\RC} + \WGC {n^{\conf}} \HC + \mathit{d}_\infty(p^{\conf}_\phi\mid p^{\unc}_\phi)\,\GC {n^{\unc}} {\mathcal{B}}\notag\\ 
        &+ \frac{D_{\mathcal{X}}}{({n^{\conf}})^2} + \frac B L\left(\sqrt{\frac{2\log(1/p)}{n^{\conf}}} + \mathit{d}_\infty(p^{\conf}_\phi\mid p^{\unc}_\phi)\sqrt{\frac{2\log(1/p)}{n^{\unc}}}\right)\bigg),
	\end{align}
	where $\GC{n^{\conf}}{\RC}$, $\WGC {n^{\conf}} \HC$, and $\GC {n^{\unc}} {\mathcal{B}}$ are the Gaussian complexities and worst-case Gaussian complexities of the corresponding functions classes. Moreover, $\mathit{d}_\infty(p^{\conf}_\phi\mid p^{\unc}_\phi)$ is the exponential in base 2 of the Rényi divergence.
\end{theorem}

\begin{proof}
First,
\begin{align}
    \epsilon_{\textup{PEHE}}(\hat{\tau}_{\cor}) &= \E[(\hat{\tau}(X^{\conf})-\tau(X^{\conf}))^2]\\
    &=\E[((\hat{\textup{h}}_1^c+\hat{\delta}_1)\circ\hat{\phi}(X^{\conf}) - ((\hat{\textup{h}}_0^c+\hat{\delta}_0)\circ\hat{\phi}(X^{\conf}) \\
    &\quad- ((\textup{h}_1^c+\delta_1)\circ\phi^\ast(X^{\conf}) - (\textup{h}_0^c+\delta_0)\circ\phi^\ast(X^{\conf})))^2]\\
    &\leq 2\E[((\hat{\textup{h}}_1^c+\hat{\delta}_1)\circ\hat{\phi}(X^{\conf}) - (\textup{h}^c_1+\delta_1)\circ\phi^\ast(X^{\conf}))^2\\
    &\quad+ ((\textup{h}^c_0+\delta_0)\circ\phi^\ast(X^{\conf}) - (\hat{\textup{h}}_0^c+\hat{\delta}_0)\circ\hat{\phi}(X^{\conf}))^2] \\
    &= 2 \sum_{t=0}^{1}\E[((\hat{\textup{h}}_t^c+\hat{\delta}_t)\circ\hat{\phi}(X^{\conf}) - (\textup{h}^c_t+\delta_t)\circ\phi^\ast(X^{\conf}))^2]\\
    &= 2 \sum_{t=0}^{1}\E[((\hat{\textup{h}}_t^c+\hat{\delta}_t)\circ\hat{\phi}(X^{\conf}) - Y(t))^2 - ((\textup{h}_t^c+\delta_t)\circ\phi^\ast(X^{\conf}) - Y(t))^2]\\
    &= 2(\sum_{t=0}^{1}\E[((\hat{\textup{h}}_t^c+\hat{\delta}_t)\circ\hat{\phi}(X^{\conf}) - Y(t))^2] - \sum_{t=0}^{1}\E[((\textup{h}_t^c+\delta_t)\circ\phi^\ast(X^{\conf}) - Y(t))^2])\\
    &=2(\epsilon(\hat{\boldsymbol{\delta}}, \hat{\mathbf{h}}^c, \hat{\phi}) - \epsilon(\boldsymbol{\delta}, \mathbf{h}^c, \phi^\ast)),
\end{align}
where the inequality follows with $(x+y)^2 \leq 2(x^2 + y^2)$ and the fourth equality follows with Lemma~\ref{lemma:loss_decomp}. 


Let $\hat{\mathbf{h}}^c, \hat{\phi} = \argmin_{\mathbf{h}\in\HC^{\otimes 2}, \phi\in\RC}\emriskC{\mathbf{h}}{\phi}$, and $\tilde{\boldsymbol{\delta}} = \argmin_{\boldsymbol{\delta}\in\mathcal{B}^{\otimes 2}}\epsilon(\boldsymbol{\delta}, \hat{\mathbf{h}}^c, \hat{\phi})$, and $\hat{\boldsymbol{\delta}} = \argmin_{\boldsymbol{\delta}\in\mathcal{B}^{\otimes 2}}\emerrorUC{\boldsymbol{\delta}}{\hat{\mathbf{h}}^c}{\hat{\phi}}$. Then, we have that

\begin{equation}\label{eq:first_part}
\epsilon(\hat{\boldsymbol{\delta}}, \hat{\mathbf{h}}^c, \hat{\phi}) - \epsilon(\boldsymbol{\delta}, \mathbf{h}^c, \phi^\ast)
=\underbrace{\epsilon(\hat{\boldsymbol{\delta}}, \hat{\mathbf{h}}^c, \hat{\phi}) - \epsilon(\tilde{\boldsymbol{\delta}}, \hat{\mathbf{h}}^c, \hat{\phi})}_{=\text{(I)}} + \underbrace{\epsilon(\tilde{\boldsymbol{\delta}}, \hat{\mathbf{h}}^c, \hat{\phi}) - \epsilon(\boldsymbol{\delta}, \mathbf{h}^c, \phi^\ast)}_{=\text{(II)}}.
\end{equation}
For the first term (I), we obtain
\begin{align}
    \text{(I)} = \underbrace{\epsilon(\hat{\boldsymbol{\delta}}, \hat{\mathbf{h}}^c, \hat{\phi}) - \emerrorUC{\hat{\boldsymbol{\delta}}}{\hat{\mathbf{h}}^c}{\hat{\phi}}}_{=(i)} &+ \underbrace{\emerrorUC{\hat{\boldsymbol{\delta}}}{\hat{\mathbf{h}}^c}{\hat{\phi}} -
    \emerrorUC{\tilde{\boldsymbol{\delta}}}{\hat{\mathbf{h}^c}}{\hat{\phi}}}_{\leq 0}\notag\\ 
&+\underbrace{    \emerrorUC{\tilde{\boldsymbol{\delta}}}{\hat{\mathbf{h}}^c}{\hat{\phi}}-
\epsilon(\tilde{\boldsymbol{\delta}}, \hat{\mathbf{h}}^c, \hat{\phi})}_{=(ii)},
\end{align}
where the middle term is negative, since $\hat{\boldsymbol{\delta}} = \argmin_{\boldsymbol{\delta}\in\mathcal{B}^{\otimes 2}}\emerrorUC{\boldsymbol{\delta}}{\hat{\mathbf{h}}^c}{\hat{\phi}}$ and \newline $\tilde{\boldsymbol{\delta}} = \argmin_{\boldsymbol{\delta}\in\mathcal{B}^{\otimes 2}}\epsilon(\boldsymbol{\delta}, \hat{\mathbf{h}}^c, \hat{\phi})$. Further, $(i), (ii) \leq \sup_{\boldsymbol{\delta}\in\mathcal{B}^{\otimes 2}}\lvert \epsilon(\boldsymbol{\delta}, \hat{\mathbf{h}}^c, \hat{\phi}) - \emerrorUC{\boldsymbol{\delta}}{\hat{\mathbf{h}}^c}{\hat{\phi}} \rvert$.

From Lemma~\ref{lemma:unc_upper_bnd}, we know that, with probability at least $1-p$,
\begin{align}
\sup_{\boldsymbol{\delta}\in\mathcal{B}^{\otimes 2}}\lvert \epsilon(\boldsymbol{\delta}, \hat{\mathbf{h}}^c, \hat{\phi}) - \emerrorUC{\boldsymbol{\delta}}{\hat{\mathbf{h}}^c}{\hat{\phi}} \rvert &\leq 8L\mathit{d}_\infty(p^{\conf}_\phi\mid p^{\unc}_\phi) \GC{n^{\unc}}{\mathcal{B}}\notag\\
&+ 4\mathit{d}_\infty(p^{\conf}_\phi\mid p^{\unc}_\phi)B\sqrt{\frac{\log(1/p)}{n^{\unc}}}.
\end{align}

Moreover, for the second term (II), we obtain
\begin{align}
    \text{(II)} &= \inf_{\boldsymbol{\delta}^\prime\in\mathcal{B}^{\otimes 2}}\epsilon(\boldsymbol{\delta}^\prime, \hat{\mathbf{h}}^c, \hat{\phi}) - \epsilon(\boldsymbol{\delta}, \mathbf{h}^c, \phi^\ast)\\
    &= \inf_{\boldsymbol{\delta}^\prime\in\mathcal{B}^{\otimes 2}}\sum_{t=0}^{1}\E[((\hat{\textup{h}}_t^c + \delta_t^\prime)\circ\hat{\phi}(X^{\conf})-Y(t))^2 - ((\textup{h}_t^c + \delta_t)\circ\phi^\ast(X^{\conf})-Y(t))^2]\\
            &= \inf_{\boldsymbol{\delta}^\prime\in\mathcal{B}^{\otimes 2}}\sum_{t=0}^{1}\E[((\hat{\textup{h}}_t^c + \delta_t^\prime)\circ\hat{\phi}(X^{\conf}) - (\textup{h}_t^c + \delta_t)\circ\phi^\ast(X^{\conf}))^2]\\
    &\leq 2\inf_{\boldsymbol{\delta}^\prime\in\mathcal{B}^{\otimes 2}}\sum_{t=0}^{1}\E[(\hat{\textup{h}}_t^c\circ\hat{\phi}(X^{\conf}) - \textup{h}_t^c\circ\phi^\ast(X^{\conf}))^2+    (\delta_t^\prime\circ\hat{\phi}(X^{\conf}) -  \delta_t\circ\phi^\ast(X^{\conf}))^2]\\
&= 2\,\sum_{t=0}^{1}\E[(\hat{\textup{h}}_t^c\circ\hat{\phi}(X^{\conf}) - \textup{h}_t^c\circ\phi^\ast(X^{\conf}))^2+2\,d_{\mathcal{B}, \boldsymbol{\delta}}(\hat{\phi};\phi^\ast)\\
&= 2\,\sum_{t=0}^{1}\E[(\hat{\textup{h}}_t^c\circ\hat{\phi}(X^{\conf}) - Y^{\conf})^2 - (\textup{h}_t^c\circ\phi^\ast(X^{\conf})-Y^{\conf})^2+2\,d_{\mathcal{B}, \boldsymbol{\delta}}(\hat{\phi};\phi^\ast)\\
&= 2(\riskC{\hat{\mathbf{h}}^c}{\hat{\phi}} - \riskC{\mathbf{h}^c}{\phi^\ast})+2\,d_{\mathcal{B}, \boldsymbol{\delta}}(\hat{\phi};\phi^\ast)\label{eq:ud_d},
\end{align}
where the third equality follows with Lemma~\ref{lemma:loss_decomp}. Using Condition~\labelcref{cond:rep_diff}, $d_{\mathcal{B}, \boldsymbol{\delta}}(\hat{\phi};\phi^\ast)\leq \gamma \,d_{\HC, \mathbf{h}^c}(\hat{\phi};\phi^\ast)$, and, for $\tilde{\mathbf{h}}^c = \argmin_{\mathbf{h}\in\HC^{\otimes 2}}\riskC{\mathbf{h}}{\hat{\phi}}$,
\begin{align}
    d_{\HC, \mathbf{h}^c}(\hat{\phi};\phi^\ast) &= \underbrace{\riskC{\tilde{\mathbf{h}}^c}{\hat{\phi}} - \riskC{\hat{\mathbf{h}^c}}{\hat{\phi}}}_{\leq 0} + \riskC{\hat{\mathbf{h}^c}}{\hat{\phi}} - \riskC{\mathbf{h}^c}{\phi^\ast}\\
    &\leq \riskC{\hat{\mathbf{h}}^c}{\hat{\phi}} - \riskC{\mathbf{h}^c}{\phi^\ast}
\end{align}
yields that $\labelcref{eq:ud_d}\leq(2+2\gamma)(\riskC{\hat{\mathbf{h}}^c}{\hat{\phi}} - \riskC{\mathbf{h}^c}{\phi^\ast})$.
Then, using Lemma~\ref{lemma:conf_upper_bnd}, we have that, with probability at least $1-p$,
\begin{align}
	\riskC{\hat{\mathbf{h}}^c}{\hat{\phi}} - \riskC{\mathbf{h}^c}{\phi^\ast} \leq 2048L\Bigg(\frac{D_{\mathcal{X}}}{({n^{\conf}})^2} + \log(n^{\conf})(L(\HC) &\GC{n^{\conf}}{\RC}+\WGC{n^{\conf}}{\HC})\Bigg)\notag\\
	&+ 8B\sqrt{\frac{2\log(1/p)}{n^{\conf}}}.
\end{align}
Finally, this yields, with probability at least $1-p$,
\begin{align}
&\epsilon_{\textup{PEHE}}(\hat{\tau}_{\cor}) \leq 2((2+2\gamma)2048L\left(\frac{D_{\mathcal{X}}}{({n^{\conf}})^2}+\log(n^{\conf}) (L(\HC)\GC{n^{\conf}}{\RC} + \WGC{n^{\conf}}{\HC})\right)\\
&+ 8B\sqrt{\frac{2\log(1/p)}{n^{\conf}}}
+ 8L\mathit{d}_\infty(p^{\conf}_\phi\mid p^{\unc}_\phi) \GC{n^{\unc}}{\mathcal{B}} + 4\mathit{d}_\infty(p^{\conf}_\phi\mid p^{\unc}_\phi)B\sqrt{\frac{2\log(1/p)}{n^{\unc}}})\\
&\leq \tilde{O}\bigg(L(\HC) \GC{n^{\conf}}{\RC} + \WGC {n^{\conf}} \HC + \mathit{d}_\infty(p^{\conf}_\phi\mid p^{\unc}_\phi)\,\GC {n^{\unc}} {\mathcal{B}}\\ 
&+ \frac{D_{\mathcal{X}}}{({n^{\conf}})^2} + \frac B L\left(\sqrt{\frac{2\log(1/p)}{n^{\conf}}} + \mathit{d}_\infty(p^{\conf}_\phi\mid p^{\unc}_\phi)\sqrt{\frac{2\log(1/p)}{n^{\unc}}}\right)\bigg).
\end{align}
\end{proof}

We have proven the statement now for general function classes. Since we use neural networks for representation functions and hypotheses in the statement in the main paper, we prove the statement now using neural networks as defined in \Cref{apx:ffnn}.\\

\textbf{Proof of \Cref{thm_bound}.}
For the proof, we build upon the error bound for general function classes as in \Cref{apx:thm_gen_bound}. We first prove upper bounds for the Gaussian complexities in \Cref{thm:upper_bnd_general}. Then, we prove that Assumption~A1 holds true for neural networks.

In order to find an upper boud for the Gaussian complexity of a neural network, we make use of Theorem 2 from \citet{golowich2018size}: 

Let $\mathcal{N}$ be the class of real-valued neural networks as defined in \Cref{apx:ffnn} of depth $K$ over a domain $\mathcal{X}$ with $\lVert x_i\rVert\leq D$ for $i\in\{1, \ldots,m\}$. Let $\lVert \mathbf{W}_k\rVert_{1, \infty}\leq \Omega_k$ for all $k \in \{1, \ldots, K\}$ and let the activation function $\sigma$ be a $1$-Lipschitz function with $\sigma(0)=0$, which is applied element-wise. Then, 
\begin{align}
    \hat{\mathcal{R}}_{m}(\mathcal{N})\leq \frac{2 D\sqrt{K+1+\log(d)}\,\Omega\prod_{k=1}^{K-1} \Omega_k}{\sqrt{m}}.
\end{align}
We proceed to upper bound the Gaussian complexities.

\textbf{Upper bound for $\hat{\mathcal{G}}_{X}(\RC)$:} First, we derive an upper bound on the Gaussian complexity of the representation class $\RC$:
\begin{align}
    \hat{\mathcal{G}}_{X^{\conf}}(\RC) =  \E_{\mathbf{g}}\left[\sup_{\phi\in\RC} \frac1{n^{\conf}}\sm k {d_\phi} \sm i {n^{\conf}} g_{ki}\phi_k(x_i)\right] \leq \sm k {d_\phi} \hat{\mathcal{G}}_{X^{\conf}}(\RC_k)\\
    \leq 2 \log(n^{\conf}) \sm k {d_\phi}\hat{\mathcal{R}}_{n^{\conf}}(\RC_k) \leq 2\log(n^{\conf}) \frac{\mathcal{C}_\Phi}{\sqrt{n^{\conf}}},
\end{align}
where $\mathcal{C}_\Phi = 2 {d_\phi} D\sqrt{K+1+\log(d)}\Omega\prod_{k=1}^{K-1} \Omega_k$. We use the upper bound of the Rademacher complexity of neural networks as above \citep{golowich2018size} and $\hat{\mathcal{G}}_{X^{\conf}}(\mathcal{F})\leq 2\sqrt{\log(N)}\hat{\mathcal{R}}_{n^{\conf}}(\mathcal{F})$ for an arbitrary function class $\mathcal{F}$ \citep{ledoux2013probability}. Taking the expectation is trivial, since the upper bound does not depend on the data. 

\textbf{Upper bound for $\WGC{n^{\conf}}{\HC}$:} We upper bound on the worst-case Gaussian complexity of $\HC$ follows similar to the above for $r=1$:
\begin{align}\label{apx:continue_cor}
    \WGC{n^{\conf}}{\HC} = \max_{\bar{Z}\in\mathcal{Z}} \hat{\mathcal{G}}_{\bar{Z}}(\HC) \leq 2 \log(n^{\conf}) \frac{\mathcal{C}_\HC}{\sqrt{n^{\conf}}},
\end{align}
where, similar to the above, $\mathcal{C}_\HC = 2\lVert \phi(x)\rVert \sqrt{D+1+\log(d_\phi)} \prod_{d=1}^{D} \Omega_d$, where $\lVert \phi(x)\rVert\leq \Omega$ with Lemma~\ref{apx:sup_bnd}.

\textbf{Upper bound for $\hat{\mathcal{G}}_{Z}(\mathcal{B})$:}
Finally, we upper bound the Gaussian complexity in the second step similar to the above for $d=1$ and $z_i^{\unc}=\hat{\phi}(x_i^{\unc})$:
\begin{align}
        \hat{\mathcal{G}}_{\bar{Z}^{\unc}}(\mathcal{B}) = \E_{\mathbf{g}}\left[\sup_{\delta\in\mathcal{B}} \frac1{n^{\unc}} \sm i {n^{\unc}} g_{i}\delta(z^{\unc}_i)\right] \leq 2\log(n^{\unc}) \hat{\mathcal{R}}_{n^{\unc}}(\mathcal{B}) \leq 2\log(n^{\unc}) \frac{\mathcal{C}_\mathcal{B}}{\sqrt{n^{\unc}}},
\end{align}
where, similar to the above, $\mathcal{C}_\mathcal{B} = 2\Omega \sqrt{B+1+\log(d_\phi)} \prod_{b=1}^{B} \Omega_b$.

Next, we verify that Assumption~{A1} for \Cref{thm:upper_bnd_general} holds true for neural networks. Using Lemma~\ref{apx:sup_bnd}, the boundedness parameter is
\begin{align}
    D_{\mathcal{X}} \leq \mathcal{O}(\Omega_D).
\end{align}
For the square loss, we have that $\nabla_a (a-y)^2 = 2(a-y) = \mathcal{O}(n+\lvert a\rvert)$ and $n=\mathcal{O}(1)$. Moreover, either $\lvert a\rvert \leq \lvert \textup{h}\circ\phi(x)\rvert\leq \Omega_D$ or $\lvert a\rvert \leq \lvert \delta\circ\phi(x)\rvert\leq \Omega_B$. Hence, the square loss is Lipschitz with $L=\mathcal{O}(\max(\Omega_D, \Omega_B))$. Furthermore, by the analogous argument, the square loss is bounded by $B=\mathcal{O}(\max(\Omega_D, \Omega_B)^2)$.

Moreover, by Lemma~\ref{apx:hypo_lipschitz}, the hypothesis class $\HC$ is Lipschitz with $L(\HC)\leq \mathcal{O}(\Omega_D)$.


Putting this together with the result from \Cref{thm:upper_bnd_general} yields
	\begin{align}
		\epsilon_{\textup{PEHE}}(\hat{\tau}_{\cor}) &\leq \tilde{O}\bigg(\Omega_D \frac{\mathcal{C}_\Phi}{\sqrt{n^{\conf}}} + \frac{\mathcal{C}_\HC}{\sqrt{n^{\conf}}} + \mathit{d}_\infty(p^{\conf}_\phi\mid p^{\unc}_\phi)\,\frac{\mathcal{C}_\mathcal{B}}{\sqrt{n^{\unc}}}\\ 
        &\quad+ \frac{\Omega_D}{({n^{\conf}})^2} + {\max(\Omega_D, \Omega_B)}\left(\sqrt{\frac{\log(1/p)}{n^{\conf}}} + \mathit{d}_\infty(p^{\conf}_\phi\mid p^{\unc}_\phi)\sqrt{\frac{\log(1/p)}{n^{\unc}}}\right)\bigg)\\
        &\leq \tilde{O}\bigg(\frac{\mathcal{C}_\Phi+\mathcal{C}_\HC}{\sqrt{n^{\conf}}}  + \frac{\mathit{d}_\infty(p^{\conf}_\phi\mid p^{\unc}_\phi)\,\mathcal{C}_\mathcal{B}}{\sqrt{n^{\unc}}}+ \frac{1}{({n^{\conf}})^2}\\
        &\quad+ \bigg(\sqrt{\frac{\log(1/p)}{n^{\conf}}} + \mathit{d}_\infty(p^{\conf}_\phi\mid p^{\unc}_\phi)\sqrt{\frac{\log(1/p)}{n^{\unc}}}\bigg)\bigg)\\
        &\leq\tilde{\mathcal{O}}
        \bigg(\frac{\mathcal{C}_{\RC} + 
        \mathcal{C}_{\HC}}{\sqrt{n^{\conf}}} + 
        \frac{\mathit{d}_\infty(p^{\conf}_\phi\mid p^{\unc}_\phi)\,\mathcal{C}_{\mathcal{B}}}{\sqrt{n^{\unc}}}\bigg),
	\end{align}
for $p$ large enough, which concludes the proof of \Cref{thm_bound}.\hfill$\blacksquare$
	
\begin{lemma}\label{apx:sup_bnd}
    Let $\phi$ and $\textup{h}$ be feedforward neural networks as defined in \Cref{apx:ffnn} with bounded data $\lVert x\rVert\leq D$. Then,
    \begin{align}
        D_{\mathcal{X}} \leq \prod_{d=1}^{D} \Omega_d \prod_{k=1}^{K-1} \Omega_k\,\Omega\, D.
    \end{align}
    Moreover, if $\sigma$ is the tanh activation function, \ie, $\sigma(x) = \frac{\mathit{e}^{x} - \mathit{e}^{-x}}{\mathit{e}^{x} + \mathit{e}^{-x}}$, which is centered and 1-Lipschitz, then we obtain
    \begin{align}
        &\lVert\phi(x)\rVert\leq \lVert \mathbf{W}_K\rVert_{\infty\rightarrow 2},\\
        &D_{\mathcal{X}} \leq \Omega_D,
    \end{align}
    which does not require to bound the input data $x$.
\end{lemma}	
\begin{proof}
We proceed iteratively and, for this purpose, let $\phi_{k}$ and $\textup{h}_{d}$ denote the vector-valued output of the $k^\text{th}$ and $d^\text{th}$ layer for $k\in[K]$ and $d\in[D]$ for the neural networks $\phi$ and $\textup{h}$, respectively. Then, 
\begin{align}
        D_{\mathcal{X}} \lesssim \sup_{\textup{h}\in \HC,\phi\in\RC,x}
    \lVert \textup{h}_t\circ\phi(x)\lVert^2,
\end{align}
and 
\begin{align}
    \lVert \textup{h}_t\circ\phi(x)\lVert^2 = \lVert \textup{h}_D\lVert^2 =\lVert \mathbf{W}_{D} \sigma(\textup{h}_{D-1})\rVert_2^2\leq \lVert \mathbf{W}_{D}\rVert_2^2\, \lVert\sigma(\textup{h}_{D-1})\rVert_2^2\\
    \leq \lVert\mathbf{W}_{D}\rVert_2^2\, \lVert \textup{h}_{D-1}\rVert_2^2
\end{align}
where the second inequality holds, since $\sigma$ is element-wise 1-Lipschitz and centered. Applying this argument recursively until $\textup{h}_0 = \phi(x)$, yields
\begin{align}
    \lVert \textup{h}_D\lVert^2 \leq \prod_{d=1}^{D}\lVert \mathbf{W}_d\rVert_2^2\, \lVert \phi(x) \rVert_2^2.
\end{align}
Then, with a similar argument for the neural network $\phi$, 
\begin{align}
    &\lVert \phi(x) \rVert_2^2 = \lVert \phi_{K}(x) \rVert_2^2 = \lVert \mathbf{W}_K\sigma(\phi_{K-1}(x)) \rVert_2^2\leq 
    \lVert \mathbf{W}_K\rVert_2^2\, \lVert\sigma(\phi_{K-1}(x)) \rVert_2^2 \leq \lVert \mathbf{W}_K\rVert_2^2\\
    &\lVert \mathbf{W}_{K-1}\phi_{K-2}(x) \rVert_2^2
    \leq \lVert \mathbf{W}_K\rVert_2^2\, \lVert \mathbf{W}_{K-1}\rVert_2^2\, \lVert\phi_{K-2}(x) \rVert_2^2,
\end{align}
again, using that $\sigma$ is element-wise 1-Lipschitz and centered.
Recursively applying this argument until $\phi_0 = x$ yields the conclusion after taking square roots and noting that $\lVert \mathbf{W}_d \rVert\leq \Omega_d$, $\lVert \mathbf{W}_k \rVert\leq \Omega_k$, and $\lVert x\rVert \leq D$.

We further consider $\sigma(\cdot)$ to be the tanh activation, which is 1-Lipschitz and centered, as in the statement. Then, by noting that $\lVert \phi_{K-1}\rVert_{1,\infty} \leq 1$, we obtain 
\begin{align}
        \lVert \phi(x)\rVert_2^2 = \lVert \phi_K\rVert_2^2 = \lVert \mathbf{W}_K \sigma(\phi_{K-1})\rVert_2^2 \leq \lVert \mathbf{W}_K \rVert_{\infty \rightarrow 2}^2\leq \Omega^2,
\end{align}
which, after taking square roots, concludes the proof.

For the last claim, we consider again that $\sigma(\cdot)$ to be the tanh activation, which is 1-Lipschitz and centered. Then, by noting that $\lVert \textup{h}_{D-1}\rVert_{1,\infty} \leq 1$, we obtain 
\begin{align}
    \lVert \textup{h}_t\circ\phi(x)\lVert^2 = \lVert \textup{h}_D\lVert^2\leq \lVert\mathbf{W}_{D}\rVert_2^2\, \lVert \textup{h}_{D-1}\rVert_2^2\leq \Omega_D^2.
\end{align}
\end{proof}

\begin{lemma}\label{apx:hypo_lipschitz}
     The function class $\HC$ as in \Cref{apx:ffnn} is $L(\HC)$-Lipschitz with 
     \begin{align}
        L(\HC)=\prod_{d=1}^{D}\Omega_d.
     \end{align}
     Moreover, if $\sigma$ is the tanh activation function, \ie, $\sigma(x) = \frac{\mathit{e}^{x} - \mathit{e}^{-x}}{\mathit{e}^{x} + \mathit{e}^{-x}}$, which is centered and 1-Lipschitz, then $L(\HC) = \alpha\Omega$.
\end{lemma}
\begin{proof}
\begin{align}
        &\lVert \textup{h}_t(z) - \textup{h}_t(z^\prime)\rVert_2^2 = \lVert \mathbf{W}_D\sigma(\textup{h}_{D-1}(z)) - \mathbf{W}_D\sigma(\textup{h}_{D-1}(z^\prime))\rVert_2^2\\
        &\leq \lVert \mathbf{W}_D\rVert_2^2\, \lVert \sigma(\textup{h}_{D-1}(z)) - \sigma(\textup{h}_{D-1}(z^\prime))\rVert_2^2 \leq \lVert \mathbf{W}_D\rVert_2^2\, \lVert \textup{h}_{D-1}(z) - \textup{h}_{D-1}(z^\prime)\rVert_2^2,
\end{align}
where the last inequality holds, since $\sigma(\cdot)$ is 1-Lipschitz. Recursively applying the argument yields
\begin{align}
        &\lVert \textup{h}_t(z) - \textup{h}_t(z^\prime)\rVert_2^2 \leq \prod_{d=1}^{D}\lVert \mathbf{W}_{d}\rVert_2^2\, \lVert z - z^\prime\rVert_2^2,
\end{align}
which yields the conclusion after taking square roots and noting that $\lVert \mathbf{W}_{d}\rVert_2\leq \Omega_d$ for $d\in\{1, \ldots, D\}$. 

If we further consider $\sigma$ to be the tanh activation as in the statement, then, by an analogous argument as in the proof of Lemma~\ref{apx:sup_bnd}, it yields that $\textup{h}$ is Lipschitz with $L(\HC)=\Omega_D$.
\end{proof}

\subsection{Proof of \texorpdfstring{\Cref{lemma:unconfounded_gen_bnd}}{}}\label{apx:unconfounded_gen_bnd}
The proof relies on similar procedures as the proof of \Cref{thm_bound}. In particular, we again use a modification of the bounded differences inequality and a standard symmetrization argument. For ease of notation, we use $l_{\phi}(\mathbf{h},x,y,t) = (\textup{h}_t\circ\phi(x) - y)^2$. Then, similar to the proof of \Cref{thm_bound}, we obtain
\begin{align}
    \epsilon_{\textup{PEHE}}(\hat{\tau}_{\unc})\leq 2(\epsilon(\hat{\mathbf{h}}^u, \hat{\phi}) - \epsilon(\mathbf{h}^u, \phi^\ast)),
\end{align}
where for any $\mathbf{h}\in\HC^{\otimes 2}$ and $\phi\in\RC$,
\begin{align}
    &\epsilon(\mathbf{h}, \phi) = \sum_{t=0}^{1}\E[(\textup{h}_t\circ\phi(X^{\conf}) - Y(t))^2]\\ 
    &= \sum_{t=0}^{1}\mathbb{P}(T^{\unc}=t)\E[\E[(\textup{h}_t\circ\phi(X^{\conf}) - Y^{\unc})^2\mid T^{\unc}=t, X^{\conf}]]\\
    &=\E[\E[(\textup{h}_{T^{\unc}}\circ\phi(X^{\conf}) - Y^{\unc})^2\mid T^{\unc}, X^{\conf}]]\\&=\E[(\textup{h}_{T^{\unc}}\circ\phi(X^{\conf}) - Y^{\unc})^2]\\
    &=\E[l_{\phi}(\mathbf{h}, X^{\conf}, Y^{\unc}, T^{\unc})].
\end{align}
 Let $\hat{\mathbf{h}}^u, \hat{\phi} = \argmin_{\mathbf{h}\in\HC^{\otimes 2}, \phi\in\Phi}\emriskUC{\mathbf{h}}{\phi}$. Then, with probability at least $1-p$,
\begin{align}
    &\risk{\hat{\mathbf{h}}^u}{\hat{\phi}} - \risk{\mathbf{h}^u}{\phi^\ast}\\ &= \risk{\hat{\mathbf{h}}^u}{\hat{\phi}} - \emriskUC{\hat{\mathbf{h}}^u}{\hat{\phi}} + \underbrace{\emriskUC{\hat{\mathbf{h}}^u}{\hat{\phi}} -\emriskUC{\mathbf{h}^u}{\phi^\ast}}_{\leq 0} +\emriskUC{\mathbf{h}^u}{\phi^\ast} - \risk{\mathbf{h}^\ast}{\phi^\ast}\\
    &\leq 2\sup_{\mathbf{h}\in\HC^{\otimes 2}, \phi\in\RC}\left\lvert \risk{\mathbf{h}}{\phi} - \emriskUC{\mathbf{h}}{\phi} \right\rvert\label{eq:sup_RC}\\
    &\leq 8L\mathit{\bar{d}}_\infty(p^{\conf}_\phi\mid p^{\unc}_\phi) \GC{n^{\unc}}{\HC(\RC)} + 4\mathit{\bar{d}}_\infty(p^{\conf}_\phi\mid p^{\unc}_\phi)B\sqrt{\frac{2\log(1/p)}{n^{\unc}}},
\end{align}
where the first inequality follows, since the middle term is negative due to $\hat{\mathbf{h}}^u, \hat{\phi} = \argmin_{\mathbf{h}\in\HC^{\otimes 2}, \phi\in\Phi}\emriskUC{\mathbf{h}}{\phi}$. The second inequality follows be an identical argument as in the proof of Lemma~\ref{lemma:unc_upper_bnd}. Moreover,  $\mathit{\bar{d}}_\infty(p^{\conf}_\phi\mid p^{\unc}_\phi) = \sup_{\phi\in\RC}\mathit{d}_\infty(p^{\conf}_\phi\mid p^{\unc}_\phi)$. The term $\mathit{\bar{d}}_\infty(p^{\conf}_\phi\mid p^{\unc}_\phi)$ arises, since, in order to pull it out, we have to take the supremum over the function class $\RC$ in \labelcref{eq:sup_RC}.
Again, using the chain rule for Gaussian complexities, similar to the proof of Lemma~\ref{lemma:conf_upper_bnd}, yields
\begin{align}
    \hat{\mathcal{G}}_X(\HC(\RC))\leq 128\left(\frac{2D_{\mathcal{X}}}{{n^{\unc}}^2}+C(\HC(\RC))\log(n^{\unc})\right),
\end{align}
where $C(\HC(\RC)) = L(\HC)\hat{\mathcal{G}}_X(\RC) + \max_{Z\in\mathcal{Z}}\hat{\mathcal{G}}_Z(\HC)$. The remainder of the proof is identical to the proof of \Cref{thm:upper_bnd_general} for general function classes and the proof of \Cref{thm_bound} for the complexities of neural networks.\hfill$\blacksquare$

\subsection{Proof of \texorpdfstring{\Cref{thm_gen_bound_conf}}{}}\label{apx:thm_gen_bound_conf}

First, using Lemma~\labelcref{lemma:loss_decomp}, it follows that
\begin{align}
    &\epsilon_{\textup{PEHE}}(\hat{\tau}_{\conf}) = \E[(\hat{\tau}_{\conf}(X^{\conf})-\tau(X^{\conf}))^2]\\
    &=\E[(\hat{\textup{h}}_1^c\circ\hat{\phi}(X^{\conf}) - \hat{\textup{h}}_0^c\circ\hat{\phi}(X^{\conf}) - (\textup{h}^c_1+\delta_1)\circ\phi^\ast(X^{\conf})\\
    &\quad+ (\textup{h}^c_0+\delta_0)\circ\phi^\ast(X^{\conf})))^2]\\
    &=\E[(\hat{\textup{h}}_1^c\circ\hat{\phi}(X^{\conf})- \textup{h}^c_1\circ\phi(X^{\conf}) +\textup{h}^c_0\circ\phi^\ast(X^{\conf})- \hat{\textup{h}}_0^c\circ\hat{\phi}(X^{\conf})\\
    &\quad+ (\delta_0-\delta_1)\circ\phi^\ast(X^{\conf}))^2]\\
    &\leq 2\E_{X^{\conf}}[(\hat{\textup{h}}_1^c\circ\hat{\phi}(X^{\conf})- \textup{h}^c_1\circ\phi^\ast(X^{\conf}) + \textup{h}^c_0\circ\phi^\ast(X^{\conf})- \hat{\textup{h}}_0^c\circ\hat{\phi}(X^{\conf}))^2\\
    &\quad+ ((\delta_0-\delta_1)\circ\phi^\ast(X^{\conf}))^2]\\
    &= 4\sum_{t=0}^{1}\E[(\hat{\textup{h}}_t^c\circ\hat{\phi}(X^{\conf})- \textup{h}^c_t\circ\phi^\ast(X^{\conf}))^2] + 2\E[((\delta_0-\delta_1)\circ\phi^\ast(X^{\conf}))^2]\\
    &\leq 4\sum_{t=0}^{1}\E[(\hat{\textup{h}}_t^c\circ\hat{\phi}(X^{\conf})-Y^{\conf})^2 - (Y^{\conf} -\textup{h}^c_t\circ\phi^\ast(X^{\conf}))^2\mid T=t]\\
    &\quad+ 2\E[((\delta_0-\delta_1)\circ\phi^\ast(X^{\conf}))^2]\\
    &=4(\epsilon_{\conf}(\hat{\mathbf{h}}^c, \hat{\phi})-\epsilon_{\conf}(\mathbf{h}^c, \phi^\ast)) + 2\Delta,
\end{align}
where $\Delta=\E\left[((\delta_0-\delta_1)\circ\phi^\ast(X^{\conf}))^2\right]$.
Then, we can use the same upper bound for $\epsilon_{\conf}(\hat{\mathbf{h}}^c, \hat{\phi}) - \epsilon_{\conf}(\mathbf{h}^c, \phi^\ast)$ as in the proof of \Cref{thm_bound}. This yields the claim.\hfill$\blacksquare$

\subsection{Proof of \texorpdfstring{\Cref{thm_gen_bound_avg}}{}}\label{apx:thm_gen_bound_avg}
The proof of \Cref{thm_gen_bound_avg} is a immediate implication of \Cref{lemma:unconfounded_gen_bnd} and \Cref{thm_gen_bound_conf}. This can be seen as follows:
\begin{align}
   &\epsilon_{\textup{PEHE}}(\hat{\tau}_{\avg}(\lambda)) =  \E[(\hat{\tau}_{\avg}(\lambda)(X^{\conf})-\tau(X^{\conf}))^2]\\
   &= \E[((1-\lambda)\hat{\tau}_{\unc}(X^{\conf}) + \lambda\hat{\tau}_{\conf}(X^{\conf})-\tau(X^{\conf}))^2]\\
   &=\E[((1-\lambda)\hat{\tau}_{\unc}(X^{\conf}) + \lambda\hat{\tau}_{\conf}(X^{\conf})-(\lambda\tau(X^{\conf})+(1-\lambda)\tau(X^{\conf})))^2]\\
   &=2(1-\lambda)\E[(\hat{\tau}_{\unc}(X^{\conf}) - \tau(X^{\conf}))^2] + 2\lambda\E[(\hat{\tau}_{\conf}(X^{\conf})-\tau(X^{\conf}))^2]\\
   &=2(1-\lambda)\epsilon_{\textup{PEHE}}(\hat{\tau}_{\unc}) + 2\lambda\epsilon_{\textup{PEHE}}(\hat{\tau}_{\conf}).
\end{align}
Hence, using the bounds for $\epsilon_{\textup{PEHE}}(\hat{\tau}_{\unc})$ and $\epsilon_{\textup{PEHE}}(\hat{\tau}_{\conf})$ yields the claim.\hfill$\blacksquare$

\subsection{Proof of \texorpdfstring{\Cref{thm_gen_bound_weight}}{}}\label{apx:thm_gen_bound_weight}
First, we define $\lambda = \frac{n^{\conf}}{\lambda n^{\unc}+n^{\conf}}$. This implies that $1-\lambda = \frac{\Lambda n^{\unc}}{\lambda n^{\unc}+n^{\conf}}$. Then, the weighted empirical risk is given by
\begin{align}
    \hat{\epsilon}_{\weight}(\mathbf{h}, \phi) = (1-\lambda) &\frac{1}{n^{\unc}}\sm i {n^{\unc}} (\textup{h}_{t_i^{\unc}}\circ\phi(x_i^{\unc}) - y_i^{\unc})^2\\
    &+\lambda\frac{1}{n^{\conf}}\sm i {n^{\conf}} (\textup{h}_{t_i^{\conf}}\circ\phi(x_i^{\conf}) - y_i^{\conf})^2.
\end{align}
Further, we can write
\begin{align}
    &\epsilon_{\textup{PEHE}}(\hat{\tau}_{\weight}) = \E[(\hat{\tau}_{\weight}(\lambda)(X^{\conf})-\tau(X^{\conf}))^2]\\
    &=(1-\lambda)\underbrace{\E[(\hat{\tau}_{\weight}(\lambda)(X^{\conf})-\tau(X^{\conf}))^2]}_{=(\textup{I})} + \lambda\underbrace{\E[(\hat{\tau}_{\weight}(\lambda)(X^{\conf})-\tau(X^{\conf}))^2]}_{=(\textup{II})}.
\end{align}
Then, using the same the same argument as in the proof of \Cref{lemma:unconfounded_gen_bnd} and \Cref{thm_gen_bound_conf}, yields the following upper bounds,
\begin{align}
    &(\textup{I})\leq 2(\epsilon(\hat{\mathbf{h}}, \hat{\phi}) - \epsilon(\mathbf{h}^u, \phi^\ast)),\\
    &(\textup{II})\leq 4(\epsilon_{\conf}(\hat{\mathbf{h}}, \hat{\phi})-\epsilon_{\conf}(\mathbf{h}^c, \phi^\ast)) + 2\Delta
\end{align}
where $\hat{\mathbf{h}}$ and $\hat{\phi}$ are the estimators from \labelcref{eq:weight_minimizer}. Then, we claim follows by using the results from \Cref{lemma:unconfounded_gen_bnd} and \Cref{thm_gen_bound_conf}.\hfill$\blacksquare$
\section{Impact of Selection Bias in Observational Data} \label{apx:selection_bias}
In the main paper, we have considered the case in which there is no selection bias in the observational data. Selection bias is present if the covariate distributions $p^{\conf}_{t=1}(x) = p^{\conf}(x\mid t=1)$ and $p^{\conf}_{t=0}(x) = p^{\conf}(x\mid t=0)$ differ. This can be due to covariate-dependent treatment assignment. There is a large body of research addressing the selection bias in observational data \citep[\eg,][]{Shalit2017a, johansson2018learning, Yao2018a, zhang2020}, which is orthogonal to our work. However, since selection bias may occur in practice, we discuss this case briefly here. In particular, we show how \Cref{thm_bound} changes in the presence of selection bias. Based on this, algorithmic designs can be easily adjusted in order to counteract selection bias.

In presence of selection bias, the error bound of $\tau_{\cor}$ is given by the following result.
\begin{theorem}\label{prop_gen_bound_selection_bias}
		Let $(\hat{\mathbf{h}}^c, \hat{\phi})$ be the empirical loss minimizer of $\hat{\epsilon}_{\conf}(\cdot, \cdot)$ from \labelcref{eq:conf_loss_minimizer} over the function classes $\RC$ and $\HC$, and let  $\boldsymbol{\delta}$ be the empirical loss minimizer of $\hat{\epsilon}_{\unc}(\cdot, \hat{\mathbf{h}}^c, \hat{\phi})$ from \labelcref{eq:unc_loss_minimizer} over the function class $\mathcal{B}$. Further, let $\hat{\tau}_{\cor}$ be the resulting CATE estimator from \labelcref{eq:tau_hat}. Then, if Assumption~\labelcref{assum:realizability} and Condition~\labelcref{cond:rep_diff} hold true, we have that, with probability at least $1-p$
\begin{align}
		\epsilon_{\textup{PEHE}}(\hat{\tau}_{\cor})
        \leq \tilde{\mathcal{O}}\bigg(\mathit{\bar{d}}_\infty(p^{\conf}_{\phi,t=1}\mid p^{\conf}_{ \phi,t=0})\frac{\mathcal{C}_{\RC}+\mathcal{C}_{\HC}}{\sqrt{n^{\conf}}} + \frac{\mathit{d}_\infty(p^{\conf}_\phi\mid p^{\unc}_\phi)\,\mathcal{C}_\mathcal{B}}{\sqrt{n^{\unc}}}\bigg),
    \end{align}
	where $\mathcal{C}_{\RC}$, $\mathcal{C}_{\HC}$, and $\mathcal{C}_{\mathcal{B}}$ are constants depending on the complexity of the neural networks. Moreover, $p^{\conf}_{\phi, t}$, $p^{\conf}_\phi$, and $p^{\unc}_\phi$ denote the push-forwards through $\phi$ of the corresponding covariate distributions.
\end{theorem}
\begin{proof}
The proof of \Cref{prop_gen_bound_selection_bias} differs from the proof of \Cref{thm_bound} only in Lemma~\labelcref{lemma:conf_upper_bnd}, which is the upper bound on $d_{\HC, \mathbf{h}^c}(\hat{\phi};\phi^\ast)$. We only need to prove that, in presence of selection bias, the upper bound for $d_{\HC, \mathbf{h}^c}(\hat{\phi};\phi^\ast)$ is given by:
\begin{align}\label{eq:modi_conf_upper_bnd}
	d_{\HC, \mathbf{h}^c}(\hat{\phi};\phi^\ast) \leq &4096L\left(\frac{D_{\mathcal{X}}}{({n^{\conf}})^2} + \mathit{\bar{d}}_\infty(p^{\conf}_{\phi,t=1}\mid p^{\conf}_{\phi,t=0})\log(n^{\conf})(L(\HC) \frac{\GC{n^{\conf}}{\RC}}{n^{\conf}}+\frac{\WGC{n^{\conf}}{\HC}}{n^{\conf}})\right)\\
	&+ 8B\sqrt{\frac{\log(1/p)}{n^{\conf}}},
\end{align}
where $\mathit{\bar{d}}_\infty(p^{\conf}_{\phi,t=1}\mid p^{\conf}_{\phi,t=0}) = \sup_{\phi\in\RC}\mathit{d}_\infty(p^{\conf}_{\phi,t=1}\mid p^{\conf}_{\phi,t=0})$ and $p^{\conf}_{\phi,t=1}(z)$ is the push-forward of $p^{\conf}(x\mid t=1)$ through $\phi$ and similar for $p^{\conf}_{\phi,t=0}$. For this, we assume w.l.o.g. that $\sup_z p^{\conf}_{\phi,t=1}(z) > \sup_z p^{\conf}_{\phi,t=1}(z)$. Otherwise, the term $\mathit{\bar{d}}_\infty(p^{\conf}_{\phi,t=1}\mid p^{\conf}_{\phi,t=0})$ in the statement changes to $\mathit{\bar{d}}_\infty(p^{\conf}_{\phi,t=0}\mid p^{\conf}_{\phi,t=1})$. This is due to the fact that the Renyi divergence is asymmetric.

Similar to the proof of Lemma~\labelcref{lemma:conf_upper_bnd}, 
\begin{align}
    d_{\HC, \mathbf{h}^c}(\hat{\phi};\phi^\ast)\leq 2\sup_{\mathbf{h}\in\HC^{\otimes 2}, \phi\in\RC}\left\lvert \riskC{\mathbf{h}}{\phi} - \emriskC{\mathbf{h}}{\phi} \right\rvert.
\end{align}
We intend to proceed similar to the proof of Lemma~\labelcref{lemma:unc_upper_bnd}, where we use a change of probability measure and change of variable (to push-forward the covariate distribution into the representation space). For this, we again use the Radon-Nikodym derivative $w_\phi(z)$ for the change of probability measure. However, here, we face two different covariate distributions that we need to change. That is, we need to change from $p^{\conf}_{\phi,t=1}$ to $p^{\conf}_\phi$ and from $p^{\conf}_{\phi,t=0}$ to $p^{\conf}_\phi$. Hence, we use the Radon-Nikodym derivative $w_\phi(z) = \frac{T\cdot p^{\conf}_{\phi,t=1}(z) + (1-T)\cdot p^{\conf}_{\phi,t=0}}{p^{\conf}_{\phi}(z)}$. Further, note that $p^{\conf}_{\phi}(z)=q\cdot p^{\conf}_{\phi,t=1}(z) + (1-q)\cdot p^{\conf}_{\phi,t=0}(z)$, where $q = \mathbb{P}(T=1)$. Then, 
\begin{align}
    w_\phi(z) = \frac{T\cdot p^{\conf}_{\phi,t=1}(z) + (1-T)\cdot p^{\conf}_{\phi,t=0}}{q\cdot p^{\conf}_{\phi,t=1}(z) + (1-q)\cdot p^{\conf}_{\phi,t=0}(z)},
\end{align}
and, hence, since $T\in\{0,1\}$ and $\sup_z p^{\conf}_{\phi,t=1}(z) > \sup_z p^{\conf}_{\phi,t=1}(z)$, this yields $\sup_z w_\phi(z) < \sup_z \frac{p^{\conf}_{\phi,t=1}(z)}{p^{\conf}_{\phi,t=0}(z)}$.
Then, similar to the proof of Lemma~\labelcref{lemma:unc_upper_bnd},
\begin{align}
     &2\sup_{\mathbf{h}\in\HC^{\otimes 2}, \phi\in\RC}\left\lvert \riskC{\mathbf{h}}{\phi} - \emriskC{\mathbf{h}}{\phi} \right\rvert\\
     &\leq 4L\sup_{\phi\in\RC}\mathit{d}_\infty(p^{\conf}_{\phi,t=1}\mid p^{\conf}_{\phi,t=0}) \mathcal{R}_{n^{\unc}}(\HC(\RC)) + 4\sup_{\phi\in\RC}\mathit{d}_\infty(p^{\conf}_{\phi,t=1}\mid p^{\conf}_{\phi,t=0})B\sqrt{\frac{2\log(1/\delta)}{{n^{\unc}}}}\\
     &= 4L\mathit{\bar{d}}_\infty(p^{\conf}_{\phi,t=1}\mid p^{\conf}_{\phi,t=0}) \mathcal{R}_{n^{\unc}}(\HC(\RC)) + 4\mathit{\bar{d}}_\infty(p^{\conf}_{\phi,t=1}\mid p^{\conf}_{\phi,t=0})B\sqrt{\frac{2\log(1/\delta)}{{n^{\unc}}}}.
\end{align}
We then proceed similar to the proof of Lemma~\labelcref{lemma:conf_upper_bnd} and use the chain rule of Gaussian complexities and upper bound the Gaussian complexities for the neural networks as in \Cref{apx:thm_gen_bound}.
\end{proof}

From this result, we can see that, besides the impact of the distributional discrepancy between $\mathit{d}_\infty(p^{\conf}_\phi\mid p^{\unc}_\phi)$, another factor impacts the error as well: the distributional discrepancy between the covariate distributions of the treatment and control group, $\mathit{d}_\infty(p^{\conf}_{\phi, t=1}\mid p^{\conf}_{\phi, t=0})$. The impact of this factor is similar to the impact of the distributional discrepancy between $\mathit{d}_\infty(p^{\conf}_\phi\mid p^{\unc}_\phi)$. Namely, it does not introduce bias, but slows down the convergence of the error. The most common approach to address the discrepancy in the representation space due to selection bias is balancing, which was discussed in many works on treatment effect estimation using observational data \citep[\eg,][]{Johansson2016,Shalit2017a, zhang2020} and domain adaptation \citep{ganin2015unsupervised,ganin2016domain}. Further, we can also see from the above result that, in the case in which $n^{\conf}$ is large, the constant factor due to selection bias, $\mathit{d}_\infty(p^{\conf}_{\phi, t=1}\mid p^{\conf}_{\phi, t=0})$, may not play an important role. Hence, in observational data, it may not be necessary to address the selection bias. Nevertheless, our procedure can be easily adjusted for this case if required.

\subsection{Proof of Proposition~\texorpdfstring{\labelcref{prop:condition}}{}}\label{apx:condition}
The proof of Proposition~\labelcref{prop:condition} follows by straightforward algebraic manipulation. For Condition~\labelcref{cond:rep_diff}, which compares $\tau_{\cor}$ and $\tau_{\unc}$, the following yields the claim:
\begin{align}
    &\frac{\mathcal{C}_{\RC} + \mathcal{C}_{\HC}}{\sqrt{n^{\conf}}} + \frac{\mathit{d}_\infty(p^{\conf}_\phi\mid p^{\unc}_\phi)\mathcal{C}_{\mathcal{B}}}{\sqrt{n^{\unc}}} <    \bar{\mathit{d}}_\infty(p^{\conf}_\phi\mid p^{\unc}_\phi)\frac{\mathcal{C}_{\RC} + \mathcal{C}_{\HC}}{\sqrt{n^{\unc}}}\\
    &\frac{\mathcal{C}_{\RC} + \mathcal{C}_{\HC}}{\sqrt{n^{\conf}}}  < \frac{\bar{\mathit{d}}_\infty(p^{\conf}_\phi\mid p^{\unc}_\phi)(\mathcal{C}_{\RC} + \mathcal{C}_{\HC})- \mathit{d}_\infty(p^{\conf}_\phi\mid p^{\unc}_\phi)\mathcal{C}_{\mathcal{B}}}{\sqrt{n^{\unc}}}\\
    &\frac{\mathcal{C}_{\RC} + \mathcal{C}_{\HC}}{\mathit{\bar{d}}_\infty(p^{\conf}_\phi\mid p^{\unc}_\phi)(\mathcal{C}_{\RC} + \mathcal{C}_\mathcal{\HC}) -\mathit{d}_\infty(p^{\conf}_\phi\mid p^{\unc}_\phi) \mathcal{C}_\mathcal{B}} < \sqrt{\frac{n^{\conf}}{n^{\unc}}}.
\end{align}
For Condition~2, which compares $\tau_{\cor}$ and $\tau_{\conf}$, similar to the above, the following yields the claim:
\begin{align}
    &\frac{\mathcal{C}_{\RC} + \mathcal{C}_{\HC}}{\sqrt{n^{\conf}}} + \frac{\mathit{d}_\infty(p^{\conf}_\phi\mid p^{\unc}_\phi)\mathcal{C}_{\mathcal{B}}}{\sqrt{n^{\unc}}} <     \frac{\mathcal{C}_{\RC} + \mathcal{C}_{\HC}}{\sqrt{n^{\conf}}} +2\Delta\\
    &\frac{\mathit{d}_\infty(p^{\conf}_\phi\mid p^{\unc}_\phi)\mathcal{C}_{\mathcal{B}}}{\sqrt{n^{\unc}}} < 2\Delta.
\end{align}

\section{Theoretical Properties for CorNet}\label{apx:theoretical_properties_algo}
In this section, we give variants of the results in \Cref{thm_bound} for our algorithm CorNet. 

In particular, $\HC$ and $\mathcal{B}$ are given by 
\begin{align}
    &\mathcal{H} = \{\textup{h}\mid\, \textup{h}(z, t) = \mathbf{w}_t^\top\,z,\, \mathbf{w}_t\in\Rl^{d_\phi},\, \lVert \mathbf{w}_t \rVert_2\leq \alpha\},\label{eq:linear_predictor_hypo}\\
    &\mathcal{B} = \{\delta\mid\, \delta(z, t) = \boldsymbol{\delta}_t^\top\,z,\, \boldsymbol{\delta}_t\in\Rl^{d_\phi},\, \lVert \boldsymbol{\delta}_t \rVert_1\leq \beta\},\label{eq:linear_predictor_bias}
\end{align}
where $d_\phi$ denotes the dimension of the representation space.

\subsection{Finite Sample Learning Bound}
The following error bound holds.
\begin{corollary}\label{cor_gen_bound}
	Let $(\hat{\mathbf{h}}^c, \hat{\phi})$ be the empirical loss minimizer of $\hat{\epsilon}_{\conf}(\cdot, \cdot)$ from \labelcref{eq:conf_loss_minimizer} over the function classes $\RC$ as in \labelcref{eq:nn_rep} and $\HC$ as in \labelcref{eq:linear_predictor_hypo}, and let  $\hat{\boldsymbol{\delta}}$ be the empirical loss minimizer of $\hat{\epsilon}_{\unc}(\cdot, \hat{\mathbf{h}}^c, \hat{\phi})$ from \labelcref{eq:unc_loss_minimizer} over the function class $\mathcal{B}$ as in \labelcref{eq:linear_predictor_bias}. Further, let $\hat{\tau}_{\cor}$ be the resulting CATE estimator from \labelcref{eq:tau_hat}. Then, if Assumption~\labelcref{assum:realizability} holds true, we have that, with probability at least $1-p$,
\begin{align}
		\epsilon_{\textup{PEHE}}(\hat{\tau}_{\cor})\leq
		\tilde{\mathcal{O}}\bigg(\frac{\mathcal{C}_{\Phi} + \alpha}{\sqrt{n^{\conf}}} + \frac{\mathit{d}_\infty(p^{\conf}_{\hat{\phi}}\mid p^{\unc}_{\hat{\phi}})\beta\sqrt{2\log d_\phi}}{\sqrt{n^{\unc}}}\bigg)
	\end{align}
\end{corollary}
\begin{proof}
We continue the proof of \Cref{apx:thm_gen_bound} in \Cref{apx:continue_cor}. The term for the representation function remains the same as we still consider neural networks for $\RC$. We continue with bounded the Rademacher complexity of $\mathcal{H}$ and $\mathcal{B}$.

Then, using that $z=\phi(x)$ and $\lVert\phi(x)\rVert\leq \Omega$ (by Lemma~\labelcref{apx:sup_bnd}),
\begin{align}
    &\hat{\mathcal{G}}_{Z}(\mathcal{H}) \leq 2 \log(n^{\conf}) \hat{\mathcal{R}}_{Z}(\mathcal{H})\leq 2 \log(n^{\conf})\frac{\alpha}{\sqrt{n^{\conf}}}\sqrt{\E[\sm i {n^{\conf}}\lVert z_i\rVert_2^2]}\\
    &\leq 2 \log(n^{\conf})\frac{\alpha}{\sqrt{n^{\conf}}}\Omega\\
    &\hat{\mathcal{G}}_{Z}(\mathcal{B}) \leq 2 \log(n^{\unc}) \hat{\mathcal{R}}_{Z}(\mathcal{B})\leq 2 \log(n^{\unc}) \frac{\beta \sqrt{2\log d_\phi}}{\sqrt{n^{\unc}}}\sup_j\sqrt{\sm i {n^{\unc}} (z_i)^2_j}\\
    &\leq 2 \log(n^{\unc})\frac{\beta \sqrt{2\log d_\phi}}{\sqrt{n^{\unc}}}\Omega,
\end{align}
where we use bounds on the Rademacher complexity for linear functions (see for instance \citet{kakade2008complexity}). This yields
\begin{align}
    \WGC{n^{\conf}}{\mathcal{H}} = \max_{Z\in\mathcal{Z}} \hat{\mathcal{G}}_{Z}(\mathcal{H}) \leq 2 \log(n^{\conf})\frac{\alpha}{\sqrt{n^{\conf}}}\Omega,\\
    \hat{\mathcal{G}}_{Z}(\mathcal{B}) \leq 2 \log(n^{\unc})\frac{\beta \sqrt{2\log d_\phi}}{\sqrt{n^{\unc}}}\Omega,
\end{align}
which concludes the proof.
\end{proof}

\subsection{Improvement Condition}
In this section, we derive conditions for when combining observational data and randomized data is beneficial for our algorithm proposed in \Cref{alg:CORNet}. In particular, we compare the error bounds our estimator against the baseline estimators which only use randomized or observational data. Based on this, we derive conditions for when our algorithm improves upon the baselines estimators and when we should rely on one of the baseline estimators.

First, we derive finite sample bounds for the estimator which only uses randomized data and, then, for the estimator which only uses observational data, similarly to \Cref{lemma:unconfounded_gen_bnd} and \Cref{thm_gen_bound_conf}.

\subsubsection{Estimator on randomized data}
Similarly to the results in \Cref{lemma:unconfounded_gen_bnd}, we derive the following finite sample error bound for $\tau_{\unc}$ in \labelcref{sec::tau_unc}.

\begin{theorem}(Estimation on randomized data)\label{lemma:unconfounded_gen_bnd_algo}
Let $(\hat{\mathbf{h}}, \hat{\phi})$ be the empirical loss minimizer from \labelcref{eq:naive_unc_minimizer} over the function classes $\RC$ as in \labelcref{eq:nn_rep} and $\HC$ as in \labelcref{eq:linear_predictor_hypo} and $\hat{\tau}_{\unc}$ from \labelcref{eq:tau_hat_naive_unc}. Then, we have that with probability at least $1-p$,
    \begin{align}
        \epsilon_{\textup{PEHE}}(\hat{\tau}_{\unc})\leq\tilde{\mathcal{O}}\bigg(\mathit{\bar{d}}_\infty(p^{\conf}_\phi\mid p^{\unc}_\phi)\frac{\mathcal{C}_{\RC} + \alpha}{\sqrt{n^{\unc}}}\bigg),
    \end{align}
    where $\mathcal{C}_{\RC}$ is a constant depending on the complexity of the neural networks and $\mathit{\bar{d}}_\infty(p^{\conf}_\phi\mid p^{\unc}_\phi) = \sup_{\phi\in\RC}\mathit{d}_\infty(p^{\conf}_\phi\mid p^{\unc}_\phi)$.
\end{theorem}
\begin{proof}
    Similar to the proof of \Cref{lemma:unconfounded_gen_bnd}, but using that $\WGC{n^{\conf}}{\mathcal{H}} = \max_{\bar{Z}\in\mathcal{Z}} \hat{\mathcal{G}}_{\bar{Z}}(\mathcal{H}) \leq 2 \log(n^{\conf})\frac{\alpha}{\sqrt{n^{\conf}}}\Omega$.
\end{proof}

\subsubsection{Estimator on Observational Data}
Similarly to the results in Lemma~\labelcref{lemma:conf_upper_bnd}, we derive the following finite sample error bound for $\tau_{\conf}$ in \labelcref{sec::tau_conf}.

\begin{theorem}\label{thm_gen_bound_conf_algo}
Let $(\hat{\mathbf{h}}^c, \hat{\phi})$ be the empirical loss minimizer over the function classes $\RC$ as in \labelcref{eq:nn_rep} and $\HC$ as in \labelcref{eq:linear_predictor_hypo} from \labelcref{eq:conf_loss_minimizer} and $\hat{\tau}_{\conf}$ as in \labelcref{eq:tau_hat_conf}. Then, we have that, with probability at least $1-p$,
\begin{align}
		\epsilon_{\textup{PEHE}}(\hat{\tau}_{\conf})
        \leq \tilde{\mathcal{O}}\bigg(\frac{\mathcal{C}_{\RC} + \alpha}{\sqrt{n^{\conf}}}\bigg)+2\Delta,
	\end{align}
	where $\Delta = \E[((\boldsymbol{\delta}_1-\boldsymbol{\delta}_0)\circ\phi(X^{\conf}))^2]$ is the bias due to unobserved confounding and $\mathcal{C}_{\RC}$ is the same as in \Cref{thm_bound}.
\end{theorem}
\begin{proof}
	    Similar to the proof of Lemma~\labelcref{lemma:conf_upper_bnd}, but, again, using that $\WGC{n^{\conf}}{\mathcal{H}} = \max_{Z\in\mathcal{Z}} \hat{\mathcal{G}}_{Z}(\mathcal{H}) \leq 2 \log(n^{\conf})\frac{\alpha}{\sqrt{n^{\conf}}}\Omega$. 
\end{proof}

\subsection{Condition}
Similar to Proposition~\labelcref{prop:condition}, we derive the same conditions for our algorithm proposed in \Cref{alg:CORNet}.

\begin{proposition}\label{prop:condition_algo}
If the following conditions hold true,
    \begin{align}
       &\frac{\mathcal{C}_{\RC} + \alpha}{\mathit{\bar{d}}_\infty(p^{\conf}_\phi\mid p^{\unc}_\phi)(\mathcal{C}_{\RC} + \alpha) -\mathit{d}_\infty(p^{\conf}_\phi\mid p^{\unc}_\phi) \beta\sqrt{2\log d_\phi}} < \sqrt{\frac{n^{\conf}}{n^{\unc}}},\\
       &\frac{\mathit{d}_\infty(p^{\conf}_\phi\mid p^{\unc}_\phi)\,\beta\sqrt{2\log d_\phi}}{\sqrt{n^{\unc}}} < 2\Delta,
    \end{align}
    then the error of our estimator $\hat{\tau}_{\cor}$ can be substantially lower than the error of the baseline estimators, which only use either observational or randomized data.
\end{proposition}
\begin{proof}
    Similar to \Cref{apx:condition}, where we use $\WGC{n^{\conf}}{\mathcal{H}} = \max_{Z\in\mathcal{Z}} \hat{\mathcal{G}}_{Z}(\mathcal{H}) \leq 2 \log(n^{\conf})\frac{\alpha}{\sqrt{n^{\conf}}}\Omega$ and $\hat{\mathcal{G}}_{Z}(\mathcal{B}) \leq 2 \log(n^{\unc})\frac{\beta \sqrt{2\log d_\phi}}{\sqrt{n^{\unc}}}\Omega$.
\end{proof}


\subsection{Proof of Proposition~\texorpdfstring{\labelcref{cor:cond}}{}}\label{apx:cor_cond}
We use a modified version of Lemma~6 in \citep{tripuraneni2020theory}.
For two representation functions $\hat{\phi}$ and $\phi^\ast$, we define the population covariance as
\begin{align}
    \boldsymbol{\Sigma}(\hat{\phi}, \phi^\ast) = 
\begin{pmatrix}
\E[\hat{\phi}(X^{\conf})\hat{\phi}(X^{\conf})^\top] & \E[\hat{\phi}(X^{\conf})\phi^\ast(X^{\conf})^\top]\\
\E[\phi^\ast(X^{\conf})\hat{\phi}(X^{\conf})^\top] & \E[\phi^\ast(X^{\conf})\phi^\ast(X^{\conf})^\top]
\end{pmatrix}=
\begin{pmatrix}
\boldsymbol{\Sigma}_{\hat{\phi},\hat{\phi}} & \boldsymbol{\Sigma}_{\hat{\phi},\phi^\ast}\\
\boldsymbol{\Sigma}_{\phi^\ast,\hat{\phi}} & 
\boldsymbol{\Sigma}_{\phi^\ast,\phi^\ast}
\end{pmatrix}.
\end{align}
We further define the generalized Schur complement as
\begin{align}
    \boldsymbol{\Sigma}_{sc} = \boldsymbol{\Sigma}_{\phi^\ast\phi^\ast} - \boldsymbol{\Sigma}_{\phi^\ast\hat{\phi}}\boldsymbol{\boldsymbol{\Sigma}}_{\hat{\phi}\hat{\phi}}^{-1}\boldsymbol{\Sigma}_{\hat{\phi}\phi^\ast}.
\end{align}
Then, 
\begin{align}
    d_{\mathcal{B}, \boldsymbol{\delta}}(\hat{\phi};\phi^\ast)
    &= \inf_{\boldsymbol{\delta}^\prime\in\Rl^{d_\phi\otimes 2}}
    \sum_{t=0}^{1}\E[(\hat{\phi}(X^{\conf})\boldsymbol{\delta}^\prime_t - \phi^\ast(X^{\conf})\boldsymbol{\delta}_t)^2]\\
    &= \inf_{\boldsymbol{\delta}^\prime\in\Rl^{d_\phi\otimes 2}}
    \sum_{t=0}^{1}(\boldsymbol{\delta}_t^\prime, - \boldsymbol{\delta}_t)\boldsymbol{\Sigma}(\hat{\phi}, \phi^\ast)(\boldsymbol{\delta}_t^\prime, - \boldsymbol{\delta}_t^\prime)^\top\\
    &=\inf_{\boldsymbol{\delta}^\prime\in\Rl^{d_\phi\otimes 2}}
    \sum_{t=0}^{1}(\boldsymbol{\delta}_t^\prime, \boldsymbol{\delta}_t)\boldsymbol{\Sigma}(\hat{\phi}, \phi^\ast)(\boldsymbol{\delta}_t^\prime, \boldsymbol{\delta}^\prime)^\top\\
    &= \sum_{t=0}^{1}\boldsymbol{\delta}_t\boldsymbol{\Sigma}_{\text{sc}}\boldsymbol{\delta}^\top,
\end{align}
where the last equality follows, since the infimum is a partial minimization of a convex quadratic form (see, for example, \citet{boyd2004convex}, Example 3.15). Finally, we can use the variational characterization of the singular values and that $\lVert\boldsymbol{\delta}_t\rVert_1\leq \alpha$ implies that $\lVert\boldsymbol{\delta}_t\rVert_2\leq \alpha$, which yields
\begin{align}
    \sum_{t=0}^{1}\boldsymbol{\delta}_t\boldsymbol{\Sigma}_{\text{sc}}\boldsymbol{\delta}_t^\top \leq  \sup_{\boldsymbol{\delta}\in\Rl^{d_\phi\otimes 2}, \lVert\boldsymbol{\delta}_t\rVert_1
    \leq \alpha}\sum_{t=0}^{1}\boldsymbol{\delta}_t\boldsymbol{\Sigma}_{\text{sc}} \boldsymbol{\delta}_t^\top\leq  \sup_{\boldsymbol{\delta}\in\Rl^{d_\phi\otimes 2}, \lVert\boldsymbol{\delta}_t\rVert_2
    \leq \alpha}\sum_{t=0}^{1}\boldsymbol{\delta}_t\boldsymbol{\Sigma}_{\text{sc}} \boldsymbol{\delta}_t^\top = 2\alpha\sigma_1(\boldsymbol{\Sigma}_{sc}),
\end{align}
where $\sigma_1(\Sigma_{sc})$ denotes the largest singular value of the generalized Schur complement. Hence,
\begin{align}
    d_{\mathcal{B}, \boldsymbol{\delta}}(\hat{\phi};\phi^\ast)\leq 2\alpha\sigma_1(\boldsymbol{\Sigma}_{sc}).
\end{align}
We proceed similarly for 
\begin{align}
    &d_{\HC, \mathbf{w}^c}(\hat{\phi};\phi^\ast)
    = \inf_{\mathbf{w}^\prime\in\Rl^{d_\phi\otimes 2}}
    \sum_{t=0}^{1}\E[(\hat{\phi}(X^{\conf})\mathbf{w}^\prime_t - \phi^\ast(X^{\conf})(\mathbf{w}^c_t))^2]= \sum_{t=0}^{1}\mathbf{w}^c_t \boldsymbol{\Sigma}_{\text{sc}} (\mathbf{w}^c_t)^\top\\
    &=2\textup{Tr}(\boldsymbol{\Sigma}_{sc}\frac{(\mathbf{w}^c)^\top \mathbf{w}^c}{2}).
\end{align}
Then, using a corollary of the Von-Neumann trace inequality yields

\begin{align}
    \textup{Tr}(\boldsymbol{\Sigma}_{sc}\frac{(\mathbf{w}^c)^\top \mathbf{w}^c}{2}) \geq \sum_{i=1}^{d_\phi}\sigma_i(\boldsymbol{\Sigma}_{sc})\sigma_{d_\phi-i+1}(\frac{(\mathbf{w}^c)^\top \mathbf{w}^c}{2})\geq\textup{Tr}(\boldsymbol{\Sigma}_{sc})\sigma_{d_\phi}(\frac{(\mathbf{w}^c)^\top \mathbf{w}^c}{2})\\
    \geq\sigma_1(\boldsymbol{\Sigma}_{sc})\sigma_{d_\phi}(\frac{(\mathbf{w}^c)^\top \mathbf{w}^c}{2}).
\end{align}
Hence,
\begin{align}
    d_{\mathcal{B}, \boldsymbol{\delta}}(\hat{\phi};\phi^\ast)\leq 2\alpha\sigma_1(\boldsymbol{\Sigma}_{sc})\leq \frac{\alpha}{\sigma_{d_\phi}(\frac{(\mathbf{w}^c)^\top \mathbf{w}^c}{2})}d_{\HC, \mathbf{w}^c}(\hat{\phi};\phi^\ast).
\end{align}


\section{Auxiliary Lemmas}
\begin{lemma}\label{lemma:loss_decomp}
For random variables $X$ and $Y$ and some appropriate function $\textup{f}$, we have that
\begin{align}
    \E[(\textup{f}(X)-Y)^2] = \E[(\textup{f}(X)-\E[Y\mid X])^2] + \E[(\E[Y\mid X]-Y)^2].
\end{align}
\end{lemma}

\begin{proof}
\begin{align}
    &\E[(\textup{f}(X)-Y)^2] = \E[(\textup{f}(X)-\E[Y\mid X]+\E[Y\mid X]-Y)^2]\\ 
    &= \E[(\textup{f}(X)-\E[Y\mid X])^2] + \E[(\E[Y\mid X]-Y)^2] +\underbrace{\E[(\textup{f}(X)-\E[Y\mid X])(\E[Y\mid X]-Y)]}_{=0}\\
    &= \E[(\textup{f}(X)-\E[Y\mid X])^2] + \E[(\E[Y\mid X]-Y)^2].
\end{align}
\end{proof}

\begin{lemma}\label{lemma:change_of_variable}
 Let $\Psi:\mathcal{Z}\rightarrow\mathcal{X}$ be the inverse of $\phi:\mathcal{X}\rightarrow\mathcal{Z}$. Then, for some function $\textup{f}$, random variable $X\sim p$, and $Z=\phi(X)$ with $p_\phi$ being the push-forward of $p$ through $\phi$, the following holds true
 \begin{align}
       \E_{X\sim p}[f(X)] = \E_{Z\sim p_\phi}[\textup{f}(\Psi(Z))].
 \end{align}
\end{lemma}
\begin{proof}
    \begin{align}
        \E_{X\sim p}[\textup{f}(X)] = \int_{\mathcal{X}} \textup{f}(x) p(x) \mathrm{d}x = \int_{\mathcal{Z}} \textup{f}(\Psi(z)) p_\phi(z) \mathit{d}z = \E_{Z\sim p_\phi}[\textup{f}(\Psi(Z))]
    \end{align}
\end{proof}

\section{Details on Optimization of \texorpdfstring{$H$}{}-Divergence}\label{apx:h_divergence_details}
Let $\phi:\Rl^d\rightarrow\Rl^{d_\phi}$ be a representation and $c:\Rl^{d_\phi}\rightarrow\{0,1\}$ be a classifier, which predicts, for a given sample $x$, it is an interpolated or observational sample. Then, the $H$-divergence can be written as
\begin{align}
    \hat{d}_H(\tilde{\mathcal{U}}, \mathcal{U}^{\conf}) = 2\left(1-\min_{c}\left(\frac{1}{m} \sm i m \mathbf{1}_{\{0\}}(c(\phi(\tilde{x}_i)))+ \frac{1}{n^{\conf}}\sm i {n^{\conf}}\mathbf{1}_{\{1\}}(c(\phi(x^{\conf}_i)))\right)\right)
\end{align}
Intuitively, if the distributional discrepancy is
large, the classifier could easily distinguish the observational from interpolated samples. As a consequence, the prediction errors would be small and the H-divergence large. Since we aim at reducing the distributional discrepancy between observational and interpolated data, $\hat{d}_H(\tilde{\mathcal{U}}, \mathcal{U}^{\conf})$ is minimized. This enforces the representation, $\phi$, to learn a representation space in which observational and randomized data are balanced. This is achieved by solving the following optimization problem:
\begin{align}\label{eq:max_min_problem}
    \argmin_{\phi\in\RC}\hat{d}_H(\tilde{\mathcal{U}}, \mathcal{U}^{\conf}) = \argmax_{\phi\in\RC}\min_{c}\left(\frac{1}{m} \sm i m \mathbf{1}_{\{0\}}(c(\phi(\tilde{x}_i)))+ \frac{1}{n^{\conf}}\sm i {n^{\conf}}\mathbf{1}_{\{1\}}(c(\phi(x^{\conf}_i)))\right).
\end{align}
We optimize the max-min problem in \labelcref{eq:max_min_problem} using adversarial learning. In particular, we use a gradient reverse layer (GRL) which automatically reverse the gradient after the classifier. As such, we can directly minimize the classification loss of the classifier $c$ \citep{ganin2015unsupervised}. Hence, this renders, as mentioned in the main paper, step 1 in Algorithm~\labelcref{alg:CORNet} as
\begin{align}
    \argmin_{\phi\in\RC, \mathbf{w}^c\in{\Rl^{d_\phi\otimes 2}}} \frac1{n^{\conf}}\sm i {n^{\conf}} (\phi(x^{\conf}_i)\mathbf{w}^c_{t_i^{\conf}} - y^{\conf}_i)^2
		+\lambda_d\, \hat{d}_H(\tilde{\mathcal{U}}, \mathcal{U}^{\conf}),
\end{align}
where $\lambda_d$ trades off the predictive accuracy and the distributional discrepancy. Therefore, augmented distribution alignment can be easily incorporated by appending a classifier with the gradient reversal layer (GRL), and adding the proposed interpolated samples during mini-batch data preparation.

\section{Details on Real-World Data}\label{apx:data_sets_details}

\textbf{Tennessee Student/Teacher Achievement Ratio (STAR) experiment:} The Tennessee Student/Teacher Achievement Ratio (STAR) experiment is a randomized experiment, which started in 1985 with the objective to study the effect of class size (\ie, treatment) on students' standardized test scores (\ie, outcome). In the first school year, students (as well as teachers) were randomly assigned to class size conditions, which were tried to keep the same for the duration of the experiment. This dataset was also used in \citet{kallus2018removing} for removing bias due to unmeasured confounding in observational data. 

Our setup follows the one in \citet{kallus2018removing}. As such, we focus on two conditions (\ie, treatments): small classes (\ie, 13--17 pupils) and regular classes (\ie, 22--25 pupils). We take as treatment the students class size condition at first grade, which yields 4,509 students. The student outcome $Y$ is measured as the sum of listening, reading, and math standardized test scores at the end of the first grade. In addition to class-size (\ie, treatment) and test scores (\ie, outcome), we use the following covariates for each student: gender, race, birth month, birthday, birth year, free lunch given or not, rural or not, and teacher ID. We remove students with missing outcomes or covariates, which yields a randomized samples of 4,139 students: 1,774 students assigned to treatment (\ie, small class, $T=1$) and 2,365 students assigned to control (\ie, regular class, $T=0$).

\textbf{AIDS Clinical Trial Group (ACTG) study 175:} The AIDS Clinical Trial Group (ACTG) study 175 is a clinical trial with the goal of comparing four treatments randomly to 2,139 subjects with human immunodeficiency virus type 1 (HIV-1), whose CD4 counts were 200--500 cells/mm$^3$ \citep{hammer1996trial}. The trial compared zidovudine (ZDV) monotherapy, the didanosine (ddI) monotherapy, the ZDV combined with ddI, and the ZDV combined with zalcitabine (ZAL). The ACTG study was used in \citet{hatt2021generalizing} for learning policies that generalize to the target population. Even though this is a different problem, the study is particularly suited for evaluating our method since HIV-positive females tend to be underrepresented in clinical trials, which makes these studies not representative of the target population (\ie, the HIV-positive population) \citep{gandhi2005eligibility, greenblatt2011priority}.

The outcome $Y$ is defined as the difference between the cluster of differentiation 4 (CD4) cell counts at the beginning of the study and the CD4 counts after $20\pm5$ weeks. The average treatment effects on the male and female subgroups are $-8.97$ and $-1.39$, respectively \citep{hatt2021generalizing}, which suggests a large discrepancy in the treatment effects between both subgroups. We consider two treatment arms: one treatment arm for both zidovudine (ZDV) and zalcitabine (ZAL) ($T = 1$) vs. one treatment arm for ZDV only ($T = 0$). In total, we consider 1,056 patients and 12 covariates. There are 5 continuous covariates: age (year),  weight (kg, coded as wtkg), CD4 count (cells/$\text{mm}^3$) at baseline, Karnofsky score (scale of 0--100, coded as karnof), CD8 count ($\text{mm}^3$) at baseline. They are centered and scaled before further analysis. In addition, there are 7 binary variables: gender ($1 =$ male, $0 =$ female), homosexual activity (homo, $1 =$ yes, $0 =$ no), race ($1 =$ nonwhite, $0 =$ white), history of intravenous drug use (drug, $1 =$ yes, $0 =$ no), symptomatic status (symptom, $1 =$ symptomatic, $0 =$ asymptomatic), antiretroviral history (str2, $1 =$ experienced, $0 =$ naive) and hemophilia (hemo, $1 =$ yes, $0 =$ no).

\textbf{The National Supported Work:} The National Supported Work (NSW) Demonstration was a transitional, subsidized work experience program that operated for 4 years at 15 locations throughout the United States. The program first provided trainees with work in a sheltered training environment and then assisted them in finding regular jobs. From April 1975 to August 1977, the NSW program operated in 10 locations as a randomized experiment with some program applicants being randomly assigned to a control group that was not allowed to participate in the program. The randomized sample includes 6616 treatment and control observations for which data were gathered through a retrospective baseline interview and four follow-up interviews. These interviews cover the two years prior to random assignment and up to 36 months thereafter.

For our experiments, we consider one randomized dataset from \citet{LaLonde1986}\footnote{The study by \citet{LaLonde1986} is a widely used dataset in the causal inference literature and is also known as the ``Jobs'' dataset \citep[\eg,][]{Shalit2017a, hatt2021estimating}}. Following \citet{Smith2005}, we combine randomized samples of 465 subjects (297 treated, 425 control) with the 2,490 PSID controls to create an observational dataset. The dataset consists of 297 treatment group (\ie, $T=1$) observations and 2,915 control group (\ie, $T=0$) observations. The presence of the experimental subgroup allows us to estimate the ``ground truth'' treatment effect (see explanation below). The study includes 8 covariates: age, level of education, ethnicity (split into two covariates), marital status, and educational degree.

\section{Discussion of the Baseline in \texorpdfstring{\citet{kallus2018removing}}{}}\label{apx:baseline_improvement_discussion}
In this section, we discuss the poor performance of the baselines based on \citet{kallus2018removing} in \Cref{sec:experiments}. For this, we briefly summarize the method proposed by \citep{kallus2018removing}. First, a function $\hat{\textup{f}}$ is estimated using the observational data, \ie, $\{(x_i^{\conf}, t_i^{\conf}, y_i^{\conf})\}_{i=1}^{n^{\conf}}$. For the propensity score $e^{\unc}(x) = \Prb{T=1\mid X^{\unc}=x}$, let $q(x^{\unc}_i) = \frac{t^{\unc}_i}{e^{\unc}(x^{\unc}_i)} - \frac{1-t^{\unc}_i}{1-e^{\unc}(x^{\unc}_i)}$ be a signed re-weighting function. Then, we use this re-weighting function to estimate the following:
\begin{equation}
    \hat{\boldsymbol{\theta}}=\argmin_{\boldsymbol{\theta}\in\Rl^d}\frac{1}{n^{\unc}}\sum_{i=1}^{n^{\unc}} \left(q(x^{\unc}_i)y^{\unc}_i - \textup{f}(x^{\unc}_i) - \boldsymbol{\theta}^\top x^{\unc}_i\right)^2.
\end{equation}
Finally, the proposed estimator for the CATE is given by
\begin{equation}
    \hat{\tau}(x) = \textup{f}(x) + \hat{\boldsymbol{\theta}}^\top x.
\end{equation}
The re-weighting is justified, since $\E[q(X^{\unc})Y^{\unc}\mid X^{\unc}] = \tau(X^{\unc})$ (see Lemma~1 in \citet{kallus2018removing}). Learning a biased function and then the difference is common in transfer learning, where it is called ``offset approach'' \citep[eg, ][]{wang2014active}. Moreover, in order to extrapolate and learn from few randomized samples, \citet{kallus2018removing} assume that the bias is linear in the covariates (hence, the linear regression in the second step yielding $\hat{\boldsymbol{\theta}}$). This linearity for extrapolation has been used before for extrapolating survival functions from clinical trials using external data \citep[\eg,][]{jackson2017extrapolating}.

As discussed in Remark~\labelcref{rmk:poor_baseline_performance}, the poor performance of these baselines originates from the re-weighting of the outcomes, which is well-known to yield high-variance estimates. This is particularly pronounced in a setting with small sample size. We can see this by increasing $n^{\unc}$, which we did in \Cref{sensitivity_results_n}. In this case, the performance of the three baselines improves. However, it remains inferior to the na{\"i}ve approach $\tau_{\unc}$, which improves as well by increasing $n^{\unc}$. In addition, the three baselines remain inferior to our method in any setting.

In order to show that the poor performance originates from re-weighting with the inverse propensity score, we propose two modifications for these methods: (1)~using the estimated propensity score instead of the true one and (2)~targeting the outcomes rather than the CATE directly. 

For~(1), it has been observed that it is often more beneficial to use an estimated propensity score for re-weighting than the true propensity score \citep{Kennedy2016}. Based on this observation, we estimate the propensity score using a logistic regression and, then, re-weight the outcome identically to the algorithm in \citet{kallus2018removing}. For~(2), in order to avoid re-weighting, we target the outcome functions directly, rather than targeting the CATE as in \citet{kallus2018removing}. In particular, we fit a function, $\hat{\textup{f}}(x, t)$, on the confounded data $\{(x_i^{\conf}, y_i^{\conf}, t_i^{\conf})\}_{i=1}^{n^{\conf}}$, \ie,
\begin{align}
    \hat{\textup{f}} = \argmin_{\textup{f}} \frac1{n^{\conf}} \sm i {n^{\conf}} \left(y_i^{\conf} - \textup{f}(x_i^{\conf}, t_i^{\conf})\right)^2.
\end{align}
Then, in the second step, similar to \citet{kallus2018removing}, we remove the confounding bias assuming that it is linear. However, we remove it from the outcome function directly instead of the CATE, \ie, for $n^{\unc}_1 = \lvert \{i\in\{1, \ldots, n^{\unc}\}: t_i^{\unc}=1\}\rvert$ and $n^{\unc}_0 = \lvert \{i\in\{1, \ldots, n^{\unc}\}: t_i^{\unc}=0\}\rvert$,
\begin{align}
    \hat{\boldsymbol{\theta}}_1 = \argmin_{\boldsymbol{\theta}_1} \frac1{n^{\unc}_1} \sm i {n^{\unc}_1} \left(y_i^{\unc} - \textup{f}(x_i^{\unc}, 1) - \boldsymbol{\theta}_1^\top x_i^{\unc}\right)^2,
\end{align}
and 
\begin{align}
    \hat{\boldsymbol{\theta}}_0 = \argmin_{\boldsymbol{\theta}_0} \frac1{n^{\unc}_0} \sm i {n^{\unc}_0} \left(y_i^{\unc} - \textup{f}(x_i^{\unc}, 0) - \boldsymbol{\theta}_0^\top x_i^{\unc}\right)^2,
\end{align}
where $\boldsymbol{\theta}_1$ and $\boldsymbol{\theta}_0$ are the linear regression vectors for each of the outcomes. The CATE is then estimated by $\hat{\tau}(x)= \hat{\textup{f}}(x, 1) - \hat{\boldsymbol{\theta}}_1^\top x - (\hat{\textup{f}}(x, 0) - \hat{\boldsymbol{\theta}}_0^\top x) = \hat{\textup{f}}(x, 1) - \hat{\textup{f}}(x, 0) + (\hat{\boldsymbol{\theta}}_0  - \hat{\boldsymbol{\theta}}_1)^\top x$.

In \Cref{tbl:sim_results_baseline_improved}, we present the results for these two modifications to the original method proposed in \citet{kallus2018removing}. We observe that both modifications improves the performance by a substantial margin. First, using the estimated propensity score ($\text{2-step ridge}^\text{{prop}}$, $\text{2-step RF}^\text{{prop}}$, and $\text{2-step NN}^\text{{prop}}$) improves over the original methods. However, the performance remains poorly compared to the other approaches presented in \Cref{sec:experiments}. Second, removing the bias in the outcome functions directly ($\text{2-step ridge}^\text{out}$, $\text{2-step RF}^\text{out}$, and $\text{2-step NN}^\text{out}$), yields a much larger improvement, placing it on par with the na{\"i}ve approaches, $\tau_{\unc}$ and $\tau_{\conf}$. However, the performance remains substantially inferior to our methods.
\begin{table*}[ht]
	\caption{\footnotesize\label{tbl:sim_results_baseline_improved} Results for the modifications of the baselines proposed in \citet{kallus2018removing} on the real-world datasets STAR, ACTG, and NSW. Results are obtained via 10 runs. Lower is better.
	}
	\begin{center}
		\scalebox{1}{
				\begin{tabular}{lccccc}
					\multicolumn{4}{l}{\bf{Results: Modifications of baselines from \citet{kallus2018removing}}}\\
					\toprule\addlinespace[0.75ex] &\multicolumn{3}{c}{\bf{$\sqrt{\epsilon_{\textup{PEHE}}}$ (Mean $\pm$ Std)}}\\
					\cmidrule{2-4}\addlinespace[0.75ex]
					\tikz{\node[below left, inner sep=1pt] (est) {Estimator};%
      \node[above right,inner sep=1pt] (set) {Dataset};%
      \draw (est.north west|-set.north west) -- (est.south east-|set.south east);}
                    & \bf{STAR}
					&\bf{ACTG}
					&\bf{NSW}\\
					\hline
					\addlinespace[0.75ex]
					\makecell[l]{2-step ridge} 
					&$\text{3.01} \pm \text{0.01}$ 
					& $\text{1.51} \pm \text{0.01}$
					& $\text{2.82} \pm \text{0.02}$\\
					\addlinespace[0.75ex]
					\makecell[l]{2-step RF} 
					&$\text{3.14} \pm \text{0.03}$ 
					& $\text{1.58} \pm \text{0.07}$
					& $\text{3.10} \pm \text{0.12}$\\
					\addlinespace[0.75ex]
					\makecell[l]{2-step NN} 
					&$\text{3.03} \pm \text{0.02}$ 
					& $\text{1.60} \pm \text{0.02}$
					& $\text{2.82} \pm \text{0.02}$\\\hline
                    \addlinespace[0.75ex]
				\makecell[l]{2-step ridge\textsuperscript{prop}}  	        &$\text{1.97} \pm \text{0.01}$ 
					& $\text{1.59} \pm \text{0.01}$
					& $\text{2.77} \pm \text{0.03}$\\
					\addlinespace[0.75ex]
					\makecell[l]{2-step RF\textsuperscript{prop}} 	&$\text{2.14} \pm \text{0.01}$ 
					& $\text{1.36} \pm \text{0.01}$
					& $\text{3.08} \pm \text{0.03}$\\
					\addlinespace[0.75ex]
					\makecell[l]{2-step NN\textsuperscript{prop}} 	&$\text{1.98} \pm \text{0.03}$ 
					& $\text{1.68} \pm \text{0.02}$
					& $\text{2.79} \pm \text{0.03}$\\\hline
                    \addlinespace[0.75ex]
			\makecell[l]{2-step ridge\textsuperscript{out}}  					    &$\text{1.85} \pm \text{0.01}$ 
					& $\text{0.92} \pm \text{0.01}$
					& $\text{0.16} \pm \text{0.02}$\\
					\addlinespace[0.75ex]
					\makecell[l]{2-step RF\textsuperscript{out}} 	&$\text{2.08} \pm \text{0.01}$ 
					& $\text{1.17} \pm \text{0.02}$
					& $\text{0.97} \pm \text{0.04}$\\
					\addlinespace[0.75ex]
					\makecell[l]{2-step NN\textsuperscript{out}} 	&$\text{1.90} \pm \text{0.03}$ 
					& $\text{0.84} \pm \text{0.03}$
					& $\text{0.12} \pm \text{0.04}$\\
					\bottomrule	
			\end{tabular}}
	\end{center}
\end{table*}

\section{Sensitivity of Ablation Study with respect to \texorpdfstring{$n^{\conf}$}{} and \texorpdfstring{$n^{\unc}$}{}}\label{apx:sensitivity_ablation}
In this section, we study the sensitivity of the ablation study in \Cref{sec:::exp_ablation_study} for different numbers of $n^{\conf}$ and $n^{\unc}$. In the ablation study, we chose an explicit size of the observational and randomized dataset. We chose $n^{\conf}$ to be the maximum available samples for the observational dataset after the processing described in \Cref{sec:::exp_os_rct}. In particular, this yields $n^{\conf}\in\{643, 552, 482\}$ for the datasets STAR, ACTG, and NSW. Moreover, we chose $n^{\unc}=2\,d$, where $d$ is the number of covariates. We did so, since the randomized datasets is usually data-scarce. Here, we study the robustness of our results from the ablation study in \Cref{sec:::exp_ablation_study} when $n^{\conf}$ and $n^{\unc}$ change.

\textbf{Sensitivity in $n^{\conf}$:} We conduct an ablation study on the different components of our method across different numbers of observational samples used, \ie, we vary $n^{\conf}$. In particular, we choose $n^{\conf} \in \{300, 400, 500, 600, 643\}$ for STAR, $n^{\conf} \in \{200, 300, 400, 500, 552\}$ for ACTG, and $n^{\conf} \in \{100, 200, 300, 400, 482\}$ for NSW. The results are presented in \Cref{fig:sensitivity_ablation} (top). We observe similar behavior as in the ablation study in \Cref{sec:::exp_ablation_study}: $\tau_{\textup{CorNet}^+}$ consistently improves upon $\tau_{\textup{CorNet}}$ across all numbers of observational and randomized samples for the datasets STAR and ACTG, but not on NSW. The reason for this seems to be that the bias function is not sparse on NSW. 

\textbf{Sensitivity in $n^{\unc}$:} We conduct an ablation study on the different components of our method across different numbers of randomized samples used, \ie, we vary $n^{\unc}$. In particular, we choose $n^{\unc} \in \{1\cdot\textup{d}, 2\cdot\textup{d}, 3\cdot\textup{d}, 4\cdot\textup{d}, 6\cdot\textup{d}, 8\cdot\textup{d}\}$ for STAR, ACTG, and NSW. The results are presented in \Cref{fig:sensitivity_ablation} (bottom). We observe similar behavior as in the ablation study in \Cref{sec:::exp_ablation_study}: $\tau_{\textup{CorNet}^+}$ consistently improves upon $\tau_{\textup{CorNet}}$ across all numbers of observational and randomized samples for the datasets STAR and ACTG and the differences between the methods becomes smaller, the more randomized samples we acquire. On NSW, the results are less clear. Although $\tau_{\textup{CorNet}}$ and $\tau_{\textup{CorNet}^+}$ consistently achieve low error, depending on the number of randomized samples, $\tau_{\textup{CorNet}}$ is superior to $\tau_{\textup{CorNet}^+}$ or vice versa. We attribute this behavior to the non-sparse bias function, which leads to greater variability of the estimates.
\begin{figure}[ht]
	\centering 
	\scalebox{0.425}{\includegraphics{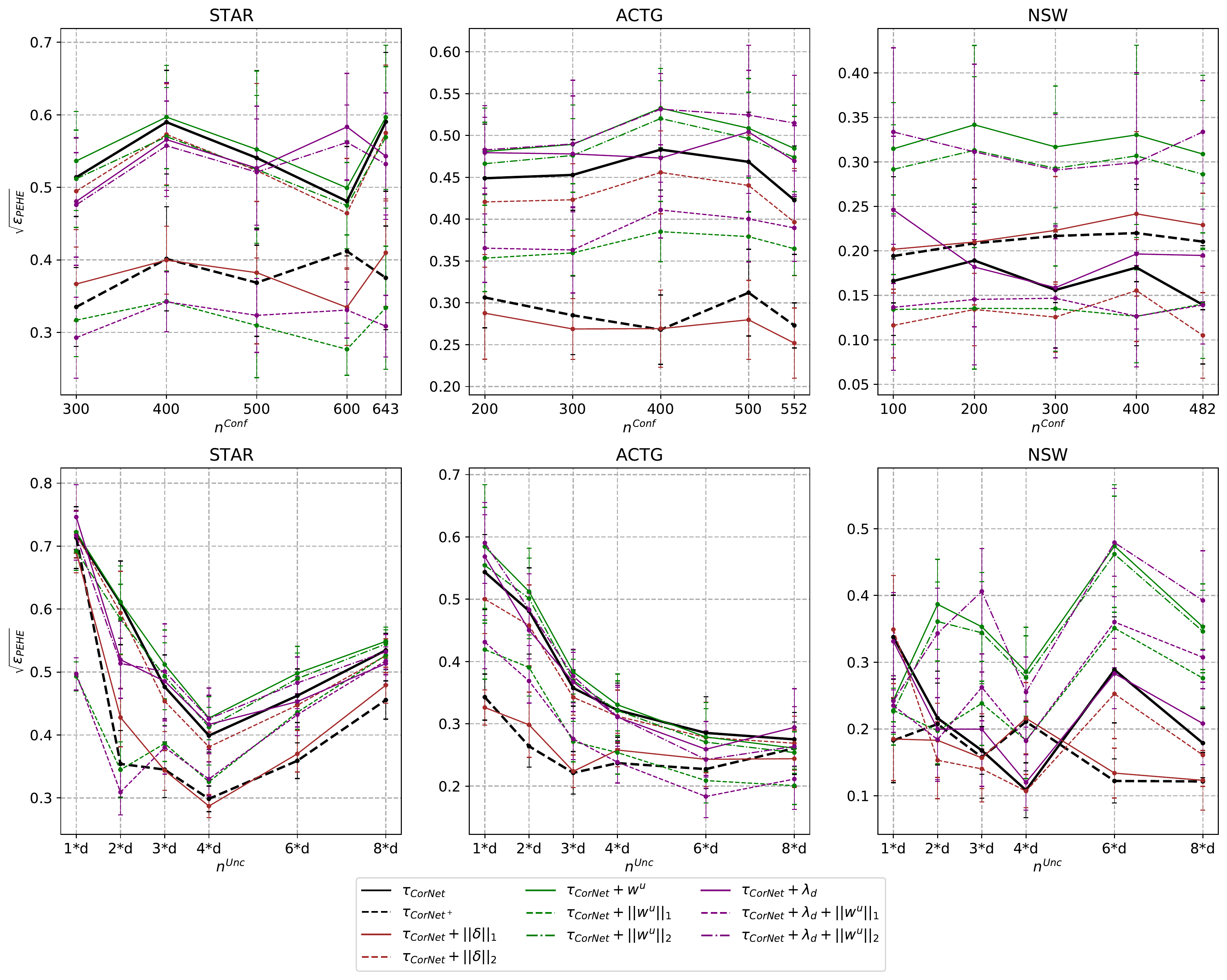}}
	\caption{\footnotesize Sensitivity study on ablation study for different numbers of $n^{\conf}$ (top) and $n^{\unc}$ (bottom) across all real-world datasets STAR, ACTG, and NSW. The findings are consistent across all number of observational (\ie, $n^{\conf}$) and randomized samples (\ie, $n^{\unc}$).}\label{fig:sensitivity_ablation}
\end{figure}

\section{Variants and Extensions of CorNet}\label{apx:variant_study}
In this section, we study multi-task variants of our method as follows. (i)~We simply estimate the representation function, $\phi$, and all four hypotheses, $\mathbf{h}^c=(\textup{h}^c_1, \textup{h}^c_0)$ and $\mathbf{h}^u=(\textup{h}^u_1, \textup{h}^u_0)$, in one step. We denote this 1-step multi-task learning (MTL) variant as $\tau_{\text{MTL}}$. (ii)~In addition this, we consider a 2-step MTL variant, which uses the representation, $\hat{\phi}$, and confounded hypotheses, $\hat{\mathbf{h}}^c$, from $\tau_{\text{MTL}}$ and, in a second step, estimate the bias function with a $L_1$-regularization. This resembles our method. (iii)~We also study the 2-step MTL variant that uses balancing and (iv)~the 2-step MTL variant with $L_1$-regularization and balancing.

The results are presented in \Cref{tbl:variants_study_results}. We observe that the multi-task learning variant, $\tau_{\text{MTL}}$, performs inferior to our method. Notably, it yields much higher variance than our method. This is, again, due to the small sample size of randomized data, which is not controlled for, since the method is learned in one step. The 2-step multi-task learning variants, which also learn the bias function in a second step and also apply $L_1$-regularization and balancing do not improve upon our methods. Similarly to the sensitivity study for $n^{\conf}$ and $n^{\unc}$ in \Cref{sec:::exp_sensitivity_n} and \Cref{apx:sensitivity_ablation}, we run such a sensitivity study for the multi-task learning variants. The findings are consistent across different numbers of observational and randomized samples. The results can be found in \Cref{apx:sensitivity_variants}.
\begin{table*}[ht!]
	\caption{\footnotesize\label{tbl:variants_study_results} Results for the multi-task learning variant study on the real-world datasets STAR, ACTG, and NSW. ``2-step'', $L_1$, and $\lambda_d$ indicate whether a 2-step procedure, $L_1$-regularization, or balancing was used. Results obtained via 10 runs. Lower is better.}
	\begin{center}
				\begin{tabular}{llcccccc}
					\multicolumn{7}{l}{\bf{Results: Variant study}}\\
					\toprule\addlinespace[0.75ex] &&&&\multicolumn{3}{c}{\bf{$\sqrt{\hat{\epsilon}_{\textup{PEHE}}}$ (Mean $\pm$ Std)}}\\
					\cmidrule{5-7}\addlinespace[0.75ex]
                    \multicolumn{2}{l}{Estimator}
                    & 2-step
                    & $L_1$
                    & $\lambda_d$
                    & \bf{STAR}
					&\bf{ACTG}
					&\bf{JOBS}\\\hline
					\addlinespace[0.75ex]
	    				\multirow{4}{*}{
				    \rotatebox[origin=c]{90}{$\tau_{\textup{MTL}}$}
					}&
	        \makecell[l]{(i)} 				
                    & \xmark
					& \xmark
					& \xmark
					&$\text{1.27} \pm \text{0.21}$ 
					& $\text{1.07} \pm \text{0.07}$
					& $\text{0.48} \pm \text{0.13}$\\
                    \addlinespace[0.75ex]
					&\makecell[l]{(ii)} 
					& \cmark
					& \cmark
					& \xmark
					&$\text{0.59} \pm \text{0.06}$ 
					& $\text{0.35} \pm \text{0.05}$
					& $\text{0.13} \pm \text{0.04}$\\
					\addlinespace[0.75ex]
					&\makecell[l]{(iii)} 
					& \cmark
					& \xmark
					& \cmark
					&$\text{1.21} \pm \text{0.25}$ 
					& $\text{1.12} \pm \text{0.09}$
					& $\text{0.47} \pm \text{0.10}$\\
					\addlinespace[0.75ex]
					&\makecell[l]{(iv)}
					& \cmark
					& \cmark
					& \cmark
					&$\text{1.16} \pm \text{0.18}$ 
					& $\text{0.89} \pm \text{0.06}$
					& $\text{0.14} \pm \text{0.06}$\\\hline		\addlinespace[0.75ex]
		    &\makecell[l]{(v)~$\tau_{\textup{CorNet}}$} 					& \cmark
					& \xmark
					& \xmark
					&$\text{0.59} \pm \text{0.10}$ 
					& $\text{0.42} \pm \text{0.06}$
					& $\text{0.14} \pm \text{0.07}$\\
					\addlinespace[0.75ex]
					&\makecell[l]{(vi)~$\tau_{\textup{CorNet}^+}$ } 
					& \cmark
					& \cmark
					& \cmark
					&$\text{0.38} \pm \text{0.07}$ 
					& $\text{0.27} \pm \text{0.03}$
					& $\text{0.21} \pm \text{0.08}$\\
					\bottomrule	
			\end{tabular}
	\end{center}
\end{table*}

\subsection{Sensitivity of Variants and Extensions with respect to \texorpdfstring{$n^{\conf}$}{} and \texorpdfstring{$n^{\unc}$}{}}\label{apx:sensitivity_variants}
In this section, we study the sensitivity of the variant study in \Cref{apx:variant_study} for different numbers of $n^{\conf}$ and $n^{\unc}$. In the variant study, we chose an explicit size of the observational and randomized dataset. We chose $n^{\conf}$ to be the maximum available samples for the observational dataset after the processing described in \Cref{sec:::exp_os_rct}. In particular, this yields $n^{\conf}\in\{643, 552, 482\}$ for the datasets STAR, ACTG, and NSW. Moreover, we chose $n^{\unc}=2\,d$, where $d$ is the number of covariates. We did so, since the randomized datasets is usually data-scarce. Here, we study the robustness of our results from the variant study in \Cref{apx:variant_study} when $n^{\conf}$ and $n^{\unc}$ change.

\textbf{Sensitivity in $n^{\conf}$:} We conduct an variant study on the different components of our method across different numbers of observational samples used, \ie, we vary $n^{\conf}$. In particular, we choose $n^{\conf} \in \{300, 400, 500, 600, 643\}$ for STAR, $n^{\conf} \in \{200, 300, 400, 500, 552\}$ for ACTG, and $n^{\conf} \in \{100, 200, 300, 400, 482\}$ for NSW. The results are presented in \Cref{fig:sensitivity_variants} (top). We observe similar behavior as in \Cref{apx:variant_study}: $\tau_{\text{MTL}}$ consistently performs inferior to our methods across all numbers of observational and randomized samples and across all datasets. Moreover, we find that the 2-step multi-task learning variant with $L_1$-regularization achieves similar performance on ACTG and NSW.

\textbf{Sensitivity in $n^{\unc}$:} We vary $n^{\unc}$. In particular, we choose $n^{\unc} \in \{1\cdot\textup{d}, 2\cdot\textup{d}, 3\cdot\textup{d}, 4\cdot\textup{d}, 6\cdot\textup{d}, 8\cdot\textup{d}\}$ for STAR, ACTG, and NSW. The results are presented in \Cref{fig:sensitivity_variants} (bottom). We observe similar behavior as in \Cref{apx:variant_study}: $\tau_{\text{MTL}}$ consistently performs inferior to our methods across all numbers of observational and randomized samples and across all datasets. However, the performance of all variants converges when the number of randomized samples increases. This is because once we exceed a certain number of randomized samples, the multi-task learning variant stops suffering from the data scarcity of the randomized samples. This behavior is less clear on NSW. Although we do not observe any benefit from the multi-task learning variant, $\tau_{\text{MTL}}$, and its covariate balanced extensions, $\tau_{\text{MTL}}+\lambda_d$ and $\tau_{\text{MTL}}+\lambda_d + \lVert \boldsymbol{\delta}\rVert_1$, we observe that the extension which re-estimates the bias function (and regularizes it) in a second step, $\tau_{\text{MTL}}+\lVert \boldsymbol{\delta}\rVert_1$, seems to have some advantage in some cases, but not consistently. However, it is doubtful how reasonable it is to re-estimate the bias function, if the unconfounded hypotheses are already estimated in the first step.
\begin{figure}[ht]
	\centering 
	\scalebox{0.415}{\includegraphics{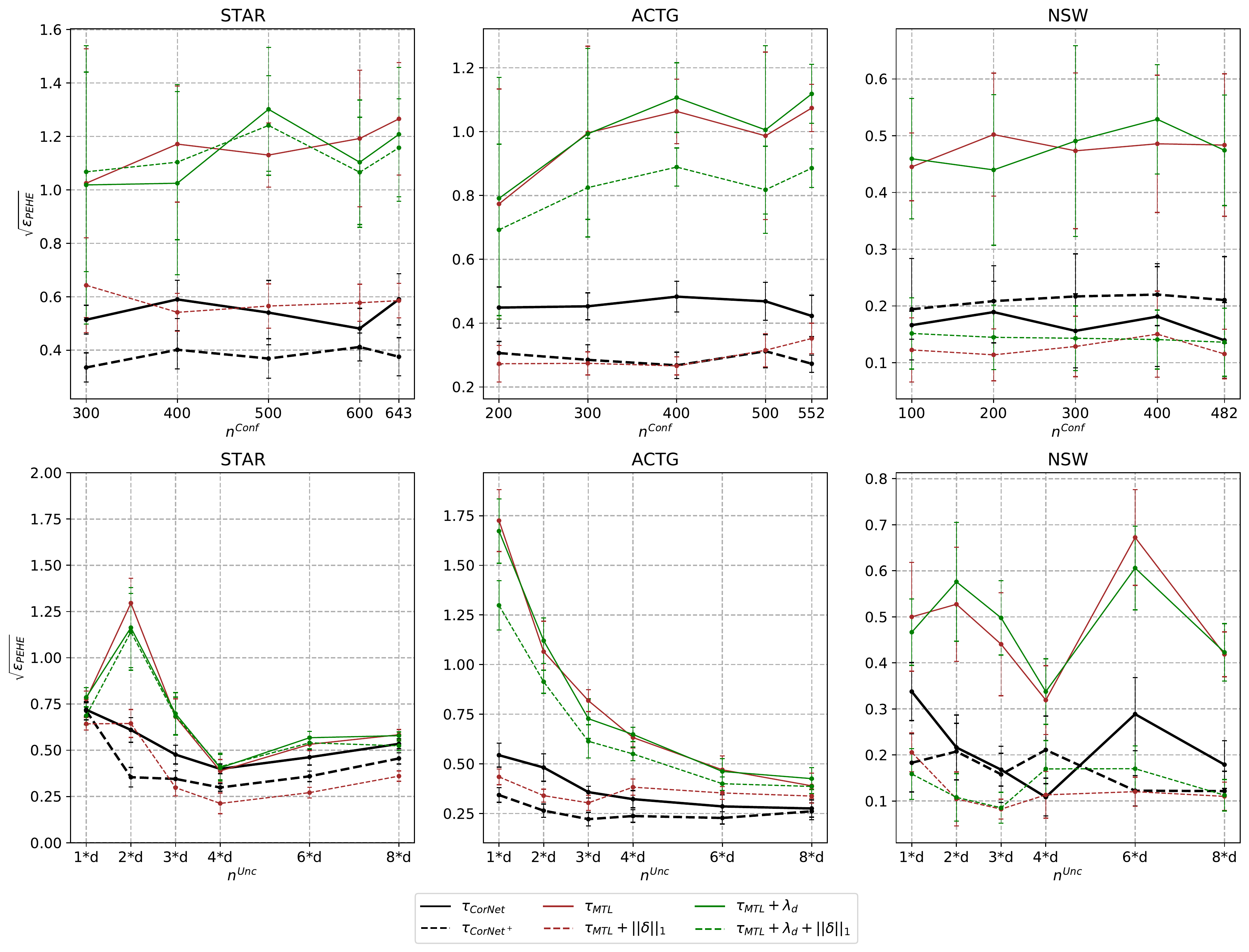}}
	\caption{\footnotesize Sensitivity study on the variant study for different numbers of $n^{\conf}$ (top) and $n^{\unc}$ (bottom) across all real-world datasets STAR, ACTG, and NSW. The results are mostly consistent across the number of observational (\ie, $n^{\conf}$) and randomized samples (\ie, $n^{\unc}$). For large number of randomized samples, the error of $\tau_{\text{MTL}}$ and its extensions approach the error of our methods, $\tau_{\cor}$ and $\tau_{\textup{CorNet}}$. This is expected, since once the randomized dataset is not small anymore. While $\tau_{\text{MTL}}$ and its covariate balanced variants perform consistently inferior than our method, $\tau_{\text{MTL}} + \lVert \boldsymbol{\delta}\rVert_1$ shows a slight advantage in some cases.}\label{fig:sensitivity_variants}
\end{figure}

\end{document}